\documentclass[english]{article}
\usepackage[T1]{fontenc}
\usepackage[latin9]{inputenc}
\usepackage{geometry}
\usepackage{color}
\usepackage{babel}
\usepackage{verbatim}
\usepackage{mathtools}
\usepackage{bm}
\usepackage{amsmath}
\usepackage{amsthm}
\usepackage{amssymb}
\usepackage{graphicx}
\usepackage[unicode=true,pdfusetitle,
 bookmarks=true,bookmarksnumbered=true,bookmarksopen=true,bookmarksopenlevel=2,
 breaklinks=false,pdfborder={0 0 0},pdfborderstyle={},backref=false,colorlinks=true]
 {hyperref}
\hypersetup{
 linkcolor=blue, urlcolor=blue, citecolor=blue, linktoc=all, linktocpage=false}

\makeatletter
\theoremstyle{definition}
\newtheorem{defn}{\protect\definitionname}
\theoremstyle{plain}
\newtheorem{prop}{\protect\propositionname}
\theoremstyle{plain}
\newtheorem{fact}{\protect\factname}
\theoremstyle{plain}
\newtheorem{thm}{\protect\theoremname}
\theoremstyle{remark}
\newtheorem{rem}{\protect\remarkname}
\theoremstyle{plain}
\newtheorem{cor}{\protect\corollaryname}
\theoremstyle{plain}
\newtheorem{lem}{\protect\lemmaname}
\theoremstyle{definition}
 \newtheorem{example}{\protect\examplename}

\usepackage{babel}
\usepackage{bm}
\usepackage{bbm}

\usepackage[numbers,square]{natbib}

\usepackage[linesnumbered,ruled,vlined]{algorithm2e}

\SetCommentSty{mycommfont}
\usepackage{algorithmic}
\renewcommand{\Pr}{{\mathbb{P}}}

\newcommand{\R}{{\mathbb{R}}}

\let\hat\widehat
\let\bar\overline
\let\tilde\widetilde

\allowdisplaybreaks

\usepackage{color}

\definecolor{todo}{RGB}{0,200,200}
\newcommand{\todo}[1]{\textcolor{todo}{[TODO: #1]}}

\definecolor{new}{RGB}{0,200,200}
\newcommand{\new}[1]{\textcolor{todo}{[NEW:]}}

\definecolor{emerald}{rgb}{0.31, 0.78, 0.47}

\theoremstyle{plain}

\newcommand{\bA}{\bm{A}}
\newcommand{\bB}{\bm{B}}

\newcommand{\bD}{\bm{D}}

\newcommand{\bI}{\bm{I}}

\newcommand{\bL}{\bm{L}}
\newcommand{\bM}{\bm{M}}

\newcommand{\bP}{\bm{P}}

\newcommand{\bU}{\bm{U}}

\newcommand{\bc}{\bm{c}}

\newcommand{\be}{\bm{e}}

\newcommand{\bi}{\bm{i}}
\newcommand{\bj}{\bm{j}}

\newcommand{\bv}{\bm{v}}
\newcommand{\bw}{\bm{w}}
\newcommand{\bx}{\bm{x}}

\newcommand{\ErdosRenyi}{{Erd\H{o}s-R\'{e}nyi} }
\newcommand{\CramerRao}{{Cram\'{e}r-Rao} }

\@ifundefined{showcaptionsetup}{}{%
 \PassOptionsToPackage{caption=false}{subfig}}
\usepackage{subfig}
\makeatother

\providecommand{\corollaryname}{Corollary}
\providecommand{\definitionname}{Definition}
\providecommand{\examplename}{Example}
\providecommand{\factname}{Fact}
\providecommand{\lemmaname}{Lemma}
\providecommand{\propositionname}{Proposition}
\providecommand{\remarkname}{Remark}
\providecommand{\theoremname}{Theorem}

\begin{document}
\title{}

\title{Ranking from Pairwise Comparisons in General Graphs and Graphs with
Locality}
\author{Yanxi Chen\thanks{Department of Electrical and Computer Engineering, Princeton University, Princeton,
NJ 08544, USA; email: \texttt{yxchen0@gmail.com}.}}
\date{\today}

\maketitle
\global\long\def\Tr{\mathsf{Tr}}%
\global\long\def\poly{\mathsf{poly}}%
\global\long\def\ltwo{\ell_{2}}%
\global\long\def\linf{\ell_{\infty}}%
\global\long\def\Lcal{\mathcal{L}}%
\global\long\def\Ccal{\mathcal{C}}%
\global\long\def\Ecal{\mathcal{E}}%
\global\long\def\var{\mathsf{var}}%
\global\long\def\polylog{\mathsf{polylog}}%

\global\long\def\Uperp{\bU_{\perp}}%
\global\long\def\col{\mathsf{col}}%

\global\long\def\btheta{\bm{\theta}}%
\global\long\def\bthetastar{\bm{\theta}^{\star}}%
\global\long\def\bthetat{\bm{\theta}^{t}}%
\global\long\def\bthetatp{\bm{\theta}^{t+1}}%

\global\long\def\bthetaMLE{\bm{\theta}^{\mathsf{MLE}}}%
\global\long\def\thetaMLE{\theta^{\mathsf{MLE}}}%
\global\long\def\bthetahat{\hat{\btheta}}%

\global\long\def\thetastar{\theta^{\star}}%
\global\long\def\thetat{\theta^{t}}%
\global\long\def\thetatp{\theta^{t+1}}%

\global\long\def\bdeltat{\bm{\delta}^{t}}%
\global\long\def\bdeltatp{\bm{\delta}^{t+1}}%
\global\long\def\bft{\bm{f}^{t}}%

\global\long\def\deltat{\delta^{t}}%
\global\long\def\deltatp{\delta^{t+1}}%
\global\long\def\ft{f^{t}}%

\global\long\def\wstar{w^{\star}}%
\global\long\def\bwstar{\bw^{\star}}%
\global\long\def\pistar{\pi^{\star}}%
\global\long\def\bpistar{\bm{\pi}^{\star}}%
\global\long\def\bpi{\bm{\pi}}%
\global\long\def\sigmoid{\sigma}%
\global\long\def\bdelta{\bm{\delta}}%

\global\long\def\Ni{N_{i}}%
\global\long\def\Ncali{\mathcal{N}(i)}%
\global\long\def\ke{\kappa_{E}}%

\global\long\def\ERnp{\mathsf{ER}(n,p)}%
\global\long\def\ind{\mathbbm{1}}%
\global\long\def\vone{\bm{1}}%
\global\long\def\vonen{\vone_{n}}%
\global\long\def\vzero{\bm{0}}%

\global\long\def\Lzstar{\bL_{z}^{\star}}%
\global\long\def\Lzstarinv{(\bL_{z}^{\star})^{\dagger}}%
\global\long\def\Lz{\bL_{z}}%
\global\long\def\Lzinv{\bL_{z}^{\dagger}}%
\global\long\def\LG{\bL_{G}}%

\global\long\def\sumE{\sum_{(i,j)\in E,i<j}}%
\global\long\def\sumL{\sum_{1\le l\le L_{ij}}}%
\global\long\def\Lij{L_{i,j}}%
\global\long\def\Lkl{L_{k,\ell}}%
\global\long\def\zkl{z_{k,\ell}}%

\global\long\def\biext{\bi^{\mathsf{ext}}}%
\global\long\def\bcext{\bc^{\mathsf{ext}}}%

\global\long\def\yij{y_{i,j}}%
\global\long\def\yji{y_{j,i}}%
\global\long\def\yijl{y_{i,j}^{(\ell)}}%
\global\long\def\epsijl{\epsilon_{i,j}^{(\ell)}}%
\global\long\def\bepshat{\hat{\bm{\epsilon}}}%
\global\long\def\Aij{A_{i,j}}%
\global\long\def\epsij{\epsilon_{i,j}}%

\global\long\def\zijstar{z_{i,j}^{\star}}%
\global\long\def\bxi{\bm{\xi}}%
\global\long\def\zij{z_{i,j}}%
\global\long\def\cij{c_{i,j}}%
\global\long\def\Cij{C_{i,j}}%
\global\long\def\Rij{R_{i,j}}%

\global\long\def\ei{\be_{i}}%
\global\long\def\ej{\be_{j}}%
\global\long\def\ek{\be_{k}}%
\global\long\def\el{\be_{\ell}}%
\global\long\def\xik{\xi_{k}}%
\global\long\def\xil{\xi_{\ell}}%

\global\long\def\deltatk{\delta_{k}^{t}}%
\global\long\def\deltatl{\delta_{\ell}^{t}}%
\global\long\def\deltati{\delta_{i}^{t}}%
\global\long\def\deltatj{\delta_{j}^{t}}%

\global\long\def\deltatpk{\delta_{k}^{t+1}}%
\global\long\def\deltatpl{\delta_{\ell}^{t+1}}%
\global\long\def\deltatpi{\delta_{i}^{t+1}}%
\global\long\def\deltatpj{\delta_{j}^{t+1}}%

\global\long\def\Omegakl{\Omega_{k,\ell}}%
\global\long\def\Omegaij{\Omega_{i,j}}%
\global\long\def\Omegamax{\Omega_{\mathsf{max}}}%
\global\long\def\OmegamaxE{\Omega_{\mathsf{max},E}}%

\global\long\def\brt{\bm{r}^{t}}%
\global\long\def\zetatij{\zeta_{i,j}^{t}}%

\global\long\def\Vagg{V}%
\global\long\def\Vaggkl{V_{k,\ell}}%
\global\long\def\Vaggij{V_{i,j}}%
\global\long\def\Vkl{\Vaggkl}%
\global\long\def\Vij{\Vaggij}%

\global\long\def\Bij{B_{i,j}}%
\global\long\def\Bkl{B_{k,\ell}}%
\global\long\def\Qkl{Q_{k,\ell}}%
\global\long\def\Qij{Q_{i,j}}%

\global\long\def\maxijE{\max_{(i,j)\in E}}%
\global\long\def\maxklE{\max_{(k,\ell)\in E}}%
\global\long\def\vi{v_{i}}%
\global\long\def\vj{v_{j}}%

\global\long\def\gridone{\mathsf{Grid1D}}%
\global\long\def\gridtwo{\mathsf{Grid2D}}%

\global\long\def\distrov{\mathsf{DC\text{-}overlap}}%
\global\long\def\distrcomm{\mathsf{DC\text{-}community}}%

\global\long\def\DCov{\distrov}%
\global\long\def\DCcomm{\distrcomm}%

\global\long\def\Va{V_{(a)}}%
\global\long\def\Vb{V_{(b)}}%
\global\long\def\Ea{E_{(a)}}%
\global\long\def\Ga{G_{(a)}}%
\global\long\def\bea{\be_{(a)}}%
\global\long\def\beb{\be_{(b)}}%

\global\long\def\bphi{\bm{\phi}}%
\global\long\def\bthetaa{\bm{\theta}_{(a)}}%
\global\long\def\bthetab{\bm{\theta}_{(b)}}%

\global\long\def\thetaai{\theta_{(a)i}}%
\global\long\def\thetaaj{\theta_{(a)j}}%
\global\long\def\thetabi{\theta_{(b)i}}%
\global\long\def\ca{c_{(a)}}%
\global\long\def\cb{c_{(b)}}%
\global\long\def\thetabj{\theta_{(b)j}}%

\global\long\def\si{s_{i}}%
\global\long\def\wij{w_{i,j}}%

\global\long\def\Gtilde{\tilde{G}}%
\global\long\def\Vtilde{\tilde{V}}%
\global\long\def\Etilde{\tilde{E}}%
\global\long\def\Delab{\Delta_{(a,b)}}%
\global\long\def\wab{w_{(a,b)}}%

\global\long\def\bcstar{\bc^{\star}}%
\global\long\def\cstara{c_{(a)}^{\star}}%
\global\long\def\cstarb{c_{(b)}^{\star}}%
\global\long\def\bthetastara{\btheta_{(a)}^{\star}}%
\global\long\def\bthetastarb{\btheta_{(b)}^{\star}}%
\global\long\def\vonea{\vone_{(a)}}%

\global\long\def\beps{\bm{\epsilon}}%
\global\long\def\bzeta{\bm{\zeta}}%

\global\long\def\Ltilde{\tilde{\bL}}%
\global\long\def\bdela{\bm{\delta}_{(a)}}%
\global\long\def\bdelb{\bm{\delta}_{(b)}}%
\global\long\def\delai{\delta_{(a)i}}%
\global\long\def\delbi{\delta_{(b)i}}%
\global\long\def\nab{n_{(a,b)}}%
\global\long\def\Btilde{\tilde{\bB}}%

\global\long\def\delbar{\bar{\delta}}%
\global\long\def\bepstilde{\tilde{\beps}}%
\global\long\def\sumab{\sum_{(a,b)\in\Etilde,a<b}}%
\global\long\def\main{\mathsf{main}}%
\global\long\def\res{\mathsf{res}}%
\global\long\def\cbarstar{\bar{c}^{\star}}%

\begin{abstract}
This technical report studies the problem of ranking from pairwise comparisons in the classical
Bradley-Terry-Luce (BTL) model, with a focus on score estimation.
For general graphs, we show that, with sufficiently many samples,
maximum likelihood estimation (MLE) achieves an entrywise estimation
error matching the \CramerRao lower bound, which can be stated in
terms of effective resistances; the key to our analysis is a connection
between statistical estimation and iterative optimization by preconditioned
gradient descent. We are also particularly interested in graphs with
locality, where only nearby items can be connected by edges; our
analysis identifies conditions under which locality does not hurt,
i.e.~comparing the scores between a pair of items that are far apart
in the graph is nearly as easy as comparing a pair of nearby items.
We further explore divide-and-conquer algorithms that can provably
achieve similar guarantees even in the regime with the sparsest samples,
while enjoying certain computational advantages. Numerical results
validate our theory and confirm the efficacy of the proposed algorithms.
\end{abstract}


\tableofcontents

\section{Introduction \label{sec:intro}}

We study the problem of ranking from pairwise comparisons in the Bradley-Terry-Luce
(BTL) model \cite{zermelo1929berechnung,bradley1952rank,luce2012individual}.
Consider $n$ items with ground-truth scores $\bthetastar\in\R^{n}$,
and a general connected graph $G=(V,E)$ with $n=|V|$ nodes; throughout
this technical report, we always consider undirected graphs, and assume
without loss of generality that $\langle\bthetastar,\vonen\rangle=\sum_{1\le i\le n}\thetastar_{i}=0$.
Suppose that $\bthetastar$ has a finite dynamic range both globally
and locally, i.e.~there exist some $\kappa\ge\ke\ge1$ such that
\begin{equation}
\max_{1\le i<j\le n}|\thetastar_{i}-\thetastar_{j}|\le\log\kappa\quad\text{and}\quad\max_{(i,j)\in E}|\thetastar_{i}-\thetastar_{j}|\le\log\ke.\label{eq:def_kappa_ke}
\end{equation}
Denote the sigmoid function $\sigmoid$ as $\sigmoid(x)=1/(1+e^{-x})$.
For each edge $(i,j)\in E,i<j$, we collect $\Lij$ independent Bernoulli
random samples, which leads to the dataset
\[
\{\yijl,1\le\ell\le\Lij\}_{(i,j)\in E,i<j},\quad\text{where}\quad\Pr(\yijl=1)=1-\Pr(\yijl=0)=\sigmoid(\thetastar_{i}-\thetastar_{j}).
\]
It is assumed that $\yijl=1-\yji^{(\ell)}$ for all $(i,j)\in E$.
We also denote the sufficient statistics as
\[
\yij\coloneqq\frac{1}{\Lij}\sumL\yijl=1-\yji,\quad(i,j)\in E.
\]
Our goal is to accurately estimate the ground-truth scores $\bthetastar$,
based on the data $\{\yijl\}$ or $\{\yij\}$. 

There has been a vast literature on the BTL model; in particular,
statistical guarantees for the special case of \ErdosRenyi graphs
have been extensively studied. In contrast, the literature on \emph{entrywise
}estimation errors for \emph{general graphs} is relatively sparse,
and falls short of correctly capturing the underlying structures of
the problem with tight error bounds; see Section \ref{subsec:related_works}
for a more detailed discussion of prior works. Moreover, real-world
graphs often exhibit \emph{locality} \cite{chen2016community}, i.e.~it
is only possible to obtain measurements for a pair of items that are
nearby in some physical space; such a setting has been absent from
the literature of the BTL model. For concreteness, we consider 1D/2D
grids with $n$ nodes, radius $r$ and edge sampling probability $p$,
formally defined as $\gridone/\gridtwo$ in Definition~\ref{def:grids}
and visualized in Figure~\ref{fig:grids}. Throughout this work,
we will use the notions of 1D/2D grids and $\gridone$/$\gridtwo$
interchangeably. It is worth noting that much of our analysis in this
work can be extended to other graphs with locality, such as random
geometric graphs or graphs with community structures.
\begin{defn}
\label{def:grids}We define 1D and 2D grids as follows.
\begin{itemize}
\item $\gridone(n,r,p)$: nodes can be indexed by $\{1,2,\dots,n\}$, such
that $(i,j)\in E$ with probability $p$ if $|i-j|\le r$, and $(i,j)\notin E$
otherwise.
\item $\gridtwo(n,r,p)$: nodes can be indexed by $\{\bi=(i_{1},i_{2}):1\le i_{1},i_{2}\le\sqrt{n}\}$,
such that $(\bi,\bj)\in E$ with probability $p$ if $|i_{1}-j_{1}|+|i_{2}-j_{2}|\le r$,
and $(i,j)\notin E$ otherwise. 
\end{itemize}
\end{defn}
\begin{figure}
\begin{centering}
\includegraphics[width=0.4\textwidth]{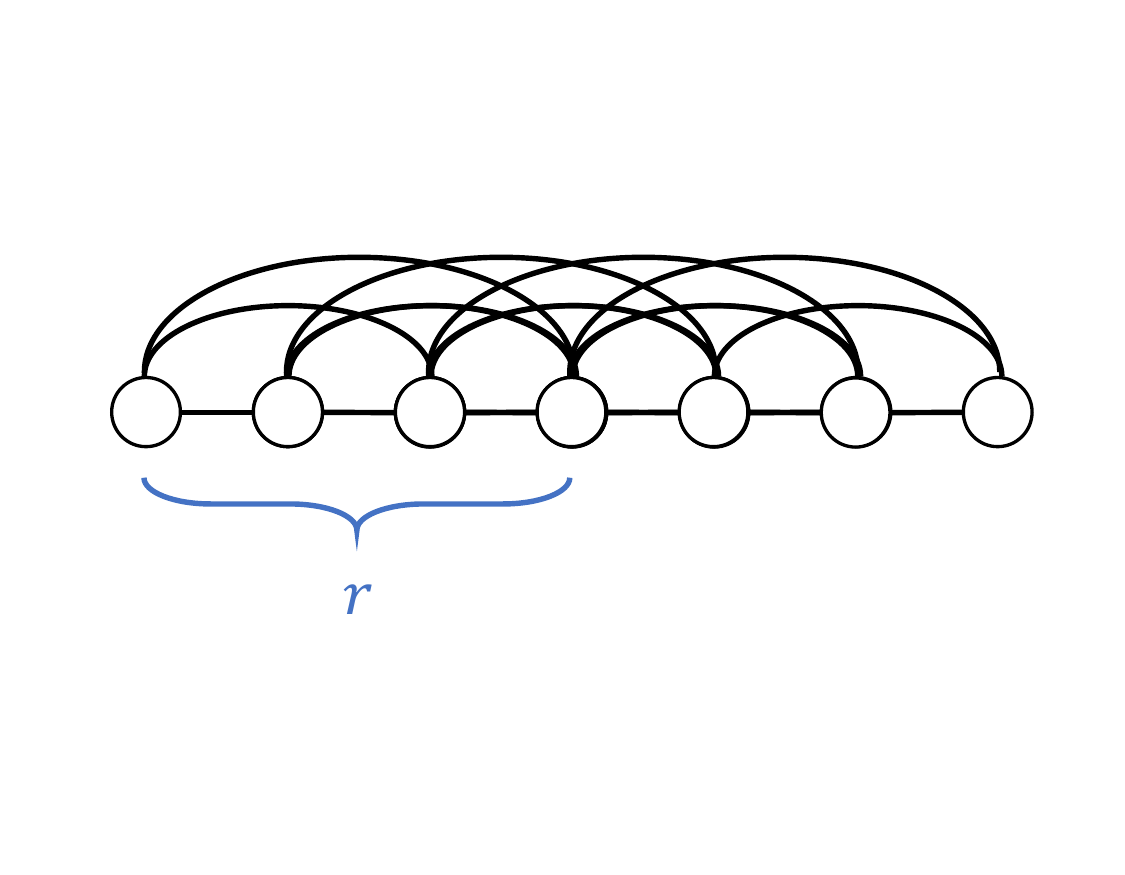}\includegraphics[width=0.4\textwidth]{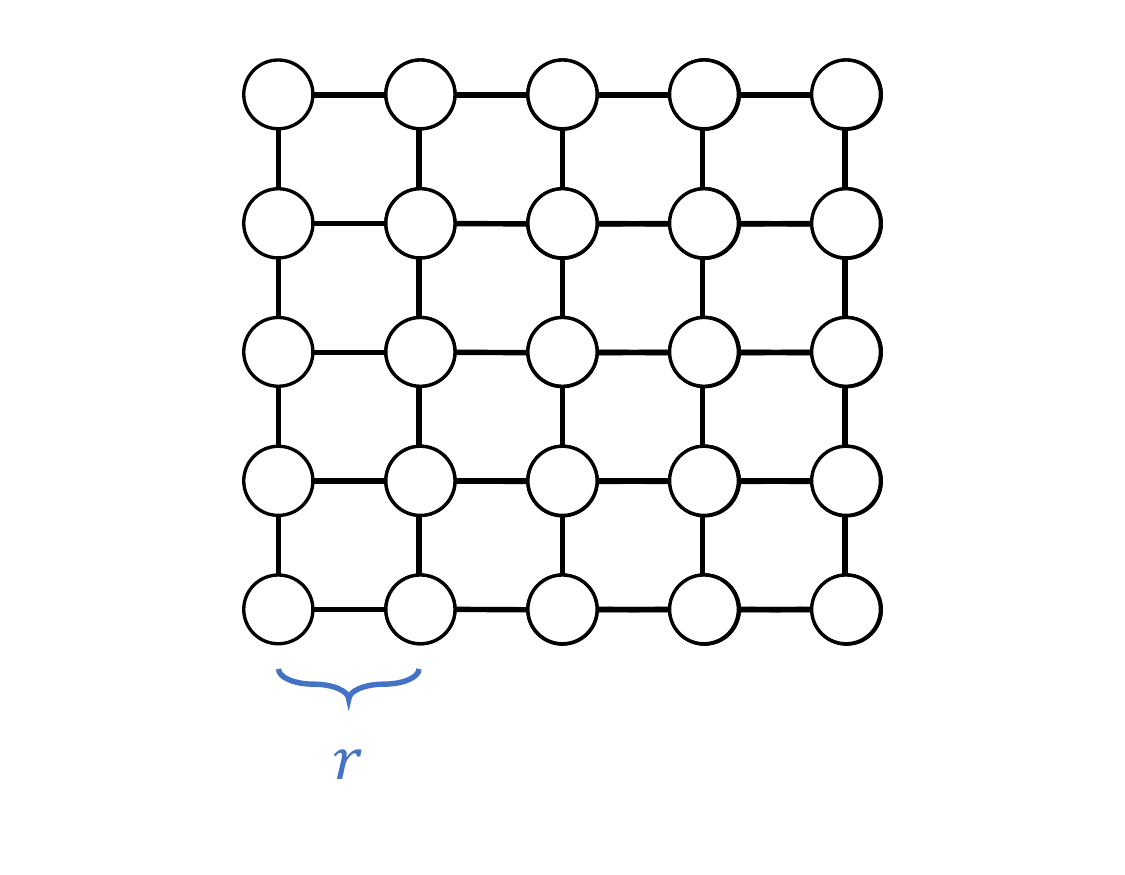}
\par\end{centering}
\caption{\label{fig:grids}A visualization of graphs with locality. Left: $\protect\gridone(n=7,r=3,p=1)$.
Right: $\protect\gridtwo(n=25,r=1,p=1)$ . }
\end{figure}

\subsection{Main contributions }

Our main contributions in this work are two-fold. 
\begin{enumerate}
\item For general graphs, we show that the classical maximum likelihood
estimation (MLE) method, with sufficiently many samples, achieves
an entrywise estimation error that matches the \CramerRao lower bound,
up to a factor of $\tilde{O}(\sqrt{\ke})$ and certain residual terms.
More specifically, denoting $\bthetaMLE$ as the MLE solution, we
have for all $1\le k<\ell\le n$,
\[
\Big|(\thetaMLE_{k}-\thetaMLE_{\ell})-(\thetastar_{k}-\thetastar_{\ell})\Big|\le\tilde{O}\Big(\sqrt{\ke\Omegakl(\Lz)}\Big)+\text{residual terms},
\]
where $\Lz=\nabla^{2}\Lcal(\bthetastar)$ is a graph Laplacian matrix,
and $\Omegakl(\Lz)=(\ek-\el)^{\top}\Lz^{\dagger}(\ek-\el)$ can be
interpreted as the effective resistance between a pair of nodes $(k,\ell)$
in a weighted graph; see Theorem~\ref{thm:general_graph_MLE} for
the formal statement. Note that we only require $\ke$ to be finite,
without any assumption on the global dynamic range $\kappa$ defined
in (\ref{eq:def_kappa_ke}); as we will see, this implies potential
disadvantages of the spectral method, another classical algorithm
for the BTL model. 
\item For graphs with locality, in particular $\gridone/\gridtwo(n,r,p)$,
we readily obtain guarantees for MLE by specifying the effective resistances
in the previous result; however, it remains an open problem to prove
that MLE achieves such guarantees in the \emph{sparsest regime}, namely
$\Lij\ge1$, and $rp\ge\tilde{\Omega}(1)$ for $\gridone$ or $r^{2}p\ge\tilde{\Omega}(1)$
for $\gridtwo$. We further explore divide-and-conquer algorithms,
which not only enjoy certain computational benefits, but also achieve
near-optimal entrywise estimation errors in the sparsest regime; specifically,
assuming that $\ke\lesssim1$ and $\Lij=L$ for all $(i,j)\in E$,
the output solution $\btheta$ satisfies for all $1\le k<\ell\le n$,
\[
\Big|(\theta_{k}-\theta_{\ell})-(\thetastar_{k}-\thetastar_{\ell})\Big|\le\begin{cases}
\tilde{O}\bigg(\sqrt{\frac{n}{r^{2}}+1}\sqrt{\frac{1}{rpL}}\bigg)+\text{residual terms} & \text{for }\gridone(n,r,p),\\
\tilde{O}\bigg(\sqrt{\frac{\log n}{r^{2}}+1}\sqrt{\frac{1}{r^{2}pL}}\bigg)+\text{residual terms} & \text{for }\gridtwo(n,r,p).
\end{cases}
\]
One surprising implication is that \emph{locality does not hurt},
as long as the radius $r$ satisfies $r\gtrsim\sqrt{n}$ for $\gridone$,
or $r\gtrsim\sqrt{\log n}$ for $\gridtwo$; in other words, comparing
the scores between a pair of items that are far apart in the graph
is nearly as easy as comparing a pair of nearby items. 
\end{enumerate}
As a byproduct of our analysis, we design two iterative methods for
solving the MLE optimization problem, namely preconditioned gradient
descent (Remark \ref{rem:PrecondGD}) and projected gradient descent
with re-parameterization (Algorithm~\ref{alg:mle_distrov}), which
enjoy fast convergence. In contrast, standard methods for finding
the MLE solution (including gradient descent, coordinate descent,
minorization-maximization, etc.) are decentralized in nature, and
thus doomed to suffer from slow convergence especially in graphs with
locality. Finally, our theoretical studies are validated with numerical
experiments.

\subsection{Related works \label{subsec:related_works}}

There has been a vast literature on ranking from pairwise comparisons
in the past decades; see e.g.~\cite{ford1957solution,yan2012sparse,chen2015spectral,chen2019spectral,han2020asymptotic,gao2021uncertainty,chen2022optimal,chen2022partial,bong2022generalized,hendrickx2019graph,hendrickx2020minimax,hunter2004mm,li2022general,liu2022lagrangian,negahban2012iterative,shah2016,simons1999asymptotics,vojnovic2019accelerated,wang2020stretching}
and the references therein. In the following, we focus our discussion
on the papers that are most related to our work.

Two classical algorithms for the Bradley-Terry-Luce (BTL) model, namely
maximum likelihood estimation (MLE) and the spectral method (i.e.~Rank
Centrality \cite{negahban2012iterative}), have been intensively studied.
For example, \cite{chen2019spectral} considered \ErdosRenyi graphs,
and showed that both regularized MLE and the spectral method achieve
optimal entrywise estimation errors with the sparsest samples. In
particular, the statistical errors of MLE were analyzed based on the
dynamics of an iterative optimization algorithm; this technique will
inspire our analysis for MLE in general graphs. \cite{gao2021uncertainty}
also considered \ErdosRenyi graphs, and studied uncertainty quantification
for both MLE and the spectral method in the sparsest regime; more
specifically, it was shown that the error vector $\btheta-\bthetastar$
(where $\btheta$ is the solution of either algorithm) can be decomposed
into the sum of a main term and a low-order term, where the main term
can be written explicitly as a linear combination of independent noise
terms $\{\yij-\sigmoid(\thetastar_{i}-\thetastar_{j})\}_{(i,j)\in E,i<j}$.
Other aspects of the MLE solution, e.g.~its existence and uniqueness
\cite{ford1957solution,simons1999asymptotics,yan2012sparse,han2020asymptotic,bong2022generalized},
bias \cite{wang2020stretching}, asymptotic normality \cite{wu2022asymptotic,liu2022lagrangian},
etc., have also been studied in the literature.

Regarding $\ell_{2}$ and $\linf$ estimation errors in the BTL model
with general graphs, results from existing works (e.g.~\cite{shah2016,li2022general})
mostly depend on spectral gaps of certain graph Laplacian matrices,
which fail to capture the correct error rates. Two exceptions are
\cite{hendrickx2019graph,hendrickx2020minimax}, which identified
effective resistances \cite{doyle1984random,chandra1996electrical,ellens2011effective}
as the relevant parameters for estimation errors in the BTL model.
These works studied (weighted) least-squares algorithms with theoretical
guarantees on some variant of $\ell_{2}$ errors, although their analysis
can be readily adapted to give $\linf$ error bounds. On the downside,
(weighted) least-squares algorithms fail to work with sparse samples
(not just in theory but also in practice), because such algorithms
require calculating $\sigmoid^{-1}(\yij)$ explicitly for each edge
$(i,j)\in E$, which is feasible only when there are sufficiently
many samples on each edge (so that $\yij\notin\{0,1\}$). 

\subsection{Notation}

Given a graph $G=(V,E)$, let $\Ncali=\{j\in V:(i,j)\in E\}$ denote
the set of neighbors of node $i$. Let $\ei\in\R^{n}$ denote the
$n$-dimensional one-hot vector, with value $1$ on the $i$-th entry
and $0$ elsewhere. For a vector $\bx$ and a symmetric matrix $\bM\succcurlyeq\mathbf{0}$,
define the norm $\|\bx\|_{\bM}\coloneqq\sqrt{\bx^{\top}\bM\bx}$.
For a symmetric matrix $\bM$ with a compact eigendecomposition $\bM=\bU\bm{\Lambda}\bU^{\top}$,
its pseudo-inverse is defined by $\bM^{\dagger}=\bU\bm{\Lambda}^{-1}\bU^{\top}$.
For a vector $\bx$, let $\exp(\bx),\log(\bx)$, etc.~be vectors
of the same size, where the operation is applied to $\bx$ in an entrywise
manner. The notion of $f_{n}\lesssim g_{n}$ or $f_{n}=O(g_{n})$
means there exists a universal constant $C>0$ such that $f_{n}\le Cg_{n}$
for any $n\ge1$; $f_{n}\gtrsim g_{n}$ or $f_{n}=\Omega(g_{n})$
is equivalent to $g_{n}\lesssim f_{n}$; $f_{n}\asymp g_{n}$ or $f_{n}=\Theta(g_{n})$
means $f_{n}\gtrsim g_{n}$ and $g_{n}\gtrsim f_{n}$ both hold. The
notion of $\tilde{O},\tilde{\Omega},\tilde{\Theta}$ has the same
meaning as $O,\Omega,\Theta$, except that logarithmic terms are hidden.

\section{Preliminaries}

\subsection{Classical algorithms for the BTL model}

The gold-standard method for score estimation in the BTL model is
maximum likelihood estimation (MLE), which is summarized in Algorithm~\ref{alg:mle}.
It aims to find a solution $\btheta$ that minimizes the negative
log-likelihood function $\Lcal(\btheta)$, which is a convex optimization
problem \cite{chen2019spectral}:
\begin{align}
\min_{\btheta\in\R^{n}:\btheta^{\top}\vone_{n}=0}\quad\Lcal(\btheta) & \coloneqq\sumE\sumL\Big(-\yijl(\theta_{i}-\theta_{j})+\log(1+e^{\theta_{i}-\theta_{j}})\Big)\nonumber \\
 & =\sumE\Lij\Big(-\yij(\theta_{i}-\theta_{j})+\log(1+e^{\theta_{i}-\theta_{j}})\Big).\label{eq:def_loss}
\end{align}
The first form of $\Lcal(\btheta)$ will be useful for our theoretical
analysis, while the second form is more efficient for the practical
implementation of MLE. The gradient and Hessian of $\Lcal(\btheta)$
are
\begin{subequations}
\label{eq:gradient_Hessian}
\begin{align}
\nabla\Lcal(\btheta) & =\sumE\sumL\Big(\sigmoid(\theta_{i}-\theta_{j})-\yijl\Big)(\be_{i}-\be_{j})=\sumE\Lij\Big(\sigmoid(\theta_{i}-\theta_{j})-\yij\Big)(\be_{i}-\be_{j}),\label{eq:def_gradient}\\
\nabla^{2}\Lcal(\btheta) & =\sumE\Lij\sigmoid'(\theta_{i}-\theta_{j})(\be_{i}-\be_{j})(\be_{i}-\be_{j})^{\top}.\label{eq:def_Hessian}
\end{align}
\end{subequations}
Below is a classical condition ensuring that the loss function (\ref{eq:def_loss})
admits a unique and finite minimizer.
\begin{prop}
[\cite{ford1957solution}, stated with our notation] \label{prop:mle_existence}The
optimization problem (\ref{eq:def_loss}) admits a unique and finite
minimizer, if and only if the following holds: for any disjoint partition
$\Omega_{1}\cup\Omega_{2}=\{1,\dots,n\}$ with $|\Omega_{1}|,|\Omega_{2}|\ge1$,
there exist some $(i,j)\in(\Omega_{1}\times\Omega_{2})\cap E$ such
that $\yijl=1$ for some $1\le\ell\le\Lij$.
\end{prop}
Besides MLE, another popular algorithm is the spectral method, which
is also called Rank Centrality \cite{negahban2012iterative}. It computes
the stationary distribution $\bpi$ of a stochastic matrix $\bP\in\R^{n\times n}$
(i.e.~$\bpi^{\top}=\bpi^{\top}\bP$) defined as follows:
\[
\text{for each }i,\quad P_{i,j}=\begin{cases}
\frac{1}{d}y_{j,i} & \text{if }j\in\Ncali,\\
1-\frac{1}{d}\sum_{k\in\Ncali}y_{k,i} & \text{if }j=i,\\
0 & \text{otherwise;}
\end{cases}
\]
here, the parameter $d>0$ is chosen such that $\bP$ has nonnegative
entries. With infinite samples, one can check that $\bpistar=\exp(\bthetastar)$
is exactly the stationary distribution of $\bP$. Therefore, with
sufficiently many samples, $\bpi$ is a good estimation of $\bpistar$,
and thus $\log\bpi$ serves as an estimation of $\bthetastar$. This
method is summarized in Algorithm~\ref{alg:spectral}.

\begin{algorithm}[tbp] 
\DontPrintSemicolon 
\caption{Maximum likelihood estimation} \label{alg:mle} 
{\bf Input:} {graph $G=(V,E)$, data $\{\yij,\Lij\}_{(i,j)\in E,i<j}$.} \\
Solve the following optimization problem:
$$
\min_{\btheta\in\R^{n}:\btheta^{\top}\vone_{n}=0}\quad\Lcal(\btheta)=\sumE\Lij\Big(-\yij(\theta_{i}-\theta_{j})+\log(1+e^{\theta_{i}-\theta_{j}})\Big).
$$ \\
{\bf Output:} the solution $\btheta \in \R^{n}$. 
\end{algorithm}

\begin{algorithm}[tbp] 
\DontPrintSemicolon 
\caption{The spectral method} \label{alg:spectral} 
{\bf Input:} {graph $G=(V,E)$, data $\{\yij\}_{(i,j)\in E,i<j}$, parameter $d$.} \\
Construct a matrix $\bP \in \R^{n \times n}$ as follows:
\begin{equation*}
\text{for each }1\le i\le n,\quad P_{i,j}=
\begin{cases}
\frac{1}{d}y_{j,i} & \text{if }j\in\Ncali,\\
1-\frac{1}{d}\sum_{k\in\Ncali}y_{k,i} & \text{if }j=i,\\
0 & \text{otherwise}.
\end{cases}
\end{equation*} \\
Let $\bpi \in \R^n$ be the stationary distribution of $\bP$, and $\btheta = \log \bpi$. \\  
{\bf Output:} $\btheta \in \R^{n}$. 
\end{algorithm}

\subsection{Graph Laplacian and effective resistances}

Consider a connected undirected graph $G=(V,E)$ with a weighted graph
Laplacian
\[
\bL\coloneqq\bD-\bA=\sumE\Cij(\ei-\ej)(\ei-\ej)^{\top},
\]
where $\bA$ is a weighted adjacency matrix, $\bD$ is a diagonal
matrix containing the degrees of the nodes, and $\Cij>0$ is the weight
of edge $(i,j)$. It is well known that $\bL$ has rank $n-1$, satisfying
$\bL\cdot\vone_{n}=\mathbf{0}$ and $\bL^{\dagger}\cdot\vone_{n}=\mathbf{0}$,
where $\bL^{\dagger}$ denotes the pseudo-inverse of $\bL$. The \emph{effective
resistance} (with respect to $\bL$) between any pair of nodes $(k,\ell)$
is defined by 
\begin{equation}
\Omegakl(\bL)\coloneqq(\ek-\el)^{\top}\bL^{\dagger}(\ek-\el).\label{eq:def_Omegakl}
\end{equation}
The physical interpretation is as follows. The graph can be regarded
as a electric network, and the weight $\Cij$ represents the conductance,
or inverse of the resistance $R_{i,j}$, on the edge $(i,j)\in E$.
According to Ohm's law, for any pair of connected nodes $(i,j)\in E$,
the voltages $v_{i},v_{j}$ on these nodes and the current (or electric
flow) $I_{i,j}$ on this edge satisfies $v_{i}-v_{j}=I_{i,j}\Rij=I_{i,j}/\Cij$.
In matrix form, this becomes 
\[
\bc=\bL\bv,\quad\text{or}\quad\bv=\bL^{\dagger}\bc,
\]
where $\bv\in\R^{n}$ denotes the voltages, and $\bc\in\R^{n}$ denotes
the external currents. Therefore, the effective resistance $\Omegakl(\bL)$
is simply the difference of voltages $v_{k}-v_{\ell}$ when one unit
of external current flows into node $k$ and out of node $\ell$,
namely $\bc=\ek-\el$.

Let us collect a few well-known properties of effective resistances,
which are standard in graph theory and circuit theory \cite{doyle1984random,ellens2011effective,lyons2017probability,spielman2019spectral}. 
\begin{fact}
\label{fact:resistances}Effective resistances (\ref{eq:def_Omegakl})
satisfy the following properties: %
\end{fact}
\begin{itemize}
\item (Series/Parallel Law) Consider $m$ resistors with resistances $R_{1},\dots,R_{m}$
and conductances $C_{1},\dots,C_{m}$, where $C_{i}=1/R_{i}$. If
they are connected in series, the effective resistance between two
endpoints is $\Omega=\sum_{i=1}^{m}R_{i}$. If they are connected
in parallel, the effective resistance is $\Omega=1/C$, where $C=\sum_{i=1}^{m}C_{i}$.
(See Figure~\ref{fig:resistance_series_parallel} for a visualization.)
\item (Rayleigh\textquoteright s Monotonicity Law) Increasing (resp.~decreasing)
the resistances on any edges of a resistance network will never decrease
(resp.~increase) the effective resistance between any pair of nodes. 
\item (Triangle inequality) For three nodes $a,b,c\in V$, the effective
resistances satisfy $\Omega_{a,b}\le\Omega_{a,c}+\Omega_{c,b}$.
\end{itemize}
\begin{figure}
\begin{centering}
\includegraphics[width=0.6\textwidth]{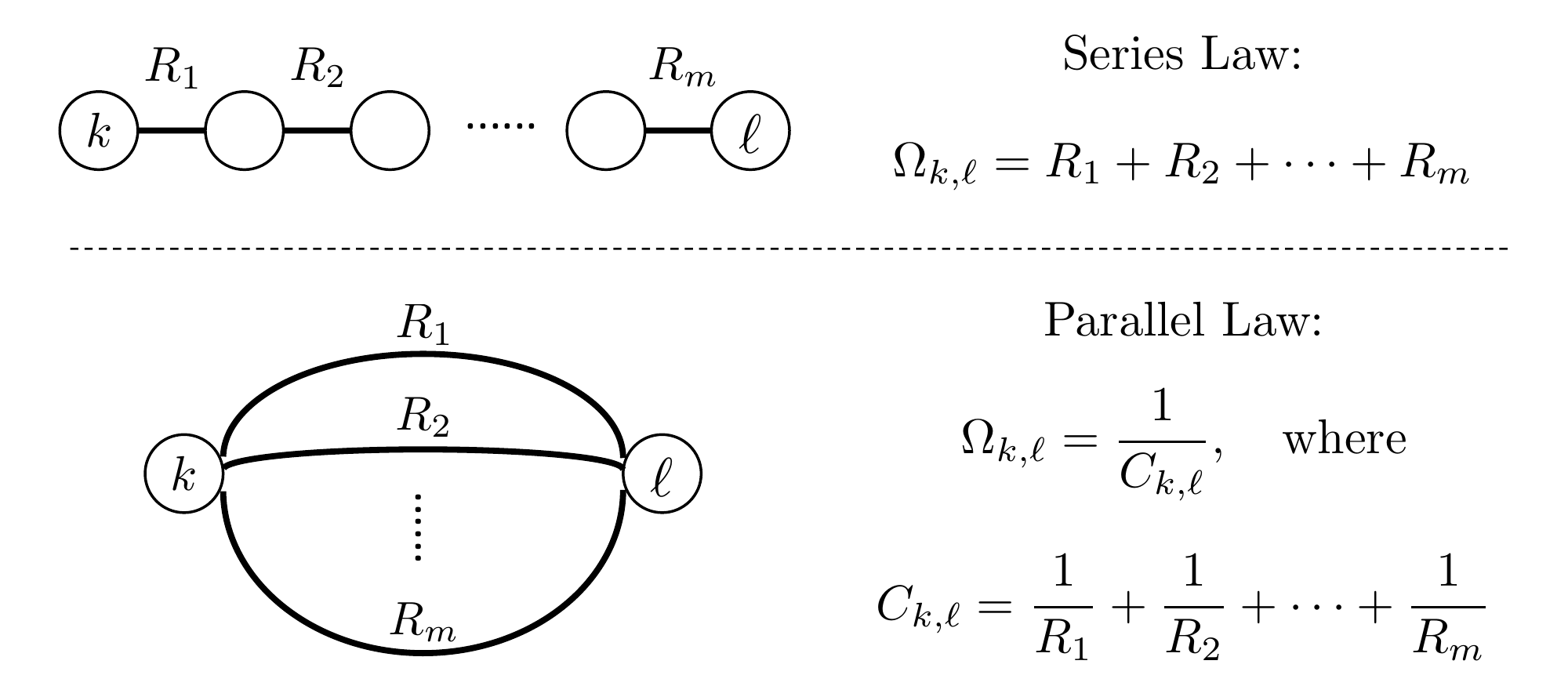}
\par\end{centering}
\caption{\label{fig:resistance_series_parallel}Series/Parallel Law for graph
resistances.}
\end{figure}

As a warmup exercise, let us utilize some of these properties to prove
an upper bound for the effective resistances of a complete graph.
\begin{fact}
\label{fact:resistance_complete_graph}Let $G=(V,E)$ be a complete
graph with $n$ nodes, and $\bL=\sum_{1\le i<j\le n}\Cij(\ei-\ej)(\ei-\ej)^{\top}$
be its weighted graph Laplacian, where $\Cij\asymp1$. Then $\Omegakl(\bL)\lesssim1/n$
for all $1\le k<\ell\le n$.
\end{fact}
\begin{proof}
Given $1\le k<\ell\le n$, consider a graph $G'=(V,E')$ with $E'=\{(i,j):1\le i<j\le n,i=k\text{ or }j=\ell\}$,
and denote $\bL'$ as its graph Laplacian. Since $G'$ can be obtained
from $G$ by increasing the resistances on edges $E\backslash E'$
to infinity, we have $\Omegakl(\bL)\le\Omegakl(\bL')$ according to
Rayleigh\textquoteright s Monotonicity Law. Note that in $G'$, nodes
$k$ and $\ell$ are connected by $\Theta(n)$ parallel paths, each
with resistance $O(1)$; therefore, we have $\Omegakl(\bL')\lesssim1/n$
by the Parallel Law.
\end{proof}

\subsection{Error metrics \label{subsec:error_metrics}}

Denote the error vector as $\bdelta=\btheta-\bthetastar$, where $\btheta\in\R^{n}$
is an estimate of the ground-truth scores $\bthetastar$. To measure
the estimation error, metrics commonly used in the literature include
the $\linf$ error $\|\bdelta\|_{\infty}$ and the $\ell_{2}$ error
$\|\bdelta\|_{2}$. One thing to note is that, in the BTL model, the
scores can only be identified up to a global shift; in other words,
the score vector $\btheta$ is essentially equivalent to $\btheta+c\vone_{n}$
for any scalar value $c\in\R$. Therefore, these error metrics are
only meaningful under additional assumptions, such as $\bdelta^{\top}\vone_{n}=\sum_{1\le i\le n}\delta_{i}=0$.

In this work, we propose to measure the estimation errors in a \emph{pairwise}
manner, using
\[
\Big\{|\delta_{k}-\delta_{\ell}|\Big\}_{1\le k<\ell\le n}\quad\text{and}\quad\max_{1\le k<\ell\le n}|\delta_{k}-\delta_{\ell}|.
\]
These metrics come with various benefits. (1) They are insensitive
to the global shift, i.e.~$\btheta$ and $\btheta+c\vone_{n}$ have
exactly the same errors. (2) One can obtain a more complete picture
of the estimation errors by measuring the pairwise error $|\delta_{k}-\delta_{\ell}|$
for every pair of items $(k,\ell)$, which helps to (say) distinguish
the hardness of comparing the scores for a pair of items that are
nearby or far apart in the graph. (3) As we will see in our analysis,
there is a natural correspondance between the pairwise errors $\{|\delta_{k}-\delta_{\ell}|\}$
and the effective resistances $\{\Omegakl\}$. (4) With these pairwise
errors in place, the standard $\linf$ and $\ell_{2}$ errors can
be easily derived: assuming without loss of generality that $\bdelta^{\top}\vonen=0$,
one has
\[
\|\bdelta\|_{\infty}\le\max_{1\le k<\ell\le n}|\delta_{k}-\delta_{\ell}|\le2\|\bdelta\|_{\infty},\quad\text{and}
\]
\begin{equation}
\|\bdelta\|_{2}^{2}=\frac{1}{n}\bdelta^{\top}(n\bI_{n}-\vone_{n}\vone_{n}^{\top})\bdelta=\frac{1}{n}\bdelta^{\top}\sum_{k<\ell}(\ek-\el)(\ek-\el)^{\top}\bdelta=\frac{1}{n}\sum_{k<\ell}(\delta_{k}-\delta_{\ell})^{2}.\label{eq:errors_pairwise_and_L2}
\end{equation}

\section{Classical algorithms \label{sec:classical_algorithms}}

In this section, we introduce novel theoretical analyses for $\linf$
estimation errors of MLE in general graphs (Section \ref{subsec:MLE_general_graphs})
and graphs with locality (Section \ref{subsec:MLE_graphs_with_locality}),
along with detailed proofs (Sections \ref{subsec:proof_mle_general}
and \ref{subsec:proof_mle_locality}). We also discuss in Section
\ref{subsec:does_spectral_work} why the spectral method might fail
to estimate $\bthetastar$ reliably in some cases where MLE succeeds. 

\subsection{MLE in general graphs\label{subsec:MLE_general_graphs}}

Our main results and analysis crucially rely on the following graph
Laplacian $\Lz$, defined by the Hessian (\ref{eq:def_Hessian}) of
$\Lcal(\btheta)$ at the ground truth $\bthetastar$:
\begin{align}
\Lz & \coloneqq\nabla^{2}\Lcal(\bthetastar)=\sumE\Lij\zij(\be_{i}-\be_{j})(\be_{i}-\be_{j})^{\top},\quad\text{where}\quad\zij\coloneqq\sigmoid'(\thetastar_{i}-\thetastar_{j})>0.\label{eq:def_Lz}
\end{align}
The following theorem provides high-probability guarantees for the
$\linf$ errors of MLE. Note that the randomness is only over the
samples $\{\yijl\}$.
\begin{thm}
\label{thm:general_graph_MLE}Consider the model and assumptions in
Section~\ref{sec:intro}. There exist some universal constant $C_{0}>0$
such that the following holds for any fixed $0<\delta<0.5$. Let $\{\Bkl,\Qkl\}_{1\le k<\ell\le n}$
be valid upper bounds of the following quantities:
\begin{equation}
\Bkl\ge C_{0}\sqrt{\Omegakl(\Lz)\cdot\ke\log\big(\frac{n}{\delta}\big)},\quad\Qkl\ge\sumE\Lij\Bij^{2}\big|(\ek-\el)^{\top}\Lzinv(\be_{i}-\be_{j})\big|,\quad1\le k,\ell\le n.\label{eq:def_Bkl_Qkl}
\end{equation}
If $\{\Bkl,\Qkl\}$ satisfy that
\begin{equation}
\Qkl\le4\Bkl\quad\text{for all }(k,\ell)\in E,\label{eq:asp_small_Bij}
\end{equation}
then with probability at least $1-\delta$, the MLE optimization problem
(\ref{eq:def_loss}) admits a unique and finite minimizer $\bthetaMLE$,
which satisfies the following: for all $1\le k<\ell\le n$,
\begin{equation}
\Big|(\thetaMLE_{k}-\thetaMLE_{\ell})-(\thetastar_{k}-\thetastar_{\ell})\Big|\le\begin{cases}
\Bkl, & (k,\ell)\in E,\\
\frac{1}{2}\Bkl+\frac{1}{8}\Qkl, & (k,\ell)\notin E.
\end{cases}\label{eq:MLE_error_bound}
\end{equation}
\end{thm}
\begin{proof}
[Proof outline]The key idea is to consider the preconditioned gradient
descent (PrecondGD) iteration 
\[
\bthetatp=\bthetat-\eta\Lzinv\nabla\Lcal(\bthetat),\quad t=0,1,2,\dots,
\]
with the initialization $\btheta^{0}=\bthetastar$. Our analysis consists
of two major components. (1)~With high probability, the PrecondGD
iterations $\{\bthetat\}$ stay close to $\bthetastar$ for all $t\ge0$,
provided a sufficiently small step size $\eta$. (2)~Conditioned
on this, it can be guaranteed that the loss function $\Lcal$ admits
a unique and finite minimizer $\bthetaMLE$, and $\{\bthetat\}$ converge
linearly to it, due to the strong convexity and smoothness of $\Lcal$
(when constrained to a bounded convex set within the subspace orthogonal
to $\vonen$). See Section~\ref{subsec:proof_mle_general} for the
complete proof.
\end{proof}
Theorem \ref{thm:general_graph_MLE} provides a simple and unified
guarantee of entrywise errors for vanilla MLE (without regularization
\cite{chen2019spectral}), which enjoys the following advantages: 
\begin{itemize}
\item It allows general connected graphs and general numbers of samples
$\{\Lij\}$. 
\item When sample sizes are sufficiently large, the square of the main error
term $\Bkl$ in (\ref{eq:MLE_error_bound}) matches the \CramerRao
lower bound $(\ek-\el)^{\top}\Lzinv(\ek-\el)=\Omegakl(\Lz)$, up to
a factor of $\tilde{O}(\ke)$. Similar to \cite{hendrickx2019graph,hendrickx2020minimax},
Theorem \ref{thm:general_graph_MLE} points out that effective resistances,
rather than spectral gaps \cite{shah2016,li2022general} or others,
are the relevant parameters for estimation errors in the BTL model
with general graphs.
\item Theorem \ref{thm:general_graph_MLE} indicates that MLE works as long
as $\ke$ defined in (\ref{eq:def_kappa_ke}) is bounded, which can
be much smaller than the dynamic range $\kappa$. A similar observation
has been made by \cite{bong2022generalized}. This implies a potential
advantage of MLE over the spectral method in cases with a small $\ke$
but large $\kappa$; we will elaborate on this in Section \ref{subsec:does_spectral_work}.
\end{itemize}
\begin{rem}
The condition (\ref{eq:asp_small_Bij}) is guaranteed to hold for
sufficiently large sample sizes; moreover, in the error bound (\ref{eq:MLE_error_bound}),
$\Bkl$ is the main term, while $\Qkl$ is a ``low-order'' term.
To see these, we first recall the definition of $\Lz$ in (\ref{eq:def_Lz}),
which depends linearly on the $\{\Lij\}$ factors (the numbers of
samples on the edges). For simplicity, assume that $\Lij=L$ for all
$(i,j)\in E$; then, with everything else fixed, $\Bkl\asymp\sqrt{1/L}$
while $\Qkl\asymp1/L$, which leads to the interpretations above.
Note that the condition (\ref{eq:asp_small_Bij}) might hold even
for the minimal sample size $L=1$, if we have sharp upper bounds
of $\{\Qkl\}$ for certain graphs of interests. See Example~\ref{exa:example_mle}
in the appendix for the implications of Theorem \ref{thm:general_graph_MLE}
on some special graphs.
\end{rem}
\begin{rem}
Upper bounds for standard $\linf$ and $\ell_{2}$ errors can be derived
easily from Theorem \ref{thm:general_graph_MLE}, according to Section
\ref{subsec:error_metrics}. For $\ell_{\infty}$ errors, since $\bdelta\coloneqq\bthetaMLE-\bthetastar$
satisfies $\bdelta^{\top}\vone_{n}=0$, it holds that 
\[
\|\bdelta\|_{\infty}\le\max_{k\neq\ell}|\delta_{k}-\delta_{\ell}|=\max_{k\neq\ell}\big|(\thetaMLE_{k}-\thetaMLE_{\ell})-(\thetastar_{k}-\thetastar_{\ell})\big|.
\]
For $\ell_{2}$ errors, recall from (\ref{eq:errors_pairwise_and_L2})
that $\|\bdelta\|_{2}^{2}=\frac{1}{n}\sum_{k<\ell}(\delta_{k}-\delta_{\ell})^{2}$.
Ignoring the low-order terms $\{\Qkl\}$ and logarithmic terms for
simplicity, we have $(\delta_{k}-\delta_{\ell})^{2}\lesssim\Bkl^{2}\lesssim\Omegakl(\Lz)\ke$,
and thus
\[
\|\bdelta\|_{2}^{2}=\frac{1}{n}\sum_{k<\ell}(\delta_{k}-\delta_{\ell})^{2}\lesssim\frac{1}{n}\sum_{k<\ell}\Bkl^{2}\lesssim\frac{1}{n}\sum_{k<\ell}\Omegakl(\Lz)\ke.
\]
Let us rewrite the total resistance \cite{ghosh2008minimizing} as
\[
\sum_{k<\ell}\Omegakl(\Lz)=\sum_{k<\ell}(\ek-\el)^{\top}\Lzinv(\ek-\el)=\Tr\Big(\Lzinv\sum_{k<\ell}(\ek-\el)(\ek-\el)^{\top}\Big)=n\Tr(\Lzinv).
\]
Putting things together, we have a concise (though informal) bound
on the $\ell_{2}$ error: $\|\bdelta\|_{2}\lesssim\sqrt{\Tr(\Lzinv)\ke}$.
\end{rem}
\begin{rem}
\label{rem:PrecondGD}Our theoretical analysis suggests a practical
PrecondGD method for solving MLE. While calculating the preconditioner
$\Lzinv$ requires knowing $\bthetastar$ (which is of course unrealistic),
we can use some surrogate matrix as the preconditioner in practice.
For example, one might simply ignore the $\{\zij\}$ terms in the
definition (\ref{eq:def_Lz}) of $\Lz$, which leads to the implementable
preconditioner $\LG^{\dagger}$ with $\LG\coloneqq\sumE\Lij(\be_{i}-\be_{j})(\be_{i}-\be_{j})^{\top}$.
Note that $\LG$ satisfies $\Lz\preccurlyeq\LG\preccurlyeq4\ke\cdot\Lz$.
Suppose that $\ke\lesssim1$ is small and sample sizes $\{\Lij\}$
are large, so that $\bthetaMLE$ and a naive initialization (say $\btheta^{0}=\mathbf{0}$)
are both close to $\bthetastar$ in terms of pairwise entrywise errors (see Section~\ref{subsec:error_metrics}).
Then, Proposition~\ref{prop:PrecondGD_convergence} suggests (informally)
that PrecondGD with step size $\eta\asymp1$ will converge linearly
to $\bthetaMLE$, at a rate that is potentially much faster than usual
first-order methods (such as gradient descent, coordinate descent,
minorization-maximization \cite{hunter2004mm,vojnovic2019accelerated},
etc.), especially when the graph Laplacian is ill-conditioned; this
will be validated numerically in Section~\ref{subsec:exp_solve_mle}.
As for the computational complexity, the PrecondGD update $\LG^{\dagger}\nabla\Lcal(\bthetat)$
can be computed approximately in $\tilde{O}(|E|)$ time (at least
in theory), if we utilize the state-of-the-art solvers for Laplacian
linear systems \cite{spielman2004nearly,vishnoi2013lx}. 
\end{rem}

\subsection{MLE in graphs with locality\label{subsec:MLE_graphs_with_locality}}

Let us apply Theorem~\ref{thm:general_graph_MLE} to $\gridone(n,r,p)$
and $\gridtwo(n,r,p)$ with $p=1$, which leads to the following result.
\begin{cor}
\label{cor:MLE_locality}Consider the model and assumptions in Section~\ref{sec:intro}.
Let $G=(V,E)$ be $\gridone(n,r,p)$ or $\gridtwo(n,r,p)$ with $p=1$,
and assume for simplicity that $\Lij=L$ for all $(i,j)\in E$. If\textcolor{red}{{}
${\color{black}rL\gtrsim\ke^{5}n\log\big(n/\delta\big)}$} for 1D
grids, or $r^{2}L\gtrsim\ke^{5}n\log\big(n/\delta\big)$ for 2D grids,
then with probability at least $1-\delta$, the MLE optimization problem
(\ref{eq:def_loss}) admits a unique and finite minimizer $\bthetaMLE$,
which satisfies the following: for all $1\le k<\ell\le n$,
\[
\big|(\thetaMLE_{k}-\thetaMLE_{\ell})-(\thetastar_{k}-\thetastar_{\ell})\big|\lesssim\begin{cases}
\ke\sqrt{(\frac{n}{r^{2}}+1)\frac{1}{rL}\log\big(\frac{n}{\delta}\big)} & \text{for 1D grids,}\\
\ke\sqrt{(\frac{\log n}{r^{2}}+1)\frac{1}{r^{2}L}\log\big(\frac{n}{\delta}\big)} & \text{for 2D grids.}
\end{cases}
\]
\end{cor}
\begin{proof}
[Proof outline]The key step of proving this corollary is to find
valid upper bounds for $\{\Omegakl,\Vkl\}_{1\le k<\ell\le n}$ in
1D/2D grids, where $\Omegakl=\Omegakl(\Lz)$ and
\begin{equation}
\Vaggkl\coloneqq\sumE\Lij\big|(\ei-\ej)^{\top}\Lzinv(\ek-\el)\big|.\label{eq:Vkl}
\end{equation}
Once this is done, we obtain valid upper bounds for $\{\Bkl,\Qkl\}$
defined in (\ref{eq:def_Bkl_Qkl}), and the corollary follows immediately
from Theorem~\ref{thm:general_graph_MLE}. See Section~\ref{subsec:proof_mle_locality}
for the complete proof.
\end{proof}
Two remarks are to follow.
\begin{itemize}
\item One surprising implication of the error bounds in Corollary \ref{cor:MLE_locality}
is that \emph{locality does not hurt} (up to a $\tilde{O}(\ke)$ factor),
as long as the radius satisfies $r\gtrsim\sqrt{n}$ (for $\gridone$)
or $r\gtrsim\sqrt{\log n}$ (for $\gridtwo$); in other words, estimating
the gap of scores for two items that are far apart in the graph is
almost as easy as estimating the gap for two items that are nearby.
\item The major downside of Corollary \ref{cor:MLE_locality} is its requirements
on the sample sizes, which is due to the loose bounds for $\{\Vkl\}$
in our proof. While we conjecture that the relevant parameters $\{\Omegakl,\Vkl\}$
can be controlled sharply via dedicated analysis for many graphs of
interests, unfortunately we fail to do so for $\gridone$ and $\gridtwo$.
Ideally, one wants to prove error bounds for MLE in the \emph{sparsest
regime}, namely $rpL\ge\tilde{\Omega}(1)$ for $\gridone(n,r,p)$,
or $r^{2}pL\ge\tilde{\Omega}(1)$ for $\gridtwo(n,r,p)$; moreover,
for the case of general sampling probability $p$, one would expect
an additional $1/\sqrt{p}$ factor in the entrywise error bounds.
Proving that MLE can achieve these is a valuable direction for future
work. In Section \ref{sec:dc_alg}, we propose alternative algorithms
that can achieve such guarantees.
\end{itemize}

\subsection{Proof of Theorem \ref{thm:general_graph_MLE} \label{subsec:proof_mle_general}}

Before we proceed, let us recall the definitions of $\Lz$ and $\{\zij\}$
in (\ref{eq:def_Lz}), as well as introduce some additional notation
that will be useful. Throughout our analysis, we will use $C_{0},C_{1},C_{2},\dots$
to represent positive universal constants. Let us define the incidence
matrix 
\[
\bB\coloneqq\Big[\dots,\sqrt{\zij}(\be_{i}-\be_{j}),\dots\Big]_{1\le\ell\le\Lij;(i,j)\in E,i<j}\in\R^{n\times L_{\mathsf{total}}},
\]
where ``$1\le\ell\le\Lij$'' in the subscript indicates that the
column $\sqrt{\zij}(\be_{i}-\be_{j})$ is repeated $\Lij$ times,
and $L_{\mathsf{total}}\coloneqq\sumE\Lij$ is the total number of
comparisons. Notice that 
\[
\bB\bB^{\top}=\sumE\Lij\zij(\ei-\ej)(\ei-\ej)^{\top}=\Lz.
\]
In addition, we define the noise terms 
\[
\epsijl\coloneqq\yijl-\sigmoid(\thetastar_{i}-\thetastar_{j}),\quad1\le\ell\le\Lij,(i,j)\in E,i<j,
\]
each with mean zero, variance $\var(\epsijl)=\zij$, and bounded on
$[-1,1]$. Furthermore, let $\bepshat$ be the vector of normalized
independent noise terms, namely
\[
\bepshat\coloneqq\Big[\dots,\epsijl/\sqrt{\zij},\dots\Big]_{1\le\ell\le\Lij;(i,j)\in E,i<j}^{\top}\in\R^{L_{\mathsf{total}}};
\]
each entry of $\bepshat$ has mean zero, variance one, and subgaussian
norm \cite[Section~2.5]{vershynin2018high} $\|\cdot\|_{\psi_{2}}\lesssim1/\sqrt{\zij}\lesssim\sqrt{\ke}$
(because $|\thetastar_{i}-\thetastar_{j}|\le\log\ke$, and thus $\zij=\sigmoid'(\thetastar_{i}-\thetastar_{j})\ge1/(4\ke)$
according to Fact~\ref{fact:sigmoid}).

\paragraph{Step 1: the PrecondGD dynamics}

Consider the dynamics $\{\bthetat\}$ of preconditioned gradient descent
for the loss function $\Lcal(\btheta)$ defined in (\ref{eq:def_loss}),
with the initialization $\btheta^{0}=\bthetastar$. At the $t$-th
PrecondGD iteration, one has
\[
\bthetatp=\bthetat-\eta\Lzinv\nabla\Lcal(\bthetat),
\]
where the gradient can be rewritten as
\begin{align}
\nabla\Lcal(\bthetat) & =\sumE\sumL\Big(\sigmoid(\thetat_{i}-\thetat_{j})-\yijl\Big)(\be_{i}-\be_{j})\nonumber \\
 & =\sumE\sumL\Big(\sigmoid(\thetat_{i}-\thetat_{j})-\sigmoid(\thetastar_{i}-\thetastar_{j})-\epsijl\Big)(\be_{i}-\be_{j})\nonumber \\
 & \overset{{\rm (i)}}{=}\sumE\sumL\Big(\sigmoid'(\thetastar_{i}-\thetastar_{j})(\deltat_{i}-\deltat_{j})+\frac{1}{2}\sigmoid''(\zetatij)(\deltat_{i}-\deltat_{j})^{2}-\epsijl\Big)(\be_{i}-\be_{j})\nonumber \\
 & =\Lz\bdeltat-\bB\bepshat+\underset{\eqqcolon\brt}{\underbrace{\frac{1}{2}\sumE\Lij\sigmoid''(\zetatij)(\deltat_{i}-\deltat_{j})^{2}(\be_{i}-\be_{j})}}\label{eq:def_rt}\\
 & =\Lz\bdeltat-\bB\bepshat+\brt;\nonumber 
\end{align}
here, (i) is due to Taylor's theorem, with $\zetatij$ lying between
$\thetastar_{i}-\thetastar_{j}$ and $\thetat_{i}-\thetat_{j}$ for
each $(i,j)\in E$. 

Recall the assumption that, without loss of generality, $\bthetastar\perp\vone_{n}$.
It is easy to check that $\nabla\Lcal(\bthetat)^{\top}\vonen=0$ and
$\Lzinv\vonen=0$, which implies that $\bthetat,\bthetat-\bthetastar\perp\vone_{n}$
for all $t\ge0$. As a result, everything about the PrecondGD dynamics
happens within the subspace orthogonal to $\vonen$.

Denoting the optimization error at the $t$-th iteration as $\bdeltat\coloneqq\bthetat-\bthetastar$,
we get
\begin{align*}
\bdeltatp & =\bdeltat-\eta\Lzinv(\Lz\bdeltat-\bB\bepshat+\brt)=\bdeltat-\eta(\bdeltat-\Lzinv\bB\bepshat+\Lzinv\brt)=(1-\eta)\bdeltat+\eta\bxi-\eta\Lzinv\brt,
\end{align*}
where we define 
\[
\bxi\coloneqq\Lzinv\bB\bepshat.
\]
Intuitively, one would like to show that $\{\bdeltat\}$ converges
to the main error term $\bxi$, up to a low-order error due to the
residual term $\Lzinv\brt$. In the following steps, we further investigate
both error terms.

\paragraph{Step 1.1: the main error term $\protect\bxi$. }

For any $k\neq\ell$, one has
\[
\xik-\xil=(\ek-\el)^{\top}\bxi=(\ek-\el)^{\top}\Lzinv\bB\bepshat=\langle\bB^{\top}\Lzinv(\ek-\el),\bepshat\rangle,
\]
which is a mean-zero subgaussian random variable with
\begin{align*}
\|\xik-\xil\|_{\psi_{2}}^{2} & \lesssim\big\|\bB^{\top}\Lzinv(\ek-\el)\big\|_{2}^{2}\cdot\ke\\
 & =(\ek-\el)^{\top}\Lzinv\bB\bB^{\top}\Lzinv(\ek-\el)\cdot\ke\\
 & =(\ek-\el)^{\top}\Lzinv(\ek-\el)\cdot\ke\\
 & =\Omegakl(\Lz)\cdot\ke.
\end{align*}
Using the concentration of subgaussian random variables \cite[Section~2.5]{vershynin2018high}
and taking a union bound over all pairs of $(k,\ell)$, we have the
following result: with probability at least $1-\delta$, it holds
that 
\begin{equation}
|\xik-\xil|\le C_{1}\sqrt{\Omegakl(\Lz)\cdot\ke\log\big(\frac{n}{\delta}\big)}\le\frac{1}{2}\Bkl\quad\text{for all }1\le k<\ell\le n,\label{eq:xi_high_prob_error}
\end{equation}
where we recall the definition (\ref{eq:def_Bkl_Qkl}) of $\Bkl\ge C_{0}\sqrt{\Omegakl(\Lz)\cdot\ke\log(\frac{n}{\delta})},$
and assume that $C_{0}>2C_{1}$.%
{} It is worth noting that, for our proof of Theorem \ref{thm:general_graph_MLE},
this step is the only part that involves randomness of the data; everything
else in our proof is deterministic. 

\paragraph{Step 1.2: the residual term $\protect\Lzinv\protect\brt$.}

Recalling the definition (\ref{eq:def_rt}) of $\brt$, we have
\[
\Lzinv\brt=\frac{1}{2}\sumE\Lij\sigmoid''(\zetatij)(\deltat_{i}-\deltat_{j})^{2}\Lzinv(\be_{i}-\be_{j}).
\]
Therefore, for any $k\neq\ell$,
\begin{align*}
\big|(\ek-\el)^{\top}\Lzinv\brt\big| & =\bigg|\frac{1}{2}\sumE\Lij\sigmoid''(\zetatij)(\deltat_{i}-\deltat_{j})^{2}(\ek-\el)^{\top}\Lzinv(\be_{i}-\be_{j})\bigg|\\
 & \le\frac{1}{2}\sumE\Lij\big|\sigmoid''(\zetatij)\big|\cdot(\deltat_{i}-\deltat_{j})^{2}\cdot\big|(\ek-\el)^{\top}\Lzinv(\be_{i}-\be_{j})\big|\\
 & \le\frac{1}{8}\sumE\Lij(\deltat_{i}-\deltat_{j})^{2}\big|(\ek-\el)^{\top}\Lzinv(\be_{i}-\be_{j})\big|,
\end{align*}
where we use the fact that $|\sigmoid''(x)|\le1/4$ for all $x\in\R$.
This has a quadratic dependence on $\{\deltati-\deltatj\}_{(i,j)\in E}$.

\paragraph{Step 2: PrecondGD stays close to the ground truth}

Our analysis in Step 1 leads to the following: for any $k\neq\ell$,
\begin{align}
|\deltatpk-\deltatpl| & \le(1-\eta)|\deltatk-\deltatl|+\eta|\xik-\xil|+\eta\big|(\ek-\el)^{\top}\Lzinv\brt\big|\nonumber \\
 & \le(1-\eta)|\deltatk-\deltatl|+\eta\Big(\frac{1}{2}\Bkl+\frac{1}{8}\sumE\Lij(\deltat_{i}-\deltat_{j})^{2}\big|(\ek-\el)^{\top}\Lzinv(\be_{i}-\be_{j})\big|\Big).\label{eq:entrywise_error_update}
\end{align}

Consider the following induction hypothesis (which obviously holds
true for $\bdelta^{0}=\btheta^{0}-\bthetastar=\boldsymbol{0}$ at
$t=0$): at the $t$-th iteration, 
\begin{equation}
|\deltatk-\deltatl|\le\Bkl\quad\text{for all}\quad(k,\ell)\in E.\label{eq:induction_t_E}
\end{equation}
Then, utilizing the assumption (\ref{eq:asp_small_Bij}), we have
for all $(k,\ell)\in E$,
\begin{align*}
|\deltatpk-\deltatpl| & \le(1-\eta)|\deltatk-\deltatl|+\eta\Big(\frac{1}{2}\Bkl+\frac{1}{8}\sumE\Lij\Bij^{2}\big|(\ek-\el)^{\top}\Lzinv(\be_{i}-\be_{j})\big|\Big)\\
 & \le(1-\eta)\Bkl+\eta\Big(\frac{1}{2}\Bkl+\frac{1}{8}\Qkl\Big)\\
 & \le(1-\eta)\Bkl+\eta\Big(\frac{1}{2}\Bkl+\frac{1}{8}\cdot4\Bkl\Big)=\Bkl.
\end{align*}
Thus, we have proved the induction hypothesis for the $(t+1)$-th
iteration. Notice that this analysis only involves the $(k,\ell)\in E$
pairs. 

Next, consider the remaining $(k,\ell)\notin E$ pairs. Starting from
(\ref{eq:entrywise_error_update}) and utilizing the induction hypothesis
(\ref{eq:induction_t_E}), one has
\begin{align*}
|\deltatpk-\deltatpl| & \le(1-\eta)|\deltatk-\deltatl|+\eta\Big(\frac{1}{2}\Bkl+\frac{1}{8}\sumE\Lij\Bij^{2}\big|(\ek-\el)^{\top}\Lzinv(\be_{i}-\be_{j})\big|\Big)\\
 & \le(1-\eta)|\deltatk-\deltatl|+\eta\Big(\frac{1}{2}\Bkl+\frac{1}{8}\Qkl\Big).
\end{align*}
As a result, $|\deltatk-\deltatl|\le\frac{1}{2}\Bkl+\frac{1}{8}\Qkl$
implies the same upper bound for $|\deltatpk-\deltatpl|$.

\paragraph{Step 3: PrecondGD converges to the MLE solution }

Our result for this step, i.e.~the existence of $\bthetaMLE$ and
the convergence of PrecondGD, is summarized in the following lemma.
We defer its proof to Appendix \ref{subsec:proof_lemma_mle}, which
is based on standard results from convex optimization, as well as
Proposition \ref{prop:mle_existence}.
\begin{lem}
\label{lem:MLE_existence_PGD_convergence}Consider the setting of
Theorem~\ref{thm:general_graph_MLE}. Conditioned on the event (\ref{eq:xi_high_prob_error}),
one has the following:
\end{lem}
\begin{enumerate}
\item The optimization problem (\ref{eq:def_loss}) has a unique and finite
minimizer $\bthetaMLE$; 
\item There exist $0<\alpha_{1}\le\alpha_{2}$ such that the PrecondGD iterations
satisfy $\|\bthetat-\bthetaMLE\|_{\Lz}\le(1-\eta\alpha_{1})^{t}\|\btheta^{0}-\bthetaMLE\|_{\Lz}$
for all $t\ge0$, provided that $0<\eta\le1/\alpha_{2}$.
\end{enumerate}

\paragraph{Summary}

Our analysis so far can be summarized as follows. There exists an
event $\Ecal$ (i.e.~(\ref{eq:xi_high_prob_error}) holding true)
with $\Pr(\Ecal)\ge1-\delta$, such that the PrecondGD iterations
$\{\bthetat\}$ satisfy the following, conditioned on $\Ecal$: first,
$|(\thetat_{k}-\thetat_{\ell})-(\thetastar_{k}-\thetastar_{\ell})|\le\Bkl$
for $(k,\ell)\in E$, and $\Bkl/2+\Qkl/8$ for $(k,\ell)\notin E$
(Step 2); moreover, a unique and finite MLE solution $\bthetaMLE$
exists, and $\|\bthetat-\bthetaMLE\|_{\Lz}\le(1-\eta\alpha_{1})^{t}\|\btheta^{0}-\bthetaMLE\|_{\Lz}$
(Step 3). In other words, the PrecondGD iterations $\{\bthetat\}$
stays close (entrywise) to $\bthetastar$, while converging to $\bthetaMLE$
up to an arbitrarily small error (after sufficiently many iterations).
This finishes our proof of Theorem~\ref{thm:general_graph_MLE}.

\subsection{Proof of Corollary \ref{cor:MLE_locality} \label{subsec:proof_mle_locality}}

Our proof of Corollary \ref{cor:MLE_locality} relies on the following
key lemma. Recall the notation 
\[
\Omegakl=\Omegakl(\Lz)=(\ek-\el)^{\top}\Lzinv(\ek-\el),\quad\Vkl=\sumE\Lij\big|(\ei-\ej)^{\top}\Lzinv(\ek-\el)\big|.
\]

\begin{lem}
\label{lem:locality_resistance}For $\gridone(n,r,p=1)$, we have
\begin{align*}
\Omegakl\lesssim\begin{cases}
\frac{1}{r}\cdot\frac{\ke}{L}, & (k,\ell)\in E,\\
\frac{1}{r}\cdot\frac{\ke}{L}\cdot(\frac{n}{r^{2}}+1) & (k,\ell)\notin E,
\end{cases}\qquad\Vkl\lesssim\begin{cases}
\ke^{1.5}\sqrt{n}, & (k,\ell)\in E,\\
\ke^{1.5}\sqrt{n}\cdot\sqrt{\frac{n}{r^{2}}+1}, & (k,\ell)\notin E;
\end{cases}
\end{align*}
similarly, for $\gridtwo(n,r,p=1)$, we have%
\begin{align*}
\Omegakl\lesssim\begin{cases}
\frac{1}{r^{2}}\cdot\frac{\ke}{L}, & (k,\ell)\in E,\\
\frac{1}{r^{2}}\cdot\frac{\ke}{L}\cdot(\frac{\log n}{r^{2}}+1) & (k,\ell)\notin E,
\end{cases}\qquad\Vkl\lesssim\begin{cases}
\ke^{1.5}\sqrt{n}, & (k,\ell)\in E,\\
\ke^{1.5}\sqrt{n}\cdot\sqrt{\frac{\log n}{r^{2}}+1}, & (k,\ell)\notin E.
\end{cases}
\end{align*}
\end{lem}
With this lemma in place, for $\gridone$, let us define
\[
B_{E}\coloneqq\sqrt{\frac{\ke^{2}}{rL}\log(\frac{n}{\delta})},\quad V_{E}\coloneqq\ke^{1.5}\sqrt{n},
\]
and set
\[
\Bkl\asymp\begin{cases}
B_{E}, & (k,\ell)\in E,\\
B_{E}\sqrt{\frac{n}{r^{2}}+1}, & (k,\ell)\notin E,
\end{cases}\quad\Qkl\asymp\begin{cases}
B_{E}^{2}V_{E}, & (k,\ell)\in E,\\
B_{E}^{2}V_{E}\sqrt{\frac{n}{r^{2}}+1}, & (k,\ell)\notin E.
\end{cases}
\]
Then, it is easy to check that $\{\Bkl,\Qkl\}$ satisfy both conditions
(\ref{eq:def_Bkl_Qkl}) and (\ref{eq:asp_small_Bij}) in Theorem~\ref{thm:general_graph_MLE},
as long as $B_{E}V_{E}\lesssim1$, namely $rL\gtrsim\ke^{5}n\log(n/\delta)$.
This finishes our proof of Corollary~\ref{cor:MLE_locality} for
$\gridone$. The analysis for $\gridtwo$ is very similar; we skip
the details for brevity. 

The remaining of this section is dedicated to the proof of Lemma~\ref{lem:locality_resistance}.

\paragraph{Effective resistances for $\protect\gridone$. }

We can upper bound $\Omegakl(\Lz)$, where $\Lz=\sumE\Lij\zij(\ei-\ej)(\ei-\ej)^{\top}$,
for an $\gridone(n,r,p=1)$ graph, by transforming it into another
graph whose effective resistances can be easily controlled. The proof
follows the steps below.
\begin{enumerate}
\item Let us increase the resistance $1/(\Lij\zij)=1/(L\zij)$ on each edge
$(i,j)$ of the original 1D grid to the same value $4\ke/L$, which
(according to Rayleigh\textquoteright s Monotonicity Law) can never
decrease the effective resistance between any pair of nodes. In the
following, we assume (for notational simplicity) that all edges have
the same resistance $1$; after we finish the analysis below, we simply
need to multiply the obtained result by a factor of $\ke/L$, in order
to achieve the final result on $\{\Omegakl\}$ in Lemma~\ref{lem:locality_resistance}.
\item For $(k,\ell)\in E$, $\Omegakl$ can be upper bounded (again, by
monotonicity) by the effective resistance of a complete subgraph (containing
$k,\ell$) with $\Theta(r)$ nodes, which is $O(1/r)$ according to
Fact~\ref{fact:resistance_complete_graph}.
\item For general $k\neq\ell$, we construct a subgraph with $O(n)$ nodes
(containing $k,\ell$) as follows: (1) divide the $O(n)$ nodes into
a sequence of $O(n/r)$ blocks, each with $\Theta(r)$ nodes; (2)
preserve the full connection with $\Theta(r^{2})$ edges between any
pair of adjacent blocks, as well as the $\Theta(r)$ edges between
node $k$ (resp.~$\ell$) and the first (resp.~last) block. See
Figure~\ref{fig:resistance_grid}(a) for a visualization. By symmetry,
all nodes within the same block are equivalent to each other, and
hence they can be regarded as one single node, without affecting the
resistance $\Omegakl$. Now we can easily control $\Omegakl$ by the
Series/Parallel Law:
\[
\Omegakl\lesssim\frac{1}{r}+\frac{n}{r}\cdot\frac{1}{r^{2}}=\Big(\frac{n}{r^{2}}+1\Big)\frac{1}{r}.
\]
\end{enumerate}
\begin{figure}
\begin{centering}
\subfloat[$\protect\gridone(n,r,p=1)$]{\begin{centering}
\includegraphics[width=0.4\textwidth]{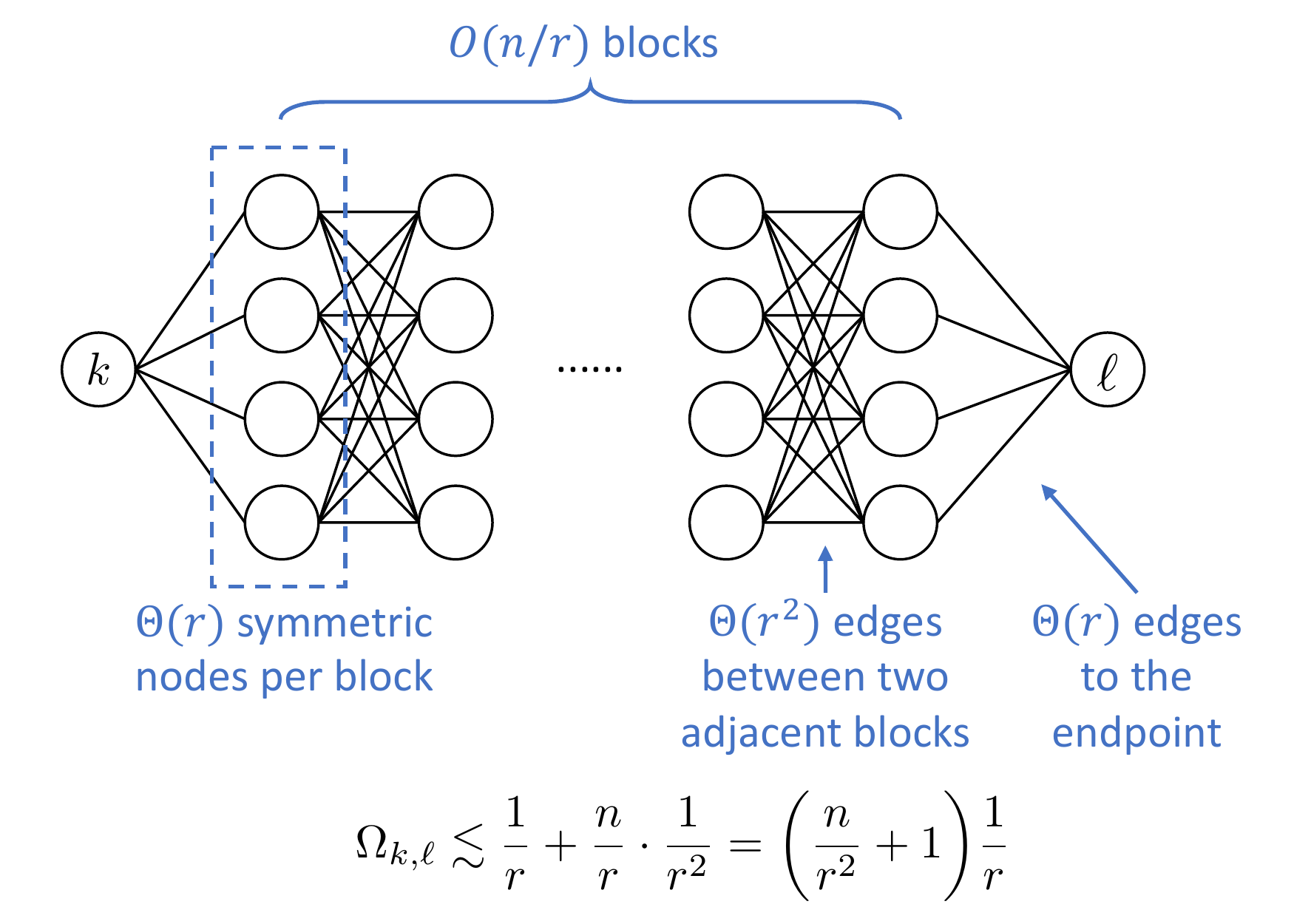}
\par\end{centering}
}\subfloat[$\protect\gridtwo(n,r,p=1)$]{\begin{centering}
\includegraphics[width=0.4\textwidth]{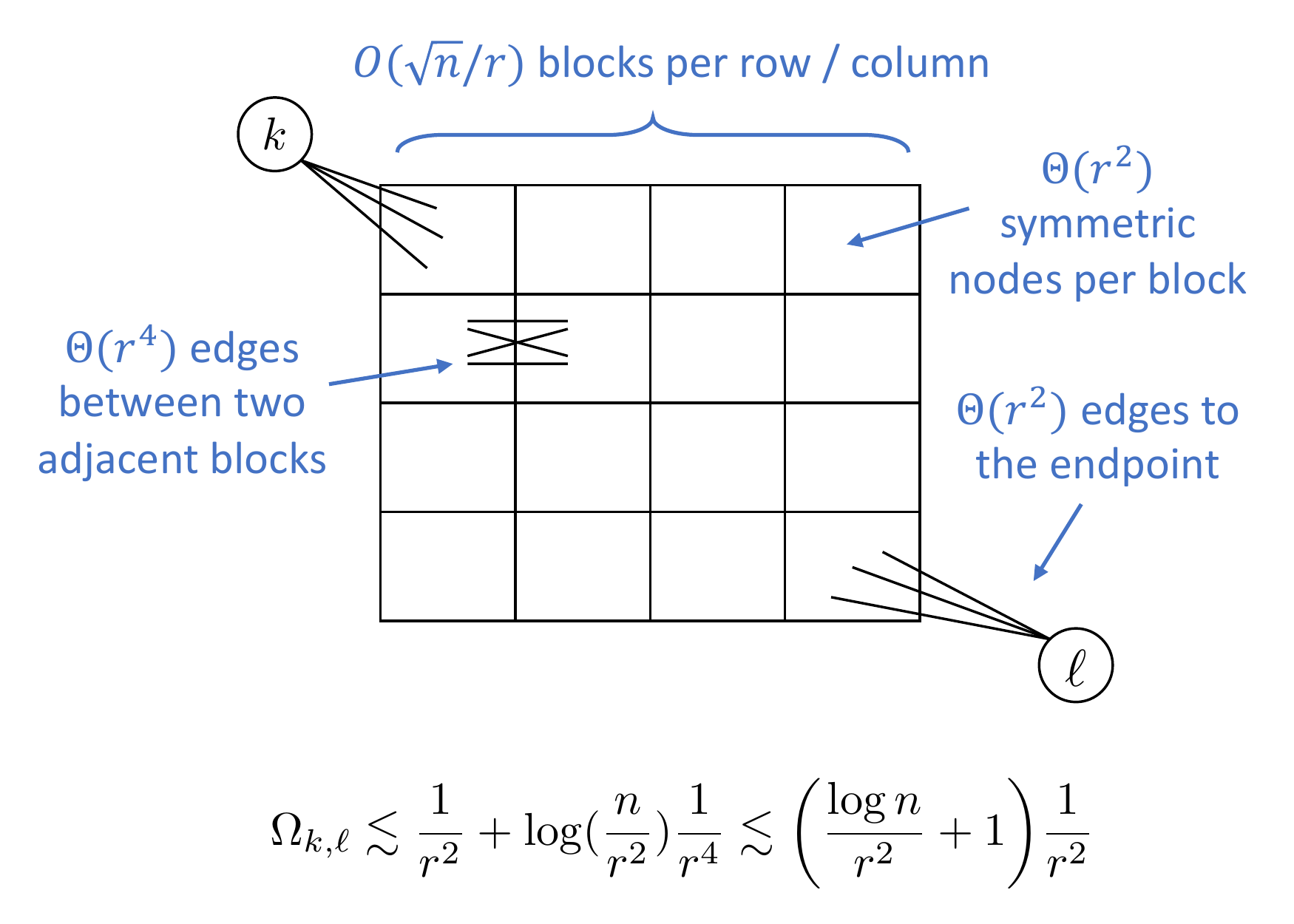}
\par\end{centering}
}
\par\end{centering}
\caption{\label{fig:resistance_grid}Effective resistances for graphs with
locality.}
\end{figure}

\paragraph{Effective resistances for $\protect\gridtwo$. }

Our proof for 2D grids is similar to the previous analysis for 1D
grids, with some modifications in Steps 2 and 3. For $(k,\ell)\in E$,
we can find a complete subgraph with $\Theta(r^{2})$ nodes containing
$k$ and $\ell$ , and hence $\Omegakl\lesssim1/r^{2}$. For general
$k\neq\ell$, we construct a subgraph with $O(n)$ nodes (containing
$k,\ell$) as follows: (1)~divide the $O(\sqrt{n})\times O(\sqrt{n})$
2D grid into blocks of size $\Theta(r)\times\Theta(r)$; (2)~preserve
the full connection with $\Theta(r^{4})$ edges between any pair of
adjacent blocks, as well as $\Theta(r^{2})$ edges between $k$/$\ell$
and its closest block. See Figure~\ref{fig:resistance_grid}(b) for
a visualization. By symmetry, all nodes within one block are equivalent,
and hence they can be regarded as one single node, without affecting
the resistance $\Omegakl$. Now, utilizing the triangle inequality
for effective resistances and the fact that $\gridtwo(m,r=1,p=1)$
satisfies $\Omegaij\lesssim\log m$ for all $1\le i<j\le m$ \cite[Theorem~6.1]{chandra1996electrical},
we have
\[
\Omegakl\lesssim\frac{1}{r^{2}}+\log\Big(\frac{n}{r^{2}}\Big)\cdot\frac{1}{r^{4}}\lesssim\Big(\frac{\log n}{r^{2}}+1\Big)\frac{1}{r^{2}}.
\]

\paragraph{Controlling the residual terms.}

Consider the electric network corresponding to $\Lz$, and denote
the voltage $\bv=\Lzinv(\ek-\el)$ for some fixed $k\neq\ell$. Recalling
the definition of $\Vaggkl$ (\ref{eq:Vkl}), one has
\begin{align*}
\Vkl & =\sumE\Lij\big|(\ei-\ej)^{\top}\Lzinv(\ek-\el)\big|=\sumE\Lij|\vi-\vj|\\
 & =\sumE\frac{\Lij\zij|\vi-\vj|}{\zij}\le4\ke\sumE\Lij\zij|\vi-\vj|\\
 & =4\ke\sumE\sqrt{\Lij\zij}\cdot\sqrt{\Lij\zij}|\vi-\vj|\\
 & \overset{{\rm (i)}}{\le}4\ke\sqrt{\sumE\Lij\zij}\cdot\sqrt{\sumE\Lij\zij(\vi-\vj)^{2}}\\
 & \le4\ke\sqrt{L\cdot|E|}\cdot\sqrt{\bv^{\top}\Lz\bv}\\
 & =4\ke\sqrt{L\cdot|E|}\cdot\sqrt{(\ek-\el)^{\top}\Lzinv(\ek-\el)}=4\ke\sqrt{L\cdot|E|}\cdot\sqrt{\Omegakl(\Lz)},
\end{align*}
where (i) follows from the Cauchy-Schwarz inequality. Here, $|E|\lesssim nr$
for 1D grids and $|E|\lesssim nr^{2}$ for 2D grids; moreover, upper
bounds for $\Omegakl(\Lz)$ have already been obtained in our earlier
analysis. %
Putting these together finishes our proof of Lemma~\ref{lem:locality_resistance}.

\subsection{Does the spectral method work? \label{subsec:does_spectral_work}}

We argue that the spectral method (Algorithm~\ref{alg:spectral})
is likely inferior to MLE, in terms of its dependence on $\ke$ and
$\kappa$ defined in (\ref{eq:def_kappa_ke}); in particular, while
MLE works as long as $\ke$ is bounded (regardless of $\kappa$),
the spectral method does not enjoy such a benign property, especially
in graphs with locality. To see this, let us consider a $\gridone(n,r,p)$
graph with large $n$, small $r\lesssim1$ and $p=1$; moreover, let
$\bthetastar=\log\bpistar\in\R^{n}$ be defined by $\thetastar_{i}=i/r,1\le i\le n$.
This implies that $\ke\lesssim1$, but $\kappa\ge\exp(c_{0}n)$ for
some universal constant $c_{0}>0$; in addition, $\pistar_{i}\le\exp(-c_{1}n)$
for all $1\le i\le n/2$, where $c_{1}>0$ is a universal constant.
This leads to two potential issues as follows.
\begin{enumerate}
\item Even if the spectral method is allowed to construct the stochastic
matrix $\bP$ with infinite samples, it will still encounter \emph{numerical
issues} when calculating the stationary distribution $\bpistar$,
especially for the entries of $\bpistar$ that are exponentially small;
as a result, the estimated scores $\btheta=\log\bpi$ might be inaccurate.
In fact, this issue is common to any algorithm that attempts to first
estimate $\bpi\approx\bpistar$ and then compute $\btheta=\log\bpi$.\footnote{It is easy to check that an accurate $\btheta\approx\bthetastar$
(in terms of additive errors) implies an accurate $\bpi\approx\bpistar$,
but the reverse is not necessarily true in our example. Throughout
this work, we mostly focus on accurately estimating $\bthetastar$.}
\item In the realistic scenario with finite samples, one need to ensure
that the entrywise error of $\bpi$ is \emph{exponentially small},
in order to achieve a meaningful estimation of $\bthetastar$ for
those entries with low scores; even without numerical concerns, it
is still questionable whether this can be achieved with a polynomial
number of samples. 
\end{enumerate}
An empirical comparison between the spectral method and MLE can be
found in Section \ref{subsec:exp_mle_spectral}.

\section{Divide-and-conquer algorithms \label{sec:dc_alg}}

In this section, we design principled divide-and-conquer algorithms
for estimating $\bthetastar$ in a $\gridone$ or $\gridtwo$ graph,
by regarding such a graph as a combination of small \ErdosRenyi 
subgraphs \cite{chen2016community}. We will show that the proposed
methods enjoy performance guarantees of entrywise estimation errors
even with the sparsest samples, although they require additional information
(i.e.~an appropriate graph partition). Some ideas in this section
might also inspire sharper theoretical analysis for MLE in graphs
with locality.

\paragraph{Additional notation. }

Suppose that we are given a graph $G=(V,E)$ with $|V|=n$ nodes,
and a partition of its nodes $V=V_{(1)}\cup V_{(2)}\cup\dots\cup V_{(m)}$,
which can be overlapping or disjoint. Throughout this section, we
use ``$(a)$'' in subscripts to index the subgraphs, where $1\le a\le m$.
For $\btheta\in\R^{n}$, let $\btheta|_{\Va}\in\R^{|\Va|}$ denote
the sub-vector with entries of $\btheta$ in $\Va$, which has the
same indices as in the original vector $\btheta$; for example, if
$\Va=\{15,20,25\}$, then $\btheta|_{\Va}\in\R^{3}$ has $3$ entries
indexed by $\{15,20,25\}$ instead of $\{1,2,3\}$. Similarly, we
use the notion of $\bthetaa\in\R^{|\Va|}$ and all-one vector $\vonea\in\R^{|\Va|}$,
which are indexed by $\Va$ instead of $\{1,2,\dots,|\Va|\}$.

\subsection{DC-overlap: divide and conquer with overlapping subgraphs}

We propose $\DCov$, a two-step algorithm consisting of local estimation
and global alignment. 

\paragraph{Idea.}

Assume that we are given a partition of nodes $V=V_{(1)}\cup V_{(2)}\cup\dots\cup V_{(m)}$;
moreover, each subgraph $\Ga=(\Va,\Ea)$ (where $\Ea$ contains the
edges in $E$ with endpoints in $\Va$) enjoys certain benign properties,
e.g.~it is (approximately) an \ErdosRenyi  graph. Define a super-graph
$\Gtilde=(\Vtilde,\Etilde)$ with $|\Vtilde|=m$; each super-node
stands for one subgraph, and for any $1\le a<b\le m$, $(a,b)\in\Etilde$
if and only if $\Va\cap\Vb\neq\emptyset$. We assume that $\Gtilde$
is a connected graph.

For any ground-truth score vector $\bthetastar$, let us define $\bcstar=[\cstara]_{1\le a\le m}$,
where
\[
\cstara\coloneqq\frac{1}{|\Va|}\sum_{i\in\Va}\thetastar_{i},\quad1\le a\le m
\]
stands for the average score of $\Va$. In addition, let 
\[
\bthetastara\coloneqq\bthetastar|_{\Va}-\cstara\vonea,\quad1\le a\le m,
\]
which satisfies $\bthetastara\perp\vonea$. Hence, to estimate $\bthetastar,$we
can instead estimate each $\bthetastara$ separately, and then align
these estimates by shifting each of them appropriately, which translates
into an estimate of $\bthetastar$.

\paragraph{Algorithm. }

First, for each subgraph $\Ga=(\Va,\Ea)$, we use the relevant subset
of data $\{\yij\}$ to estimate the scores $\bthetaa\in\R^{|\Va|}$;
this can be done in parallel across the subgraphs. Since each $\bthetaa$
is an accurate estimation of $\bthetastara$ only up to a constant
shift, in the second step, we attempt to align $\{\bthetaa\}_{1\le a\le m}$,
namely to find an appropriate constant shift $\ca\in\R$ for each
$\bthetaa$. The idea is to enforce that each item $1\le i\le n$
has approximately the same score across all subgraphs that it belongs
to; more formally, let $\bc=[\ca]_{1\le a\le m}$ be the solution
of the following least-squares problem:
\[
\min_{\bc\in\R^{m}}\quad\sum_{(a,b)\in\Etilde,a<b}\sum_{i\in\Va\cap\Vb}\Big((\thetaai+\ca)-(\thetabi+\cb)\Big)^{2}.
\]
By examining its optimality condition, it can be checked that the
solution can be found by solving a Laplacian linear system. Once
this is solved, we update $\bthetaa\gets\bthetaa+\ca\vonea$ for each
$1\le a\le m$. For the final solution $\btheta\in\R^{n}$, we set
$\theta_{i}$ to be the average of $\{\thetaai\}$ over subgraphs
$1\le a\le m$ that contain node $i$. The complete procedure of $\DCov$
is summarized in Algorithm~\ref{alg:distrov}. 

\begin{algorithm}[tbp] 
\DontPrintSemicolon 
\SetNoFillComment
\caption{$\distrov$: divide and conquer with overlapping subgraphs} \label{alg:distrov} 
{\bf Input:} {graph $G=(V,E)$ and subgraphs $\{ \Ga=(\Va,\Ea)\}_{1\le a\le m}$, data $\{\yij,\Lij\}_{(i,j)\in E,i<j}$.} \\
Define $\Gtilde = (\Vtilde, \Etilde)$, such that $|\Vtilde| = m$, and $(a,b) \in \Etilde$ if and only if $\Va \cap \Vb \neq \emptyset$. \\
\tcp{Step 1: local estimation}
For each $1 \le a \le m$, let $\bthetaa$ be the solution of Algorithm~\ref{alg:mle} or~\ref{alg:spectral} with input $\{\yij,\Lij\}_{(i,j)\in\Ea,i<j}$. \\
\tcp{Step 2: global alignment}
Compute $\bc=\Ltilde^{\dagger}\bx \in \R^m$, where
\begin{align}
\Ltilde&=\sum_{(a,b)\in\Etilde,a<b} |\Va\cap\Vb| (\bea-\beb)(\bea-\beb)^{\top}, \label{eq:DC_ov_Ltilde} \\
\bx&=\sum_{(a,b)\in\Etilde,a<b} \sum_{i\in\Va\cap\Vb}(\thetabi-\thetaai)(\bea-\beb).
\end{align} \\
For all $1 \le i \le n$, let 
$$
\theta_{i}\gets\frac{1}{\si}\sum_{1\le a\le m:i\in\Va}(\thetaai + \ca),\quad\text{where}\quad\si=\big|\{1\le a\le m:i\in\Va\}\big|.
$$ \\
{\bf Output:} $\btheta \in \R^{n}$. 
\end{algorithm}

\subsection{Analysis of DC-overlap}

\subsubsection{Error decomposition}

We first provide a basic result on error decomposition for $\DCov$,
which is applicable to general graphs and partitions of nodes. Let
us first introduce some useful notation. Denote the error vectors
\[
\bdelta=\btheta-\bthetastar,\quad\text{and}\quad\bdela=\bthetaa-\bthetastara,\quad1\le a\le m.
\]
Denote the average of $\bcstar$ as
\[
\cbarstar\coloneqq\frac{1}{m}\langle\bcstar,\vone_{m}\rangle.
\]
Define an incidence matrix 
\[
\Btilde\coloneqq\Big[\dots,\sqrt{\nab}(\bea-\beb),\dots\Big]_{(a,b)\in\Etilde,a<b},\quad\text{where}\quad\nab\coloneqq|\Va\cap\Vb|;
\]
recalling the Laplacian matrix $\Ltilde$ defined in (\ref{eq:DC_ov_Ltilde}),
one can check that $\Btilde\Btilde^{\top}=\Ltilde$. Finally, define
\begin{equation}
\bepstilde\coloneqq\Big[\dots,\sqrt{\nab}\big(\delbar_{(b)}^{a,b}-\delbar_{(a)}^{a,b}\big),\dots\Big]_{(a,b)\in\Etilde,a<b}^{\top},\label{eq:def_eps_tilde}
\end{equation}
where
\[
\delbar_{(b)}^{a,b}=\frac{1}{\nab}\sum_{i\in\Va\cap\Vb}\delbi,\quad\delbar_{(a)}^{a,b}=\frac{1}{\nab}\sum_{i\in\Va\cap\Vb}\delai,\quad(a,b)\in\Etilde.
\]

The following result for $\DCov$ characterizes the error in estimating
$\bcstar$, as well as a decomposition of the final $\linf$ estimation
error into the sum of errors for local estimation and global alignment.
\begin{prop}
\label{prop:DC_ov_error}In Algorithm~\ref{alg:distrov}, the estimation
$\bc\in\R^{m}$ satisfies
\begin{align}
\bc-\bcstar & =-\cbarstar\vone_{m}+\Ltilde^{\dagger}\sumab\sum_{i\in\Va\cap\Vb}(\delbi-\delai)(\bea-\beb)=-\cbarstar\vone_{m}+\Ltilde^{\dagger}\Btilde\bepstilde.\label{eq:error_alignment}
\end{align}
Moreover, the estimation error of $\btheta\in\R^{n}$ can be controlled
as follows: for all $1\le k,\ell\le n$,
\begin{equation}
\big|\delta_{k}-\delta_{\ell}\big|\le\max_{a:k\in\Va}\big|\delta_{(a),k}\big|+\max_{a:\ell\in\Va}\big|\delta_{(a),\ell}\big|+\max_{a,b:k\in\Va,\ell\in\Vb}\big|(\ca-\cstara)-(\cb-\cstarb)\big|.\label{eq:decompose_delta}
\end{equation}
\end{prop}
\begin{proof}
We first compute $\bc-\bcstar$. Recall from Algorithm~\ref{alg:distrov}
that $\bc=\Ltilde^{\dagger}\bx$, where
\begin{align*}
\Ltilde & =\sumab\nab(\bea-\beb)(\bea-\beb)^{\top},\\
\bx & =\sumab\sum_{i\in\Va\cap\Vb}(\thetabi-\thetaai)(\bea-\beb).
\end{align*}
Plugging in
\begin{align*}
\thetaai & =\thetastar_{(a)i}+\delai=\thetastar_{i}-\cstara+\delai,\\
\thetabi & =\thetastar_{(b)i}+\delbi=\thetastar_{i}-\cstarb+\delbi,
\end{align*}
one has
\begin{align*}
\bx & =\sumab\sum_{i\in\Va\cap\Vb}(\thetabi-\thetaai)(\bea-\beb)\\
 & =\sumab\sum_{i\in\Va\cap\Vb}(\cstara-\cstarb+\delbi-\delai)(\bea-\beb)\\
 & =\Ltilde\bcstar+\sumab\sum_{i\in\Va\cap\Vb}(\delbi-\delai)(\bea-\beb).
\end{align*}
As a result,
\begin{align*}
\bc-\bcstar & =\Ltilde^{\dagger}\bx-\bcstar\\
 & =\Ltilde^{\dagger}\Ltilde\bcstar-\bcstar+\Ltilde^{\dagger}\sumab\sum_{i\in\Va\cap\Vb}(\delbi-\delai)(\bea-\beb)\\
 & \overset{{\rm (i)}}{=}-\frac{1}{m}\vone_{m}\vone_{m}^{\top}\bcstar+\Ltilde^{\dagger}\sumab\sum_{i\in\Va\cap\Vb}(\delbi-\delai)(\bea-\beb)\\
 & =-\cbarstar\vone_{m}+\Ltilde^{\dagger}\Btilde\bepstilde,
\end{align*}
where (i) uses the fact that $\Ltilde^{\dagger}\Ltilde=\bI_{m}-\vone_{m}\vone_{m}^{\top}/m$
as long as $\Gtilde$ is connected. This proves our first result (\ref{eq:error_alignment}).
With regard to the final estimation error $\bdelta=\btheta-\bthetastar$,
recall from Algorithm~\ref{alg:distrov} that 
\[
\theta_{k}=\frac{1}{s_{k}}\sum_{1\le a\le m:k\in\Va}\Big(\theta_{(a),k}+\ca\Big),\quad1\le k\le n,
\]
which implies 
\[
\delta_{k}=\theta_{k}-\thetastar_{k}=\frac{1}{s_{k}}\sum_{1\le a\le m:k\in\Va}\Big(\theta_{(a),k}-\theta_{(a),k}^{\star}+\ca-\cstara\Big)=\frac{1}{s_{k}}\sum_{1\le a\le m:k\in\Va}\Big(\delta_{(a),k}+\ca-\cstara\Big).
\]
The desired result (\ref{eq:decompose_delta}) follows immediately
from applying the triangle inequality to $|\delta_{k}-\delta_{\ell}|$.
\end{proof}

\subsubsection{Implications for graphs with locality }

Proposition \ref{prop:DC_ov_error} decomposes the entrywise error
of $\btheta$ into a sum of local estimation errors $\{\delta_{(a),k}\}$,
and the errors of global alignment in the form of $\{(\ca-\cstara)-(\cb-\cstarb)\}$.
Let us specify these error terms for $\gridone(n,r,p)$ or $\gridtwo(n,r,p)$
graphs, although much of the following analysis applies to more general
cases. More specifically, we aim to show that $\bdelta=\btheta-\bthetastar$
satisfies for all $1\le k<\ell\le n$,
\[
|\delta_{k}-\delta_{\ell}|\le\begin{cases}
\tilde{O}\bigg(\sqrt{\frac{n}{r^{2}}+1}\sqrt{\frac{1}{rpL}}\bigg)+\text{residual terms} & \text{for }\gridone(n,r,p),\\
\tilde{O}\bigg(\sqrt{\frac{\log(n)}{r^{2}}+1}\sqrt{\frac{1}{r^{2}pL}}\bigg)+\text{residual terms} & \text{for }\gridtwo(n,r,p),
\end{cases}
\]
where the meaning of ``residual terms'' will become clear in our
analysis.

\paragraph{Preparations.}

Assume that the partition $V=V_{(1)}\cup\dots\cup V_{(m)}$ and the
induced super-graph $\Gtilde=(\Vtilde,\Etilde$) satisfy the following: 
\begin{enumerate}
\item Each subgraph is an \ErdosRenyi  graph, with $\Theta(r)$ nodes for
$\gridone$, or $\Theta(r^{2})$ nodes for $\gridtwo$; 
\item For each $(a,b)\in\Etilde$, it holds that $\nab=|\Va\cap\Vb|\asymp r$
for $\gridone$, or $\nab\asymp r^{2}$ for $\gridtwo$; 
\item Each node of $G$ appears in $O(1)$ subgraphs. 
\end{enumerate}
Such a partition is easy to achieve, provided the correct indices
of the nodes as stated in Definition \ref{def:grids}: for $\gridone$,
each subgraph consists of $\Theta(r)$ consecutive nodes, while for
$\gridtwo$, each subgraph consists of a $\Theta(r)\times\Theta(r)$
block on the 2D lattice. We further assume that $\ke\lesssim1$, $\Lij=L\ge1$,
and $rp\ge\tilde{\Omega}(1)$ for $\gridone$ or $r^{2}p\ge\tilde{\Omega}(1)$
for $\gridtwo$, which is the sparsest regime.

One key ingredient to our analysis is the following result on uncertainty
quantification for MLE in \ErdosRenyi  graphs \cite[Theorem~2.2]{gao2021uncertainty};
in the following, we recall that $\bA$ stands for the adjacency matrix
of a (random) graph.
\begin{lem}
[\cite{gao2021uncertainty}]\label{lem:ER_UQ}Consider an \ErdosRenyi 
graph with $N$ nodes, edge sampling probability $p$, and $L$ measurements
on each edge. Assume that $\kappa\lesssim1$ and $Np\gg\log^{1.5}N$.
Then, the MLE solution $\bthetaMLE$ to (\ref{eq:def_loss}) satisfies
\begin{equation}
\bthetaMLE-\bthetastar=\bdelta^{\main}+\bdelta^{\res},\label{eq:ER_UQ_main_res}
\end{equation}
where the main error term $\bdelta^{\main}$ is defined by 
\[
\bdelta_{i}^{\main}=\frac{\sum_{j\neq i}A_{i,j}\big(\yij-\sigmoid(\thetastar_{i}-\thetastar_{j})\big)}{\sum_{j\neq i}A_{i,j}\sigmoid'(\thetastar_{i}-\thetastar_{j})},\quad1\le i\le N,
\]
and the residual term $\bdelta^{\res}$ satisfies with probability
$1-O(N^{-10})$
\[
\|\bdelta^{\res}\|_{\infty}=o\bigg(\frac{1}{\sqrt{NpL}}\bigg).
\]
\end{lem}
\begin{rem}
For local estimation in Algorithm~\ref{alg:distrov}, one might use
the spectral method instead of MLE. For theoretical analysis in this
case, we simply need to replace Lemma \ref{lem:ER_UQ} with \cite[Theorem~2.5]{gao2021uncertainty},
namely the result of uncertainty quantification for the spectral method
(which is slightly worse than that of MLE).
\end{rem}
\begin{figure}
\begin{centering}
\includegraphics[width=0.5\textwidth]{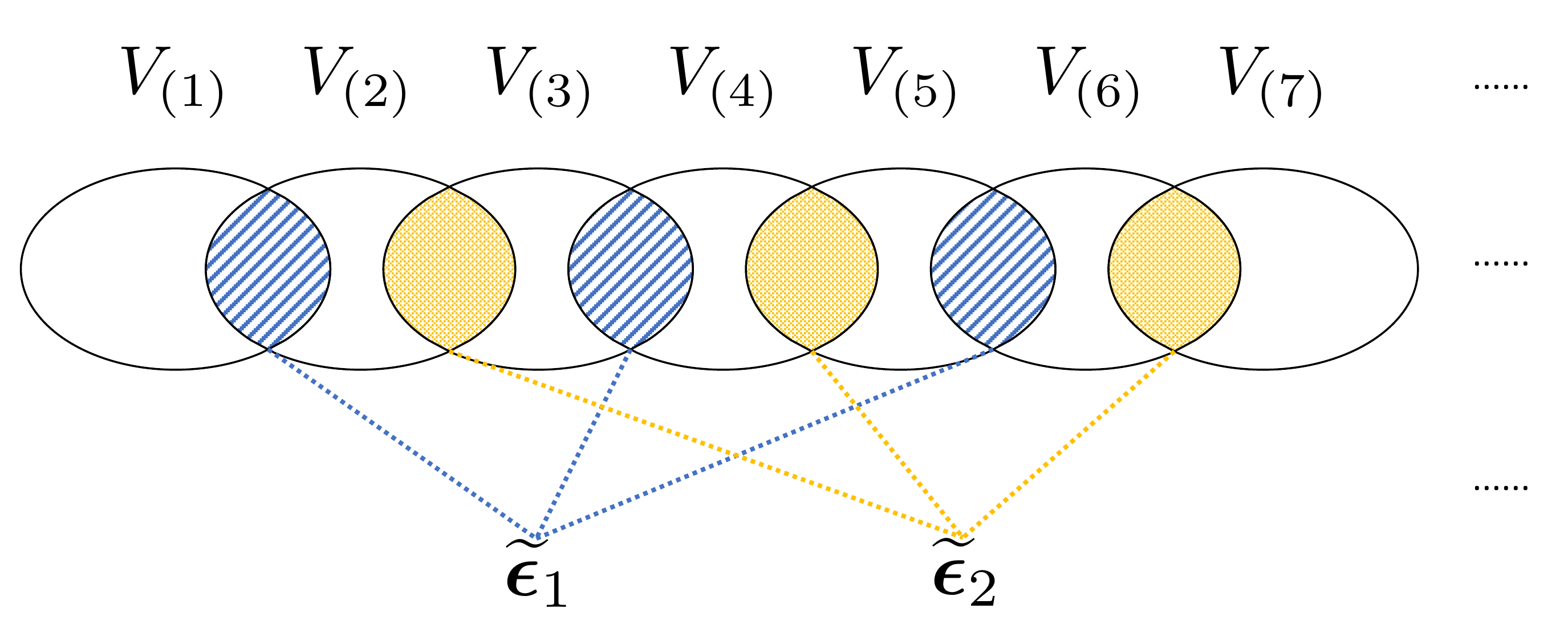}
\par\end{centering}
\caption{\label{fig:grid_epstilde_decomposition}The decomposition $\protect\bepstilde_{>}=\protect\bepstilde_{1}+\protect\bepstilde_{2}$
for $\protect\gridone$, so that each term contains independent entries. }

\end{figure}

\paragraph{Errors of global alignment.}

The local error terms in (\ref{eq:decompose_delta}) can be easily
controlled by Lemma \ref{lem:ER_UQ}, while the errors of global alignment
require further analysis. Recall from (\ref{eq:error_alignment})
that $\bc-\bcstar=-\cbarstar\vone_{m}+\Ltilde^{\dagger}\Btilde\bepstilde$,
where $\bepstilde$ is defined in (\ref{eq:def_eps_tilde}). Although
the entries of $\bepstilde$ are dependent, it is easy to make a decomposition
\[
\bepstilde=\sum_{f=1}^{F}\bepstilde_{f},
\]
such that $F\lesssim1$, and each $\bepstilde_{f}$ has independent
entries. To achieve this, we first rewrite
\[
\bepstilde=\bepstilde_{>}-\bepstilde_{<},\quad\text{where}\quad\bepstilde_{>}\coloneqq\Big[\dots,\sqrt{\nab}\delbar_{(b)}^{a,b},\dots\Big]_{(a,b)\in\Etilde,a<b}^{\top},\quad\bepstilde_{<}=\bepstilde\coloneqq\Big[\dots,\sqrt{\nab}\delbar_{(a)}^{a,b},\dots\Big]_{(a,b)\in\Etilde,a<b}^{\top}.
\]
Then, we further decompose each of $\bepstilde_{>}$ and $\bepstilde_{<}$
into a sum of $O(1)$ terms, by assigning each entry of $\bepstilde_{>}$
(or $\bepstilde_{<}$) to one of these terms, and setting the remaining
entries in these terms to zero. This is done based on the key observation
that each $\delbar_{(b)}^{a,b}$ is statistically dependent only on
the subset of data $\{\Aij,\yij\}_{i,j\in\Vb}$. For example, such
an decomposition for $\gridone$ is visualized in Figure \ref{fig:grid_epstilde_decomposition}.

With this decomposition in place, we can rewrite
\[
\bc-\bcstar=-\cbarstar\vone_{m}+\sum_{f=1}^{F}\Ltilde^{\dagger}\Btilde\bepstilde_{f}.
\]
Note that 
\[
(\ca-\cstara)-(\cb-\cstarb)=\big\langle\bea-\beb,\sum_{f=1}^{F}\Ltilde^{\dagger}\Btilde\bepstilde_{f}\big\rangle,\quad1\le a<b\le m.
\]
In the remaining analysis, we will focus on the ``main'' error terms.
Let us start with the decomposition
\[
\bdela=\bdela^{\main}+\bdela^{\res},\quad1\le a\le m,
\]
which corresponds to the decomposition (\ref{eq:ER_UQ_main_res})
in Lemma \ref{lem:ER_UQ}. Similarly, we rewrite
\[
\bepstilde=\bepstilde^{\main}+\bepstilde^{\res},\quad\text{and}\quad\bepstilde_{f}=\bepstilde_{f}^{\main}+\bepstilde_{f}^{\res},\quad1\le f\le F.
\]
Assuming that $\bepstilde_{f}^{\main}$ has independent zero-mean
entries with subgaussian norm $\|\cdot\|_{\psi_{2}}\le\sigma_{\main}$,
we have 
\[
\big\|\langle\bea-\beb,\Ltilde^{\dagger}\Btilde\bepstilde_{f}^{\main}\rangle\big\|_{\psi_{2}}\le\sigma_{\main}\big\|\Btilde^{\top}\Ltilde^{\dagger}(\bea-\beb)\big\|_{2}=\sigma_{\main}\sqrt{\Omega_{a,b}(\Ltilde)},\quad1\le a<b\le m.
\]
By standard concentration arguments, one has with high probability
\begin{equation}
\big|(\ca-\cstara)-(\cb-\cstarb)\big|\le\tilde{O}\big(\sigma_{\main}\sqrt{\Omega_{a,b}(\Ltilde)}\big)+\text{residual terms},\quad1\le a<b\le m.\label{eq:alignment_error_decomposition}
\end{equation}
It remains to find a valid upper bound $\sigma_{\main}$ for the subgaussian
norm of $\bepstilde_{f}^{\main}$, and control the effective resistances
corresponding to the Laplacian matrix $\Ltilde=\sumab\nab(\bea-\beb)(\bea-\beb)^{\top}$.

\paragraph{Special case: $\protect\gridone$.}

Focusing on the main error terms, we have the following approximation:
\[
\delbar_{(b)}^{a,b}=\frac{1}{\nab}\sum_{i\in\Va\cap\Vb}\delbi\approx\frac{1}{\nab}\sum_{i\in\Va\cap\Vb}\delbi^{\main}=\frac{1}{\nab}\sum_{i\in\Va\cap\Vb}\frac{\sum_{j\neq i,j\in\Vb}\Aij\epsij}{\sum_{j\neq i,j\in\Vb}\Aij\zij}\eqqcolon\delbar_{(b)}^{a,b,\main},
\]
where $\zij=\sigmoid'(\thetastar_{i}-\thetastar_{j})\asymp1$, and
$\epsij=\yij-\sigmoid(\thetastar_{i}-\thetastar_{j})$ is a zero-mean
noise term with $\|\epsij\|_{\psi_{2}}\lesssim\sqrt{1/L}$. Moreover,
our assumption of $rp\ge\tilde{\Omega}(1)$ implies that $\sum_{j\neq i,j\in\Vb}\Aij\asymp rp$
with high probability. Conditioned on this, it is easy to check that
each entry of $\bepstilde_{f}^{\main}$ has subgaussian norm bounded
by $\sigma_{\main}\lesssim1/\sqrt{rpL}$. As for the $\Omega_{a,b}(\Ltilde)$
term in (\ref{eq:alignment_error_decomposition}), note that the super-graph
$\Gtilde$ is a line graph of length $O(n/r)$; in addition, recall
the assumption that $\nab\asymp r$. Therefore, by the Series/Parallel
Law,
\[
\Omega_{a,b}(\Ltilde)\lesssim\frac{n}{r}\cdot\frac{1}{r}\lesssim\frac{n}{r^{2}}.
\]
Plugging these, as well as the local error bound $\tilde{O}(1/\sqrt{rpL})$,
into (\ref{eq:decompose_delta}), we arrive at
\[
|\delta_{k}-\delta_{\ell}|\le\tilde{O}\bigg(\sqrt{\frac{n}{r^{2}}+1}\sqrt{\frac{1}{rpL}}\bigg)+\text{residual terms},\quad1\le k<\ell\le n.
\]

\paragraph{Special case: $\protect\gridtwo$.}

The analysis for $\gridtwo$ follows the same steps. First, one can
show that each entry of $\bepstilde_{f}^{\main}$ has subgaussian
norm bounded by $\sigma_{\main}\lesssim1/\sqrt{r^{2}pL}$. Moreover,
note that the super-graph $\Gtilde$ is a 2D lattice of size $O(\frac{\sqrt{n}}{r}\times\frac{\sqrt{n}}{r})$;
in addition, recall the assumption that $\nab\asymp r^{2}$. Therefore,
by the Parallel Law and the effective resistance of a 2D lattice \cite[Theorem~6.1]{chandra1996electrical},
one has
\[
\Omega_{a,b}(\Ltilde)\lesssim\log\Big(\frac{\sqrt{n}}{r}\Big)\cdot\frac{1}{r^{2}}\lesssim\frac{\log(n)}{r^{2}}.
\]
Plugging these, as well as the local error bound $\tilde{O}(1/\sqrt{r^{2}pL})$,
into (\ref{eq:decompose_delta}), we arrive at
\[
|\delta_{k}-\delta_{\ell}|\le\tilde{O}\bigg(\sqrt{\frac{\log(n)}{r^{2}}+1}\sqrt{\frac{1}{r^{2}pL}}\bigg)+\text{residual terms},\quad1\le k<\ell\le n.
\]

\paragraph{Caution about ``main'' and ``residual'' error terms. }

In \ErdosRenyi graphs, the main error terms of the MLE solution dominate
the residual terms. However, this is not always the case for graphs
with strong locality. To see this, consider the simplest example with
a line graph, i.e.~$\gridone(n,r=1,p=1)$. The MLE solution $\bthetaMLE$
admits a closed form: 
\[
\thetaMLE_{k+1}-\thetaMLE_{k}=\sigmoid^{-1}(y_{k+1,k}),\quad1\le k\le n-1,
\]
where $\sigmoid^{-1}$ is the inverse of the sigmoid function. Denote
the error as $\bdelta=\bthetaMLE-\bthetastar$ as usual. Recall that
$y_{k+1,k}=\sigmoid(\thetastar_{k+1}-\thetastar_{k})+\epsilon_{k+1,k}$,
where $\epsilon_{k+1,k}$ is a zero-mean noise term with $\|\epsilon_{k+1,k}\|_{\psi_{2}}\asymp1$.
Moreover, $\sigmoid^{-1}$ is concave on $(0,0.5]$ and convex on
$[0.5,1)$. Consider the case where $\thetastar_{k}=k,1\le k\le n$;
then by a Taylor series expansion, $\delta_{k+1}-\delta_{k}$ can
be decomposed into an $O(1/\sqrt{L})$ zero-mean noise term (main),
plus an $O(1/L)$ bias term (residual). Through the line graph, the
zero-mean terms accumulate at a rate of $\sqrt{n}$ (due to independence
and concentration), while the bias terms accumulate linearly with
$n$. Consequently, one has
\[
|\delta_{n}-\delta_{1}|\asymp\sqrt{\frac{n}{L}}+\frac{n}{L}.
\]
This shows that, although the residual terms have smaller magnitudes
than the main terms, the sum of them might eventually dominate the
overall $\linf$ estimation error in the line graph, unless $L\gtrsim n$.
While this result is implied by our Theorem \ref{thm:general_graph_MLE},
formal analysis for the effect of ``residual'' terms in more general
graphs with locality remains open.

\subsection{Solving MLE by projected gradient descent }

Inspired by $\distrov$, we design a new method for solving MLE in
graphs with locality, which enjoys fast convergence numerically and
might inspire further theoretical analysis for MLE. 

\paragraph{Ingredient 1: re-parameterization.}

Given a partition of nodes $V=V_{(1)}\cup\dots\cup V_{(m)}$, define
the subgraphs $\Ga=(\Va,\Ea),1\le a\le m$ and the super-graph $\Gtilde=(\Vtilde,\Etilde)$
in the same way as in $\DCov$. Assume that the subgraphs cover all
edges in the original graph, namely $E_{(1)}\cup\dots\cup E_{(m)}=E$.
We propose to solve MLE with the re-parameterization $\bphi=\{\bthetaa\in\R^{|\Va|}\}_{1\le a\le m}$.
Define the constraint set
\[
\Ccal=\Big\{\bphi=\{\bthetaa\in\R^{|\Va|}\}_{1\le a\le m}:\thetaai=\thetabi\text{ for all }(a,b)\in\Etilde\text{ and }i\in\Va\cap\Vb\Big\}.
\]
With these in place, minimizing the MLE loss function $\Lcal(\btheta)$
defined in (\ref{eq:def_loss}) is equivalent to
\[
\min_{\bphi=\{\bthetaa\}\in\Ccal}\quad\sum_{1\le a\le m}\Lcal_{(a)}(\bthetaa),
\]
where
\[
\Lcal_{(a)}(\bthetaa)=\sum_{(i,j)\in\Ea,i<j}\wij\Lij\Big(-\yij(\thetaai-\thetaaj)+\log(1+e^{\thetaai-\thetaaj})\Big).
\]
Notice that $\Lcal_{(a)}(\bthetaa)$ is the same as the MLE loss defined
on the subgraph $\Ga$, except for the additional weights $\{\wij\}$,
which are chosen to ensure that $\Lcal(\btheta)=\sum_{1\le a\le m}\Lcal_{(a)}(\btheta|_{\Va})$
for any $\btheta\in\R^{n}$.

\paragraph{Ingredient 2: projected gradient descent (PGD).}

Now, given any $\btheta\in\R^{n}$, we do the re-parameterization
$\bphi=\{\bthetaa=\btheta|_{\Va}\}_{1\le a\le m}$, and run one step
of gradient descent on $\bphi$ to obtain $\{\bthetaa'\}$, which
can be done by running one step of gradient descent on each loss function
$\Lcal_{(a)}(\bthetaa)$ separately and in parallel. The next step
is to project $\{\bthetaa'\}$ back to the constraint set $\Ccal$.
If we naively use the usual Euclidean norm as the distance metric,
then projection onto $\Ccal$ is equivalent to solving the problem
\[
\min_{\btheta''\in\R^{n}}\quad\sum_{1\le i\le n}\sum_{1\le a\le m:i\in\Va}(\thetaai'-\theta_{i}'')^{2}.
\]
The solution $\btheta''$ satisfies that each $\theta_{i}''$ is simply
the average of $\{\thetaai'\}_{1\le a\le m:i\in\Va}$; this incurs
the minimal loss
\[
\sum_{1\le i\le n}\sum_{1\le a\le m:i\in\Va}\bigg(\thetaai'-\frac{1}{\si}\sum_{1\le b\le m:i\in\Vb}\thetabi'\bigg)^{2},
\]
where $\si=|\{1\le a\le m:i\in\Va\}|,1\le i\le n$. 

To accelerate convergence, we propose to perform projection in a different
way. The key observation is that, for the loss function $\Lcal_{(a)}$,
$\bthetaa$ is equivalent to $\bthetaa+\ca\vonea$ for any $\ca\in\R$.
Hence, a better way of projection is solving a joint minimization
problem over $\btheta''\in\R^{n}$ and $\bc\in\R^{m}$. Given any
$\bc$, the optimal $\btheta''$ can be obtained by taking a simple
average for each entry; hence, it boils down to optimizing $\bc$
as follows:
\[
\min_{\bc\in\R^{m}}\quad f(\bc)=\sum_{1\le i\le n}\sum_{1\le a\le m:i\in\Va}\Big(\thetaai'+\ca-\frac{1}{\si}\sum_{1\le b\le m:i\in\Vb}(\thetabi'+\cb)\Big)^{2}.
\]
Using the fact that $\sum_{1\le i\le n}(x_{i}-\bar{x})^{2}=\frac{1}{n}\sum_{1\le i<j\le n}(x_{i}-x_{j})^{2}$
(where $\bar{x}=\frac{1}{n}\sum_{1\le i\le n}x_{i}$), this loss function
can be rewritten as
\begin{align*}
f(\bc) & =\sum_{1\le i\le n}\frac{1}{\si}\sum_{1\le a<b\le m:i\in\Va\cap\Vb}\Big((\thetaai'+\ca)-(\thetabi'+\cb)\Big)^{2}\\
 & =\sum_{(a,b)\in\Etilde,a<b}\sum_{i\in\Va\cap\Vb}\frac{1}{\si}\Big((\thetaai'+\ca)-(\thetabi'+\cb)\Big)^{2}.
\end{align*}
Hence this becomes a weighted least-squares problem; again, finding
the optimal $\bc$ can be done by solving a Laplacian linear system,
similar to the case in our earlier analysis for $\distrov$. 

Algorithm~\ref{alg:mle_distrov} summarizes the proposed procedure
of solving MLE by re-parameterization and PGD.
\begin{rem}
Common first-order iterative methods for solving MLE (e.g.~gradient
descent, coordinate descent, minorization-maximization \cite{hunter2004mm,vojnovic2019accelerated})
are doomed to suffer from slow convergence in graphs with locality,
because they are \emph{decentralized} in nature. More specifically,
at each iteration, the $i$-th entry of the estimate $\btheta$ is
updated using only its neighboring entries $\{\theta_{j}\}_{j\in\Ncali}$;
therefore, it takes many iterations for changes in one entry to propagate
through the graph, if the graph exhibits locality. The PrecondGD method
(see Remark \ref{rem:PrecondGD}) accelerate convergence by a preconditioner
that adapts to the optimization landscape of $\Lcal(\btheta)$. In
contrast, the PGD method proposed in this section is tailored to graphs
with locality, and accelerate convergence by explicit global alignment
among subgraphs; compared with PrecondGD, it involves solving Laplacian
linear systems of smaller sizes. See Section \ref{subsec:exp_solve_mle}
for an empirical comparison of various methods for solving MLE.
\end{rem}
\begin{algorithm}[tbp] 
\DontPrintSemicolon 
\SetNoFillComment
\caption{Projected gradient descent (PGD) for solving MLE} \label{alg:mle_distrov} 
{\bf Input:} {graph $G=(V,E)$ and subgraphs $\{ \Ga=(\Va,\Ea)\}_{1\le a\le m}$, data $\{\yij,\Lij\}_{(i,j)\in E,i<j}$,  initialization $\btheta^{0} \in \R^n$, step size $\eta$, number of iterations $T$.} \\
Define $\Gtilde = (\Vtilde, \Etilde)$, such that $|\Vtilde| = m$, and $(a,b) \in \Etilde$ if and only if $\Va \cap \Vb \neq \emptyset$. \\
Calculate $\{\si\}, \{\wij\}$ as follows:
$$
\si\gets\big|\{1\le a\le m:i\in\Va\}\big|,\quad 1\le i\le n; \quad \wij\gets\frac{1}{\big|\{1\le a\le m:(i,j)\in\Ea\}\big|},\quad (i,j)\in E.
$$ \\
Define the graph Laplacian matrix
$$
\Ltilde=\sum_{(a,b)\in\Etilde,a<b} \bigg(\sum_{i\in\Va\cap\Vb}\frac{1}{\si}\bigg) (\bea-\beb)(\bea-\beb)^{\top}.
$$ \\
For each $1 \le a \le m$, define the loss function 
$$
\Lcal_{(a)}(\bthetaa)=\sum_{(i,j)\in\Ea,i<j}\wij\Lij\Big(-\yij(\thetaai-\thetaaj)+\log(1+e^{\thetaai-\thetaaj})\Big).
$$ \\
\tcp{PGD iterations}
\For{$t = 0, 1, \dots, T-1$}{
For each $1 \le a \le m$, let $\bthetaa\gets\btheta^{t}|_{\Va}$ and $\bthetaa'\gets\bthetaa-\eta\nabla\Lcal_{(a)}(\bthetaa)$. \\
Compute $\bc=\Ltilde^{\dagger}\bx \in \R^m$, where
\begin{align*}
\bx&=\sum_{(a,b)\in\Etilde,a<b}\sum_{i\in\Va\cap\Vb}\frac{1}{\si}(\thetabi'-\thetaai')(\bea-\beb).
\end{align*} \\
For all $1 \le i \le n$, let $\theta_{i}^{t+1}\gets\frac{1}{\si}\sum_{1\le a\le m:i\in\Va}(\thetaai'+\ca)$. 
}
{\bf Output:} $\btheta^{T} \in \R^{n}$. 
\end{algorithm}

\subsection{DC-community: divide and conquer with communities}

\paragraph{Algorithm.}

$\DCcomm$ is another divide-and-conquer algorithm, suitable for graphs
composed of multiple communities (including $\gridone$ and $\gridtwo$).
It follows a similar two-step procedure, namely local estimation and
global alignment.

Suppose that we are given a graph $G=(V,E)$ and a \emph{disjoint}
partition of the nodes $V=V_{(1)}\cup\dots\cup V_{(m)}$; each $\Va$
can be regarded as a community. Define the super-graph $\tilde{G}=(\tilde{V},\tilde{E})$,
where $|\Vtilde|=m$, and $(a,b)\in\Etilde$ if and only if there
are edges connecting $\Va$ and $\Vb$, namely $(\Va\times\Vb)\cap E\neq\emptyset$.
Similar to $\DCov$, we first obtain local estimation $\{\bthetaa\}_{1\le a\le m}$
for each subgraph. It remains to achieve global alignment by some
constant shifts $\bc=[\ca]_{1\le a\le m}$. To do this, let us first
estimate the difference of shifts $\Delab$ between a pair of subgraphs
$(a,b)\in\Etilde$, by finding the solution to
\[
\sum_{(i,j)\in(\Va\times\Vb)\cap E}\Lij\Big(\sigmoid(\thetaai-\thetabj+\Delab)-\yij\Big)=0;
\]
this can be solved by various methods, such as Newton's method or
binary search. Equivalently, $\Delab$ is chosen to maximize the likelihood
of the samples $\{\yij\}$ on the edges connecting $\Va$ and $\Vb$.
Now, using $\{\Delab\}$ as a surrogate of $\{\ca-\cb\}$, we find
$\bc$ by solving the following weighted least-squares problem with
some given weights $\{\wab\}$ (e.g.~$\wab=1$, or $\wab=|(\Va\times\Vb)\cap E|$):
\[
\min_{\bc\in\R^{m}}\quad\sum_{(a,b)\in\Etilde,a<b}\wab\big(\ca-\cb-\Delab\big)^{2}.
\]
Again, the optimal $\bc$ is the solution to a Laplacian linear system.
Finally, for each $1\le a\le m$ and $i\in\Va$, let $\theta_{i}=\thetaai+\ca$;
since $\{\Va\}$ are disjoint, this already gives us a solution $\btheta\in\R^{n}$.
The complete procedure of $\DCcomm$ is summarized in Algorithm~\ref{alg:distrcomm}. 

\paragraph{Analysis.}

For a theoretical analysis of $\DCcomm$, the key step is to show
that each $\Delab\in\R$ is approximately linear in the local estimation
errors of $\bthetaa$ and $\bthetab$. Ingredients of the remaining
analysis are the same as those for $\DCov$, namely uncertainty quantification
for MLE in \ErdosRenyi  graphs, and the connection between Laplacian
systems and effective resistances. We skip the details for brevity.

\begin{algorithm}[tbp] 
\DontPrintSemicolon 
\SetNoFillComment
\caption{$\distrcomm$: divide and conquer by communities} \label{alg:distrcomm} 
{\bf Input:} {graph $G=(V,E)$ and subgraphs $\{ \Ga=(\Va,\Ea)\}_{1\le a\le m}$, weights $\{\wab\}$, data $\{\yij,\Lij\}_{(i,j)\in E,i<j}$.} \\
Define $\Gtilde = (\Vtilde, \Etilde)$, such that $|\Vtilde| = m$, and $(a,b) \in \Etilde$ if and only if $(\Va \times \Vb) \cap E \neq \emptyset$. \\
\tcp{Step 1: local estimation}
For each $1 \le a \le m$, let $\bthetaa$ be the solution of Algorithm~\ref{alg:mle} or~\ref{alg:spectral} with input $\{\yij,\Lij\}_{(i,j)\in\Ea,i<j}$. \\
\tcp{Step 2: global alignment}
For each $(a,b) \in \Etilde, a < b$, let $\Delab$ be the solution of the following equation:
$$
\sum_{(i,j)\in(\Va\times\Vb)\cap E}\Lij\Big(\sigmoid(\thetaai-\thetabj+\Delab)-\yij\Big)=0.
$$ \\
Let $\bc = \Ltilde^{\dagger} \bx \in \R^m$, where 
\begin{align*}
\Ltilde&=\sum_{(a,b)\in\Etilde,a<b}\wab(\bea-\beb)(\bea-\beb)^{\top},\\\bx&=\sum_{(a,b)\in\Etilde,a<b}\wab\Delab(\bea-\beb).
\end{align*} \\
Define $\btheta \in \R^n$ as follows: for each $1 \le a \le m$, let $\theta_i\gets\thetaai+\ca$ for all $i \in \Va$. \\
{\bf Output:} $\btheta \in \R^{n}$. 
\end{algorithm}

\section{Numerical experiments\label{sec:experiments}}

In this section, we conduct numerical experiments to validate our
theory and illustrate the efficacy of the proposed algorithms.

\subsection{MLE and the spectral method \label{subsec:exp_mle_spectral}}

\begin{figure}
\begin{centering}
\subfloat[``sine'', $\|\protect\bpi-\protect\bpistar\|_{\infty}/\|\protect\bpistar\|_{\infty}$]{\begin{centering}
\includegraphics[width=0.33\textwidth]{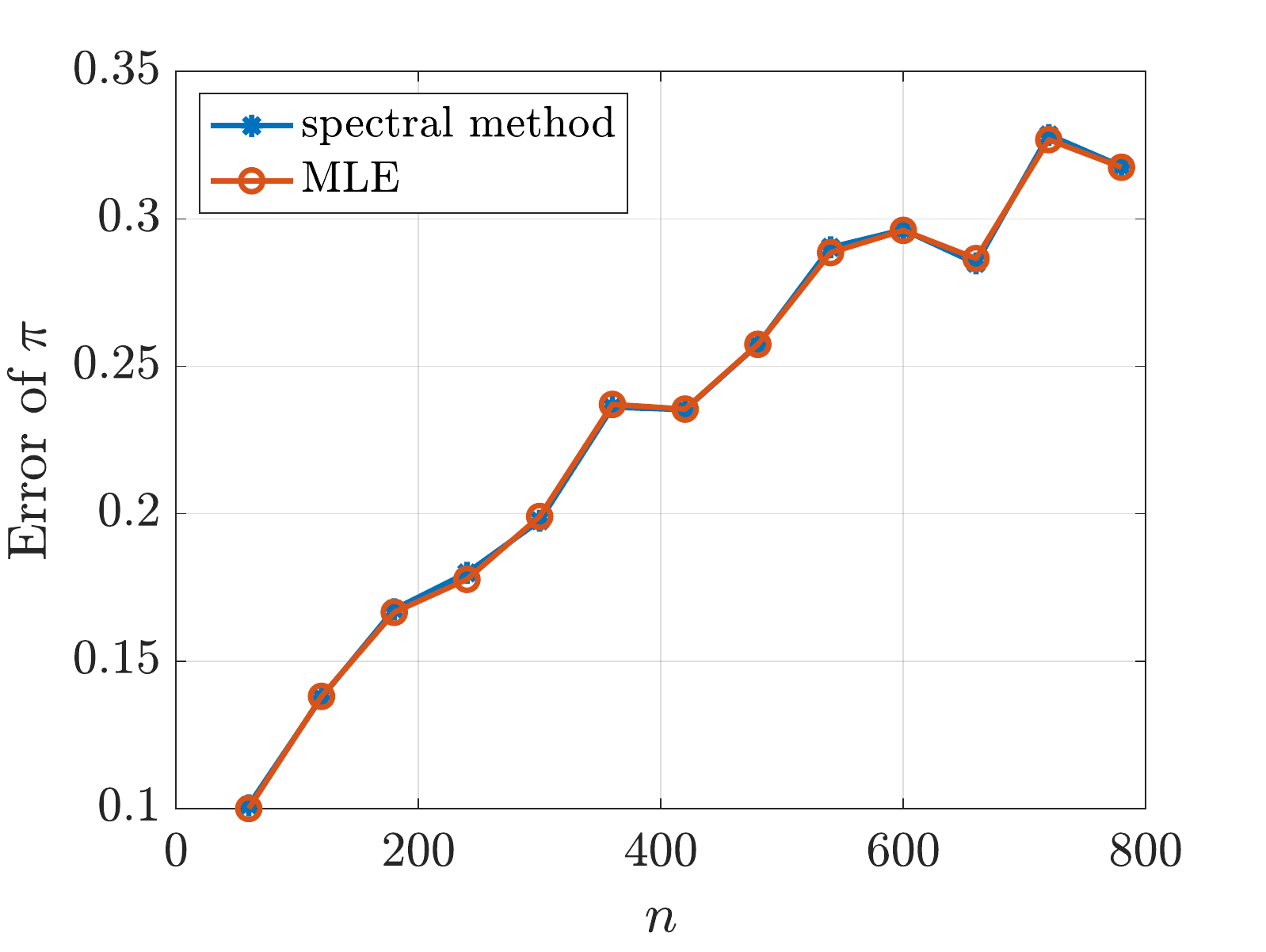}
\par\end{centering}
}\subfloat[``sine'', $\|\protect\btheta-\protect\bthetastar\|_{\infty}$]{\begin{centering}
\includegraphics[width=0.33\textwidth]{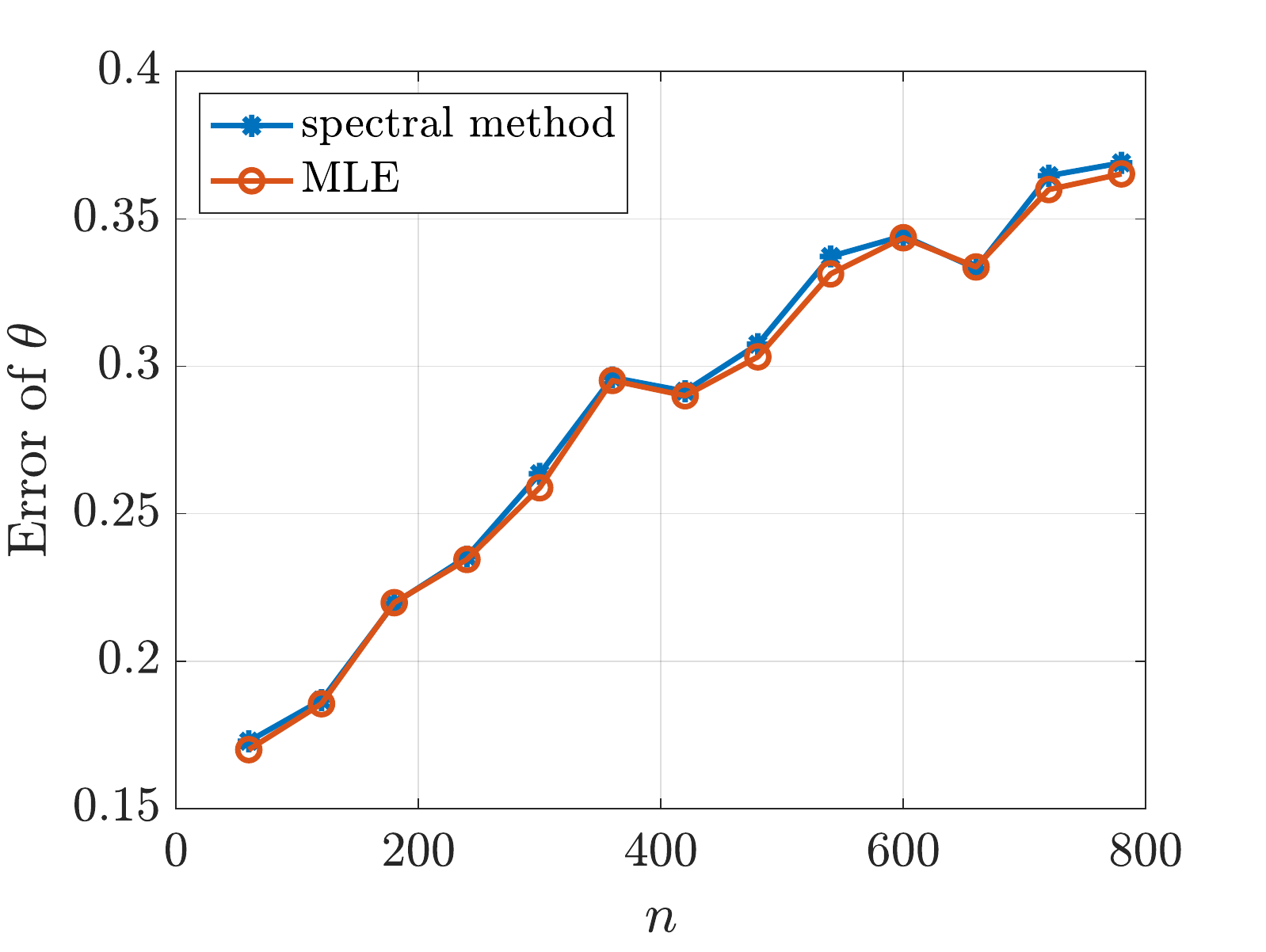}
\par\end{centering}
}\subfloat[``sine'', estimation of $\protect\bpistar$]{\begin{centering}
\includegraphics[width=0.33\textwidth]{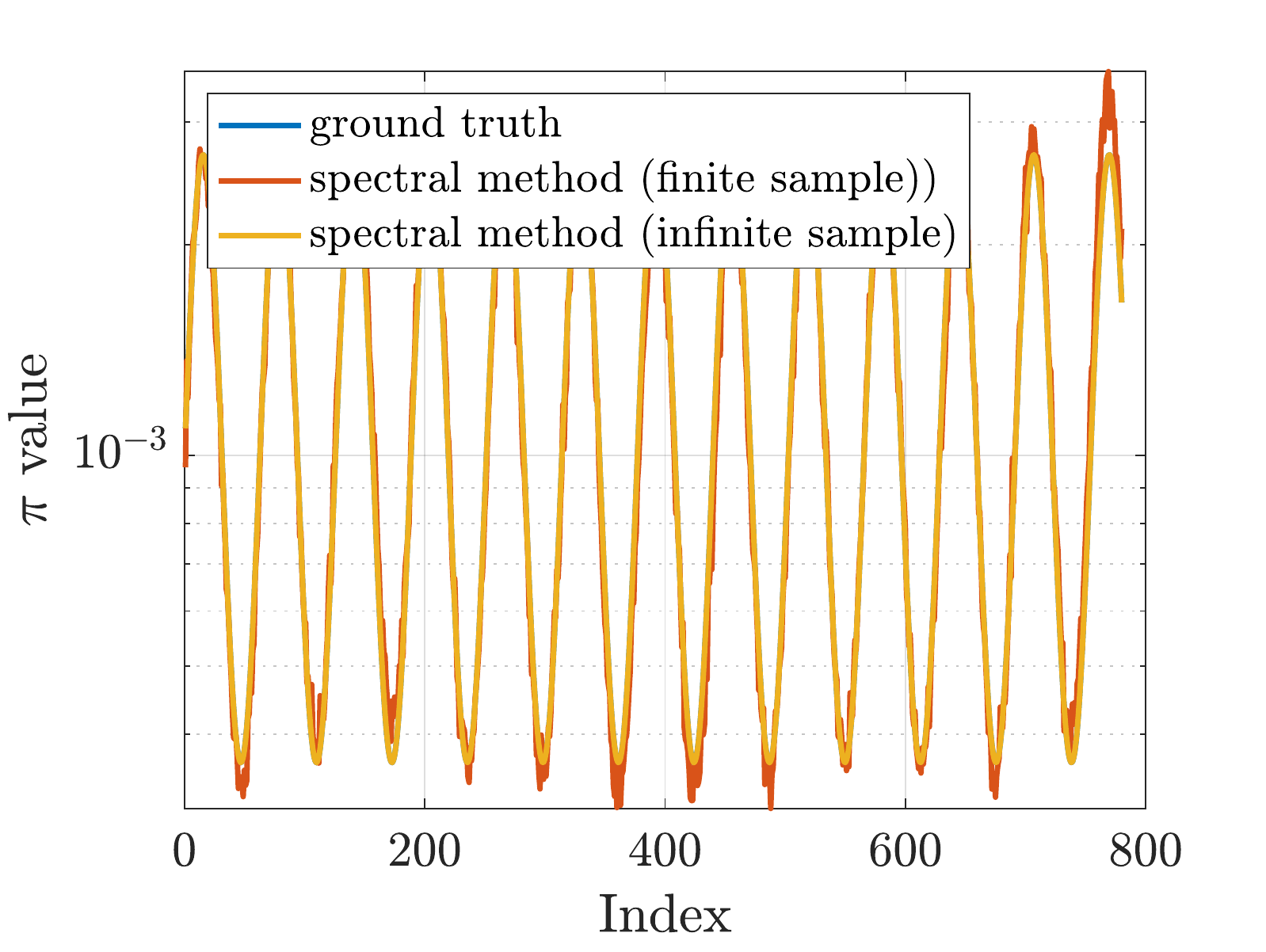}
\par\end{centering}
}
\par\end{centering}
\begin{centering}
\subfloat[``linear'', $\|\protect\bpi-\protect\bpistar\|_{\infty}/\|\protect\bpistar\|_{\infty}$]{\begin{centering}
\includegraphics[width=0.33\textwidth]{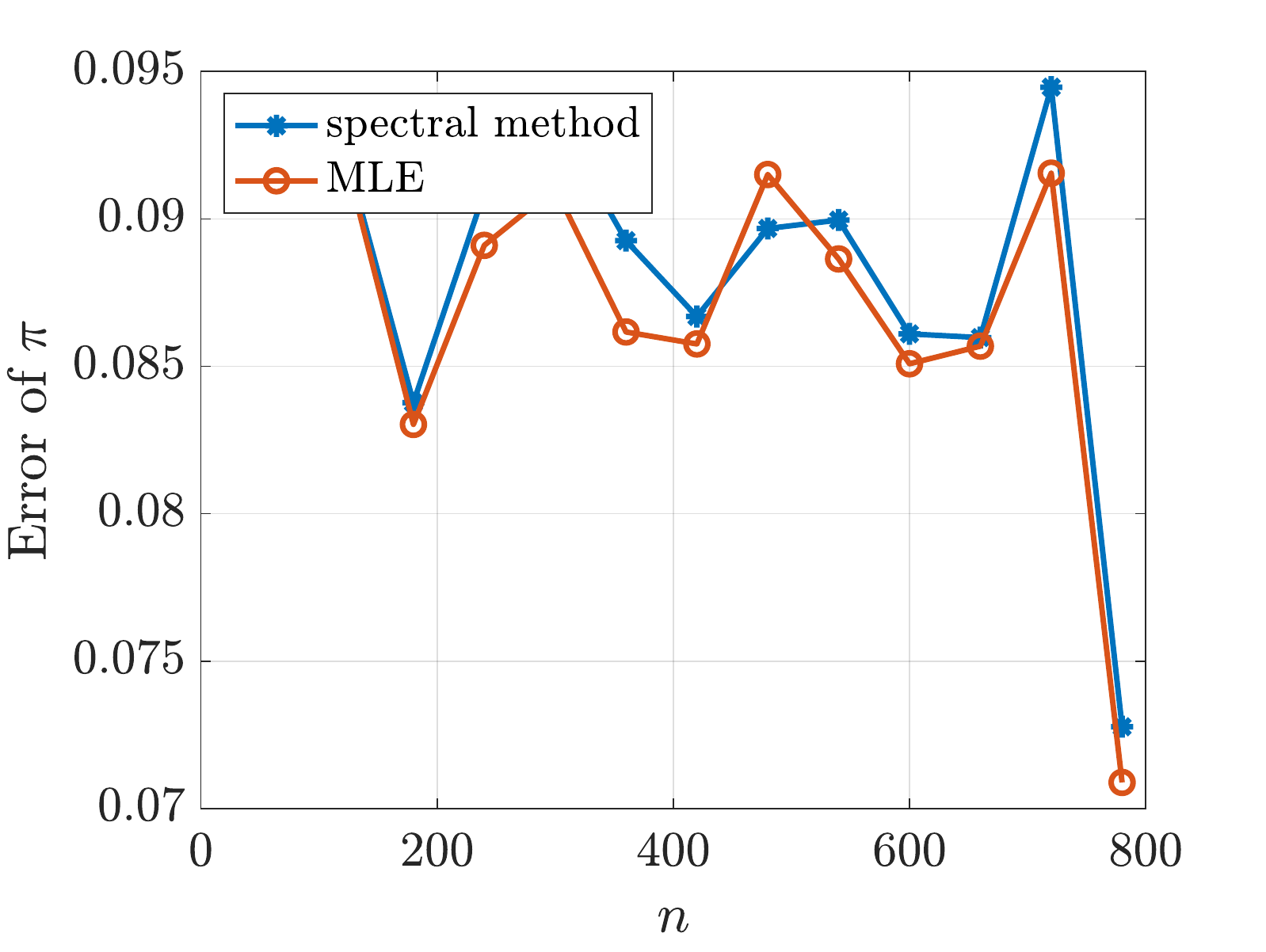}
\par\end{centering}
}\subfloat[``linear'', $\|\protect\btheta-\protect\bthetastar\|_{\infty}$]{\begin{centering}
\includegraphics[width=0.33\textwidth]{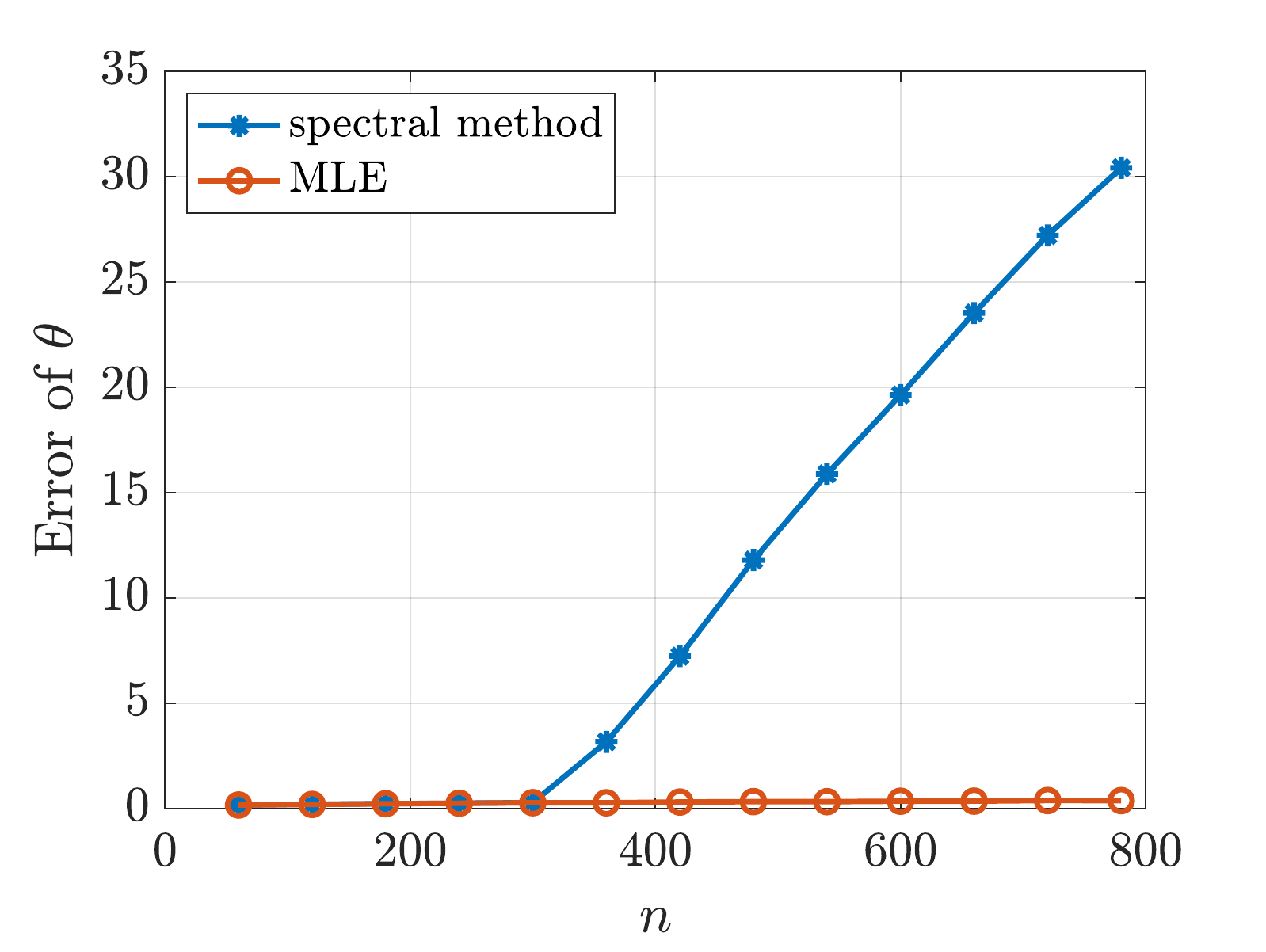}
\par\end{centering}
}\subfloat[\label{subfig:spectral_fails}``linear'', estimation of $\protect\bpistar$]{\begin{centering}
\includegraphics[width=0.33\textwidth]{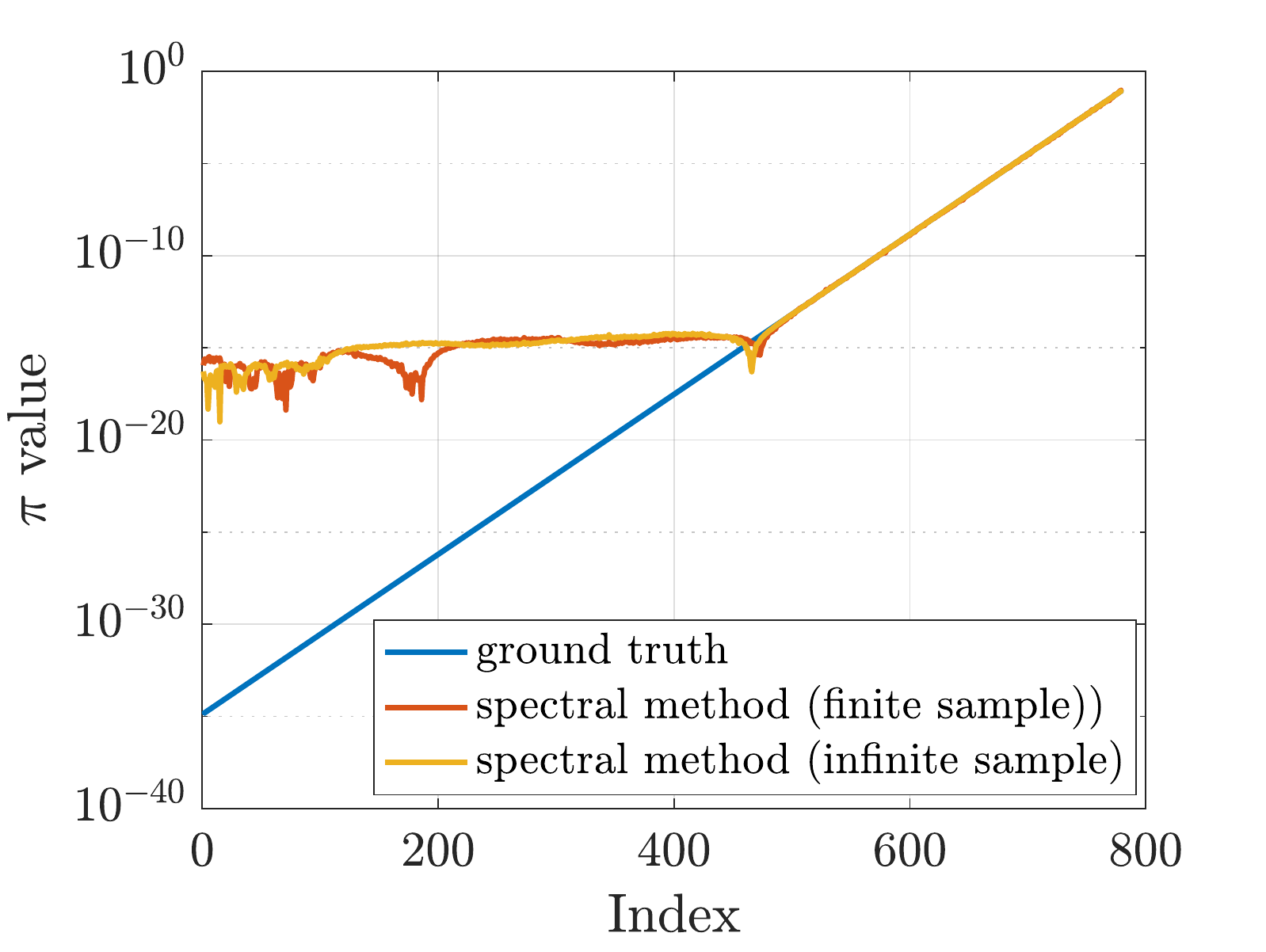}
\par\end{centering}
}
\par\end{centering}
\caption{\label{fig:exp_mle_spectral}Numerical comparisons of MLE and the
spectral method. The first row follows the ``sine'' setting, and
the second row follows the ``linear'' setting. When measuring the
error, we make sure that both $\protect\btheta$ and $\protect\bthetastar$
are zero-mean, and both $\protect\bpi$ and $\protect\bpistar$ are
probability vectors. Each curve in the first two columns is an average
over 30 independent trials. }

\end{figure}

We compare the classical algorithms, namely MLE and the spectral method,
in $\gridone(n,r,p)$ graphs with $L$ measurements on each edge.
We fix $r=10,p=0.8,L=100$, and let $n$ vary between $60$ and $780$.
In terms of the true scores $\bthetastar$, two settings are considered: 
\begin{itemize}
\item The ``sine'' setting, where $\thetastar_{i}=\sin(i/r),1\le i\le n$;
this implies that $\ke,\kappa\lesssim1$.
\item The ``linear'' setting, where $\thetastar_{i}=i/r$; this implies
that $\ke\lesssim1$, while $\kappa$ grows exponentially with $n$.
\end{itemize}
Numerical results are summarized in Figure \ref{fig:exp_mle_spectral}.
In the ``sine'' setting, both algorithms have similar performance.
However, in the ``linear'' setting, the spectral method is able
to estimate $\bpistar$ accurately, while suffering from a linearly
growing $\linf$ error in estimating $\bthetastar$, as shown in Figure
\ref{fig:exp_mle_spectral}(e). Figure \ref{fig:exp_mle_spectral}(f)
reveals that this is because the spectral method fails to numerically
calculate those extremely small stationary probabilities on some nodes,
even in the case of infinite samples. This confirms the advantage
of MLE over the spectral method, in the scenarios where $\ke$ is
small but $\kappa$ can be very large.

\subsection{MLE and DC-overlap }

\begin{figure}
\begin{centering}
\subfloat[$\protect\gridone,p=0.5,L=30$]{\begin{centering}
\includegraphics[width=0.33\textwidth]{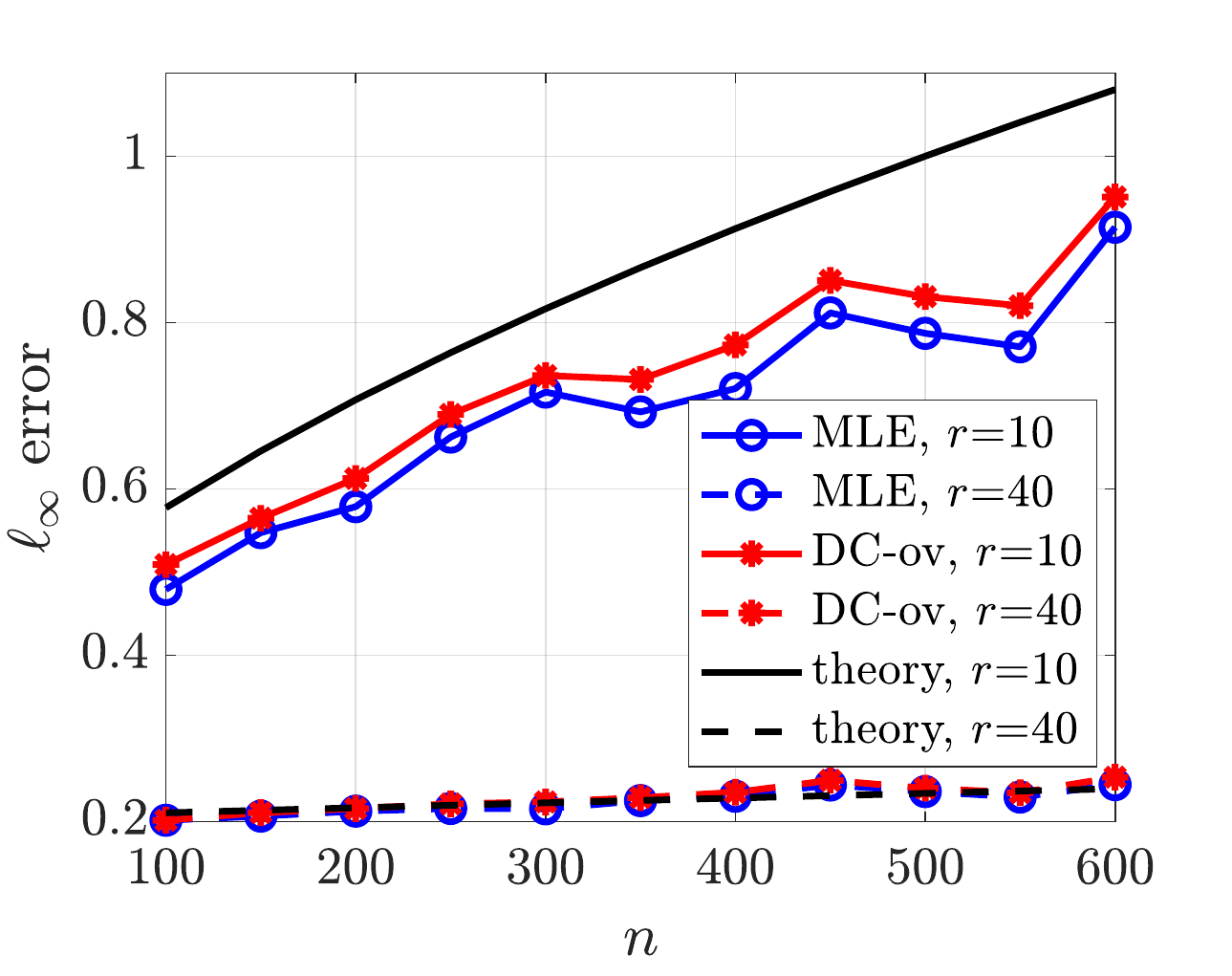}
\par\end{centering}
}\subfloat[$\protect\gridone,n=600,p=0.5$]{\begin{centering}
\includegraphics[width=0.33\textwidth]{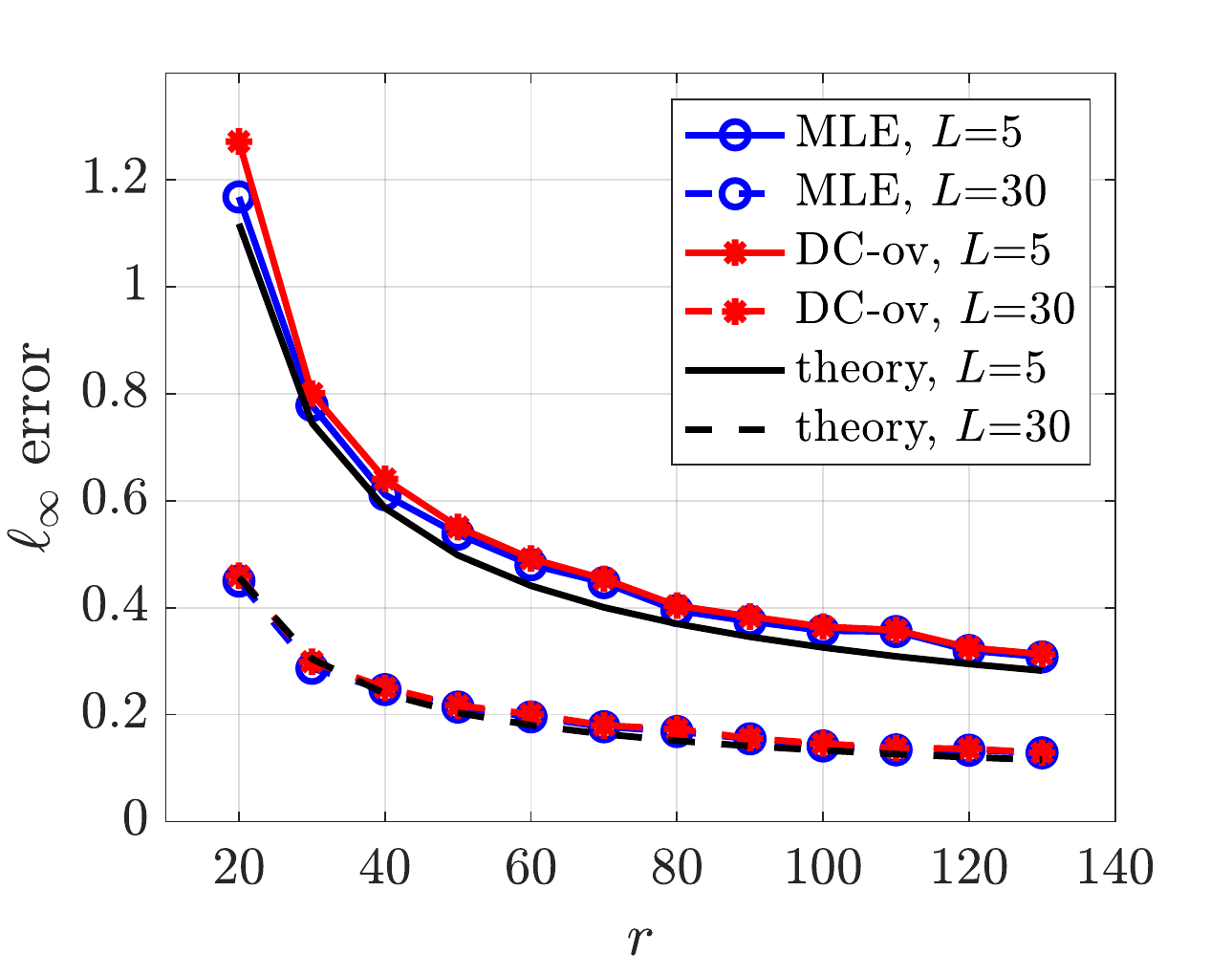}
\par\end{centering}
}\subfloat[$\protect\gridone,n=500,r=20$]{\begin{centering}
\includegraphics[width=0.33\textwidth]{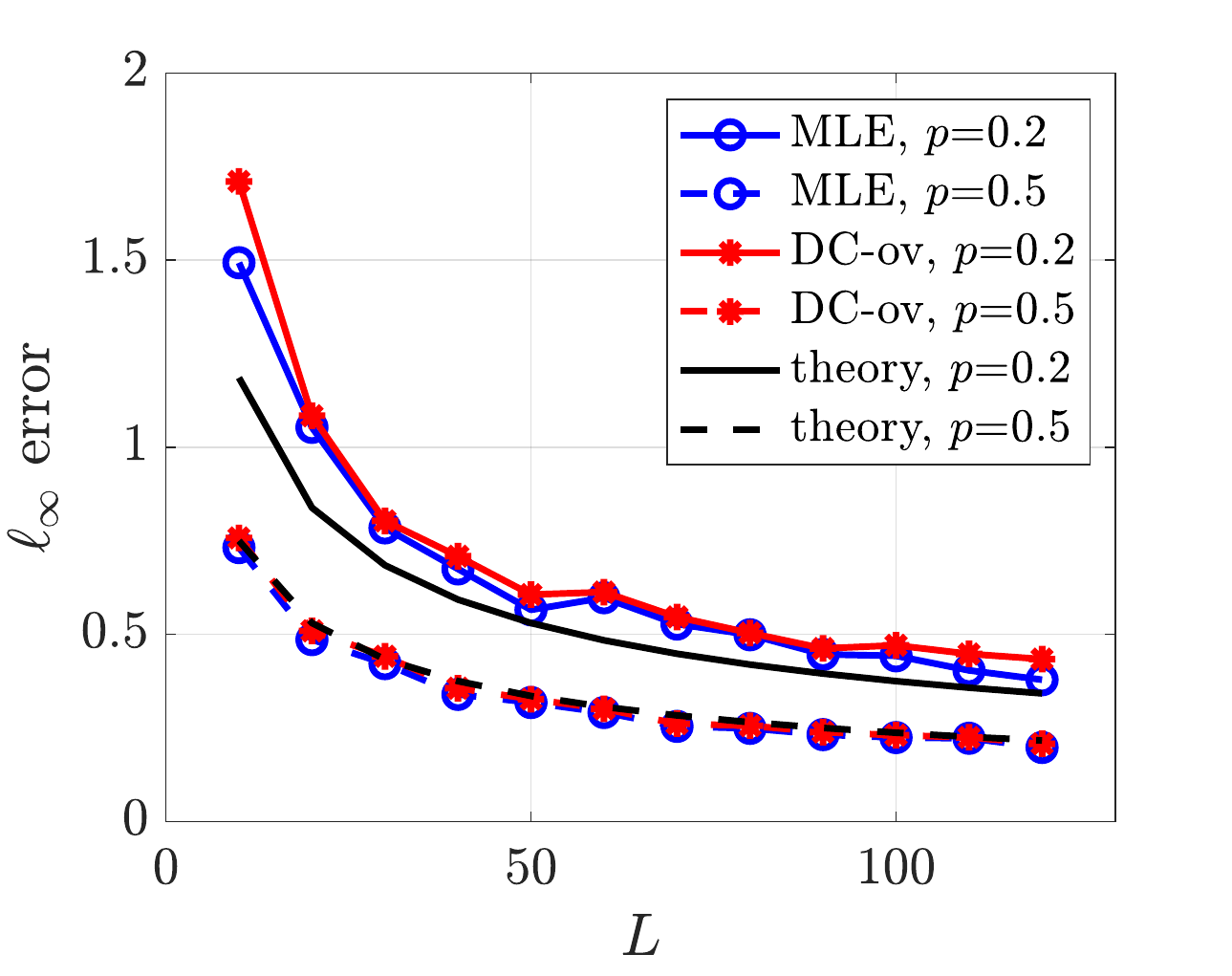}
\par\end{centering}
}
\par\end{centering}
\begin{centering}
\subfloat[$\protect\gridtwo,p=0.5,L=30$]{\begin{centering}
\includegraphics[width=0.33\textwidth]{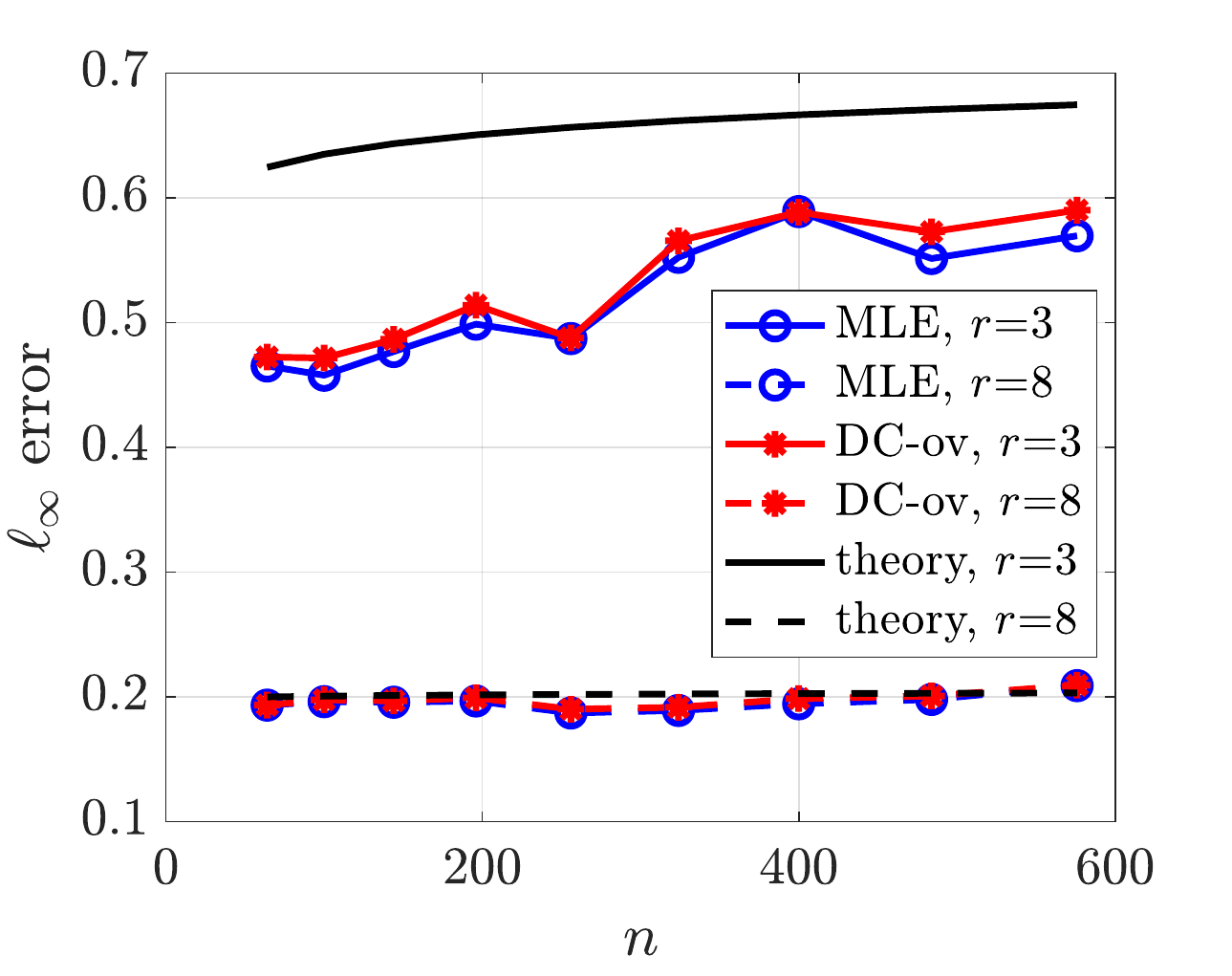}
\par\end{centering}
}\subfloat[$\protect\gridtwo,n=20^{2},p=0.5$]{\begin{centering}
\includegraphics[width=0.33\textwidth]{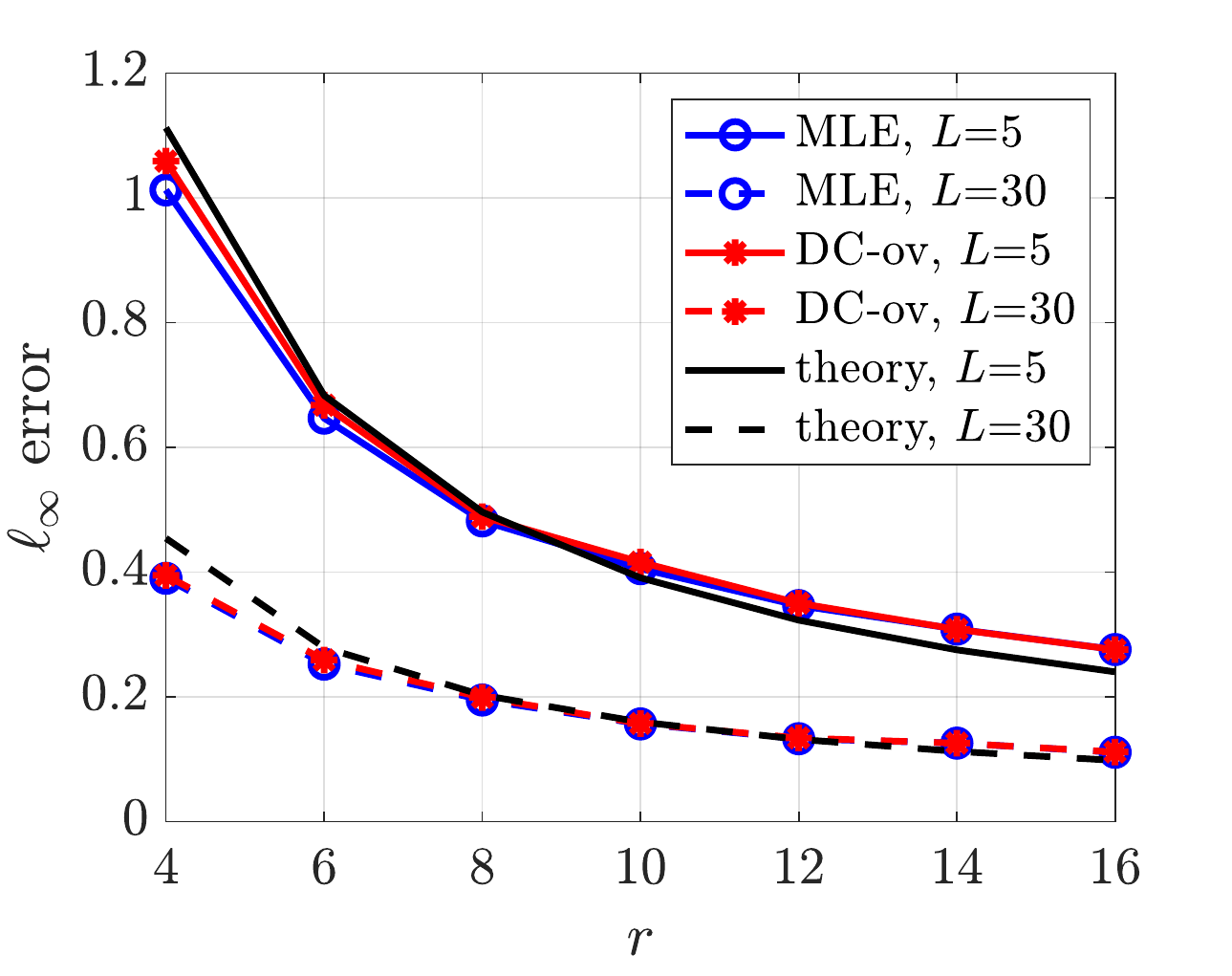}
\par\end{centering}
}\subfloat[$\protect\gridtwo,n=20^{2},r=5$]{\begin{centering}
\includegraphics[width=0.33\textwidth]{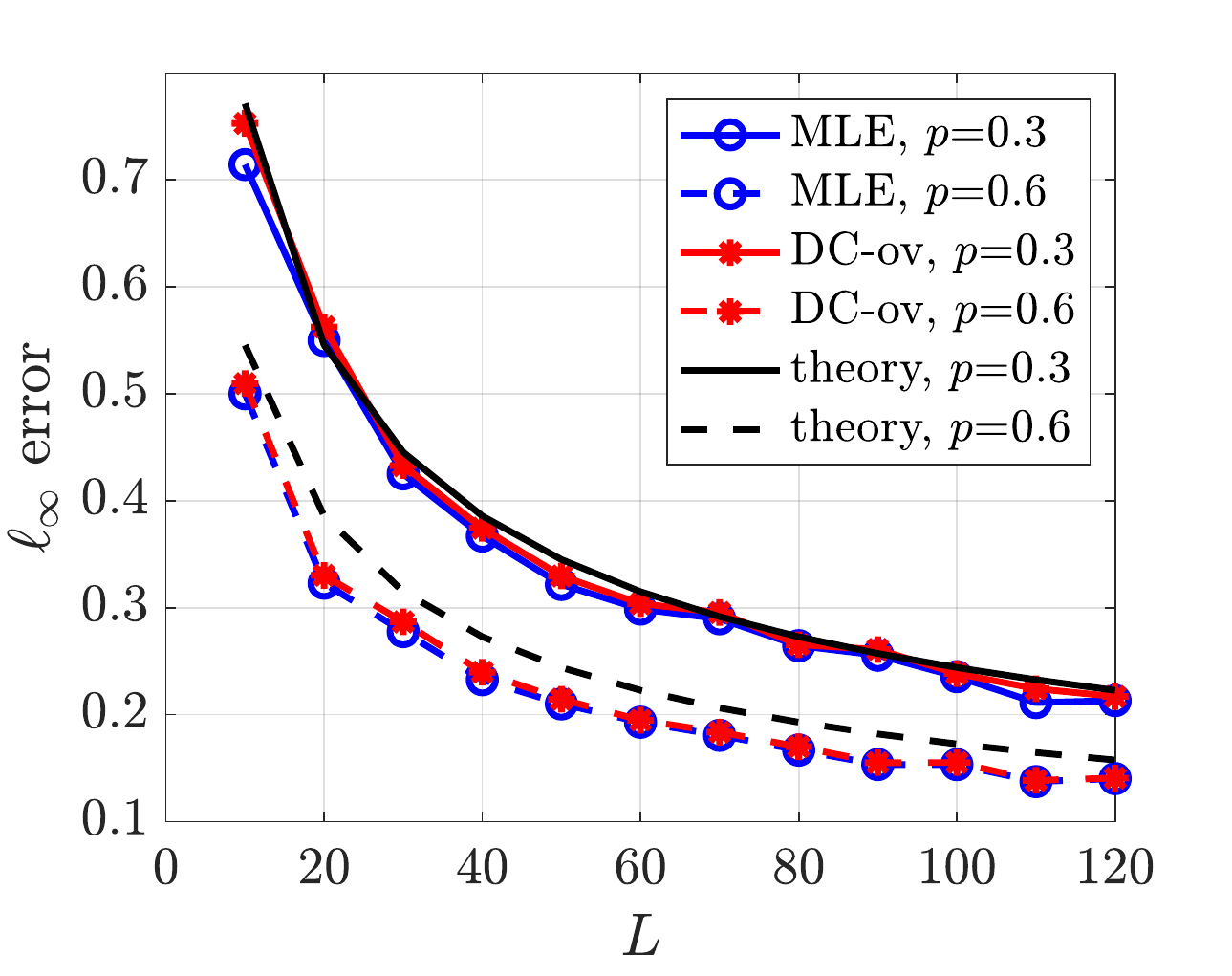}
\par\end{centering}
}
\par\end{centering}
\caption{\label{fig:exp_mle_dcov}An empirical comparison between MLE and $\protect\distrov$,
with the error metric $\|\protect\btheta-\protect\bthetastar\|_{\infty}$.
Each curve is an average over 40 independent trials.}
\end{figure}

We compare the $\linf$ errors achieved by MLE and $\distrov$. Two
settings are considered:
\begin{itemize}
\item For $\gridone(n,r,p)$, we set $\thetastar_{i}=i/r,1\le i\le n$,
and the theory bound $5\sqrt{n/r^{2}+1}\sqrt{1/rpL}$.
\item For $\gridtwo(n,r,p)$, we set $\thetastar_{\bm{i}}=(i_{1}+i_{2})/r,1\le i_{1},i_{2}\le\sqrt{n}$
(where $\bm{i}=(i_{1},i_{2})$ is the two-dimensional index for each
node, as in Definition \ref{def:grids}), and the theory bound $6\sqrt{\log(n)/r^{2}+1}\sqrt{1/r^{2}pL}$.
\end{itemize}
Figure \ref{fig:exp_mle_dcov} summarizes how the $\linf$ errors
change with various parameters. It is confirmed that MLE and $\DCov$
have very close performance, and match the theoretical bounds quite
well.

\subsection{Solving MLE by various methods \label{subsec:exp_solve_mle}}

\begin{figure}
\begin{centering}
\subfloat[$\protect\gridone$]{\begin{centering}
\includegraphics[width=0.4\textwidth]{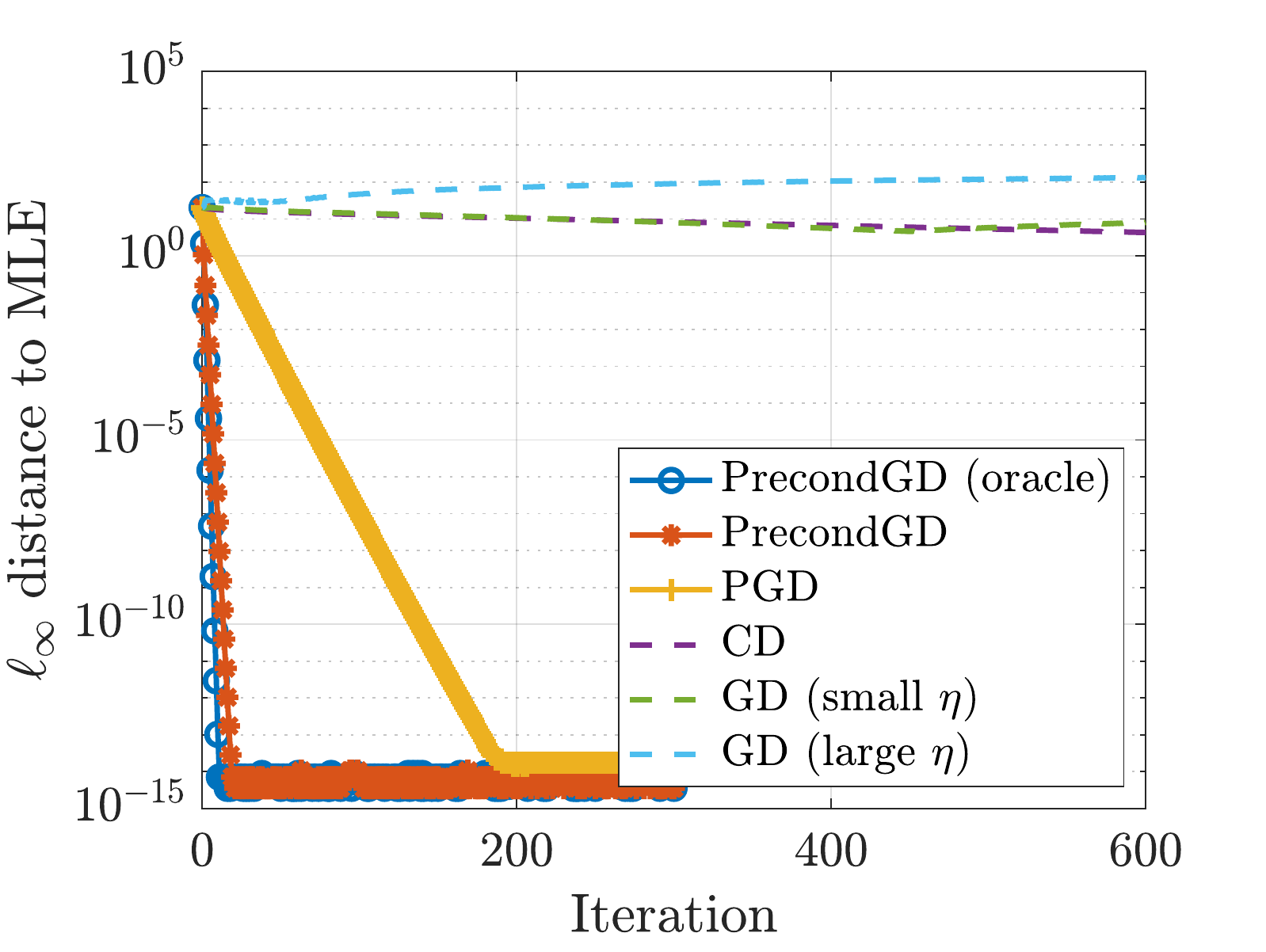}
\par\end{centering}
}\subfloat[$\protect\gridtwo$]{\begin{centering}
\includegraphics[width=0.4\textwidth]{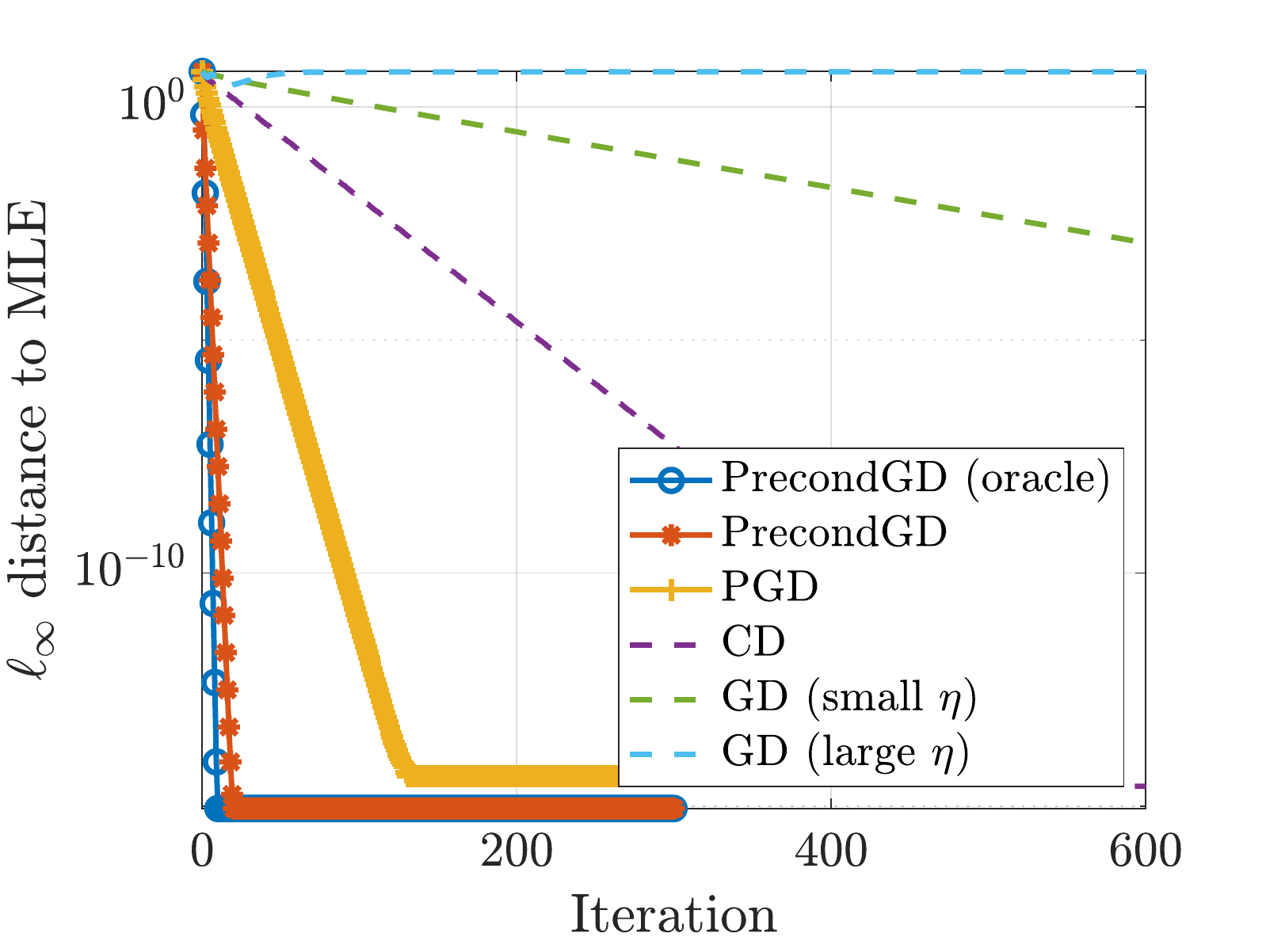}
\par\end{centering}
}
\par\end{centering}
\caption{\label{fig:exp_mle_convergence}Convergence of various methods for
solving MLE.}
\end{figure}

We compare the convergence of various algorithms for solving MLE (\ref{eq:def_loss}).
Two settings are considered:
\begin{itemize}
\item Consider a $\gridone$ graph with $n=400,r=10,p=0.8,L=100$. The ground-truth
$\bthetastar$ is defined by $\thetastar_{i}=i/r,1\le i\le n$.
\item Consider a $\gridtwo$ graph with $n=30^{2}=900,r=5,p=0.8,L=100$.
The ground-truth $\bthetastar$ is defined by $\thetastar_{\bm{i}}=(i_{1}+i_{2})/r$.
\end{itemize}
Algorithms of interests (and their configurations) include the following:
\begin{itemize}
\item Preconditioned gradient descent (PrecondGD): we fix the step size
$\eta=1$; for the preconditioner, we consider the ``oracle'' choice
$\Lzinv$, with $\Lz$ defined in (\ref{eq:def_Lz}), as well as the
implementable choice $\bL^{\dagger}$, where $\bL=(1/4)\sumE\Lij(\be_{i}-\be_{j})(\be_{i}-\be_{j})^{\top}$
(i.e.~the $\{\zij=\sigmoid'(\thetastar_{i}-\thetastar_{j})\}$ terms
in the definition of $\Lz$ is replaced with $\sigmoid'(0)=1/4$).
\item Projected gradient descent (PGD): the graph is partitioned into overlapping
\ErdosRenyi  subgraphs of size $\Theta(r)$ for $\gridone$, or $\Theta(r^{2})$
for $\gridtwo$; we set the step size $\eta\asymp1/rpL$ for $\gridone$,
or  $\eta\asymp1/r^{2}pL$ for $\gridtwo$.
\item Coordinate descent (CD): one entry of $\btheta$ is updated at each
step (by minimizing the sum of terms in (\ref{eq:def_loss}) that
are related to this entry), in a cyclic manner; we regard a sequence
of $n$ updates as one iteration.
\item Gradient descent (GD): for each setting, we consider two step sizes
$\eta_{\mathsf{small}}$ and $\eta_{\mathsf{large}}=5\eta_{\mathsf{small}}$,
where  $\eta_{\mathsf{small}}\asymp1/rpL$ for $\gridone$, or $\eta_{\mathsf{small}}\asymp1/r^{2}pL$
for $\gridtwo$; these are chosen at the edge of convergence, i.e.~GD
with $\eta_{\mathsf{small}}$ converges while GD with $\eta_{\mathsf{large}}$
does not.
\end{itemize}
It is easy to check that the computational complexity per iteration
is $O(|E|)$ for CD and GD. This is also true (at least in theory)
for PrecondGD and PGD, if they are implemented with the state-of-the-art
solvers for Laplacian linear systems.

Figure \ref{fig:exp_mle_convergence} illustrates the empirical comparisons
of the algorithms in both settings. It is confirmed that PrecondGD
has the fastest convergence, followed by PGD; on the other hand, vanilla
first-order methods like CD and GD are much slower, especially in
1D grids with a large size $n$ and small radius $r$.

\section{Discussion }

We have studied the problem of ranking from pairwise comparisons in
the BTL model with general graphs or graphs with locality. Moving
forward, it is highly desirable to prove that MLE, the principled
gold-standard algorithm for the BTL model, can achieve the optimal
error rates \emph{in the sparsest regime} in graphs with locality;
we hope that our work will be helpful for achieving this ultimate
goal. On the practical side, it is worthwhile to explore more computationally
efficient methods for solving the MLE problem (especially when the
loss landscape is ill-conditioned), or more practical divide-and-conquer
algorithms (e.g.~without requiring a graph partition as an input).

\section*{Acknowledgements}

We would like to thank Yuxin Chen for numerous helpful discussions. 

\appendix

\section{Deferred proofs}

\subsection{Proof of Lemma \ref{lem:MLE_existence_PGD_convergence} \label{subsec:proof_lemma_mle}}

\subsubsection{Part 1}

We prove the first part of Lemma \ref{lem:MLE_existence_PGD_convergence}
by contradiction. Suppose that a unique and finite MLE solution does
not exist. Then, according to Proposition \ref{prop:mle_existence},
there exists a disjoint partition $\Omega_{1},\Omega_{2}$ of the
nodes, such that $\yijl=1$ for all $(i,j)\in(\Omega_{1}\times\Omega_{2})\cap E$
and $1\le\ell\le\Lij$. Therefore, within any bounded set containing
$\bthetastar$, there is a strict descent direction (which increases
the scores on $\Omega_{1}$ and decreases the scores on $\Omega_{2}$)
for \emph{any} $\btheta$ in this set, which implies that $\|\nabla\Lcal(\btheta)\|_{2}$
is bounded away from zero.

However, we have proved in Step 2 of Section \ref{subsec:proof_mle_general}
that, conditioned on (\ref{eq:xi_high_prob_error}), the PrecondGD
iterations $\{\bthetat\}$ stays within a bounded set containing $\bthetastar$;
moreover, since the loss function $\Lcal(\btheta)$ is convex and
smooth, the classical descent lemma (with modifications that account
for the preconditioner) tells us that, with a sufficiently small step
size $\eta$, one has $\min_{1\le t\le T}\|\nabla\Lcal(\bthetat)\|_{2}\lesssim1/T\rightarrow0$
as $T$ grows to infinity. Hence we encounter a contradiction, which
means that a unique and finite MLE solution must exist.

\subsubsection{Part 2}

Consider minimizing a general loss function $\Lcal(\btheta)$, where
$\btheta\in\R^{d}$. The $t$-th iteration of PrecondGD is
\begin{equation}
\bthetatp=\bthetat-\eta\bM\cdot\nabla\Lcal(\bthetat),\label{eq:PGD_general}
\end{equation}
where $\bM\in\R^{d\times d}$ is a symmetric preconditioner. The following
result guarantees the convergence of PrecondGD, whose proof will be
provided momentarily.
\begin{prop}
\label{prop:PrecondGD_convergence}Consider the $t$-th iteration
of PrecondGD (\ref{eq:PGD_general}). Let $\bU\in\R^{d\times m}$
be an orthonormal matrix representing a $m$-dimensional subspace
in $\R^{d}$. Suppose that the following conditions hold: 
\begin{itemize}
\item The optimization problem $\min_{\btheta\in\col(\bU)}\Lcal(\btheta)$,
where $\Lcal$ is smooth and convex, has a unique minimizer $\bthetahat$; 
\item There exists a convex set $\Ccal\subseteq\col(\bU)$ such that $\bU^{\top}\bthetat,\bU^{\top}\bthetahat\in\textsf{interior}(\bU^{\top}\Ccal)$; 
\item The Hessian and preconditioner satisfy that $\col(\nabla^{2}\Lcal(\btheta)),\col(\bM)\subseteq\col(\bU)$,
$\bM\succcurlyeq\epsilon\cdot\bU\bU^{\top}$ (for some $\epsilon>0$)
is symmetric, and $\alpha_{1}\bM^{\dagger}\preccurlyeq\nabla^{2}L(\btheta)\preccurlyeq\alpha_{2}\bM^{\dagger}$
for all $\btheta\in\Ccal$, where $0<\alpha_{1}\le\alpha_{2}$.
\end{itemize}
Under these conditions, we have $\|\bthetatp-\bthetahat\|_{\bM^{\dagger}}\le(1-\eta\alpha_{1})\|\bthetat-\bthetahat\|_{\bM^{\dagger}}$,
provided that $0<\eta\le1/\alpha_{2}$. 
\end{prop}
It remains to show how Part 2 of Lemma~\ref{lem:MLE_existence_PGD_convergence}
follows from this proposition. First, we shall specify the preconditioner
$\bM=\Lzinv$ and the subspace $\bU$ with $\col(\bU)=\{\bx\in\R^{n}:\bx^{\top}\vonen=0\}$.
Moreover, let $\Ccal$ be a sufficiently large bounded convex set,
such that both $\bthetaMLE$ and
\[
\Ccal_{B}\coloneqq\Big\{\btheta\in\R^{n}:\btheta^{\top}\vone_{n}=0;\big|(\theta_{k}-\theta_{\ell})-(\thetastar_{k}-\thetastar_{\ell})\big|\le\Bkl,(k,\ell)\in E\Big\}
\]
are contained in the interior of $\Ccal$ (after projection onto the
subspace $\col(\bU)$). Finally, recall the expression~(\ref{eq:def_Hessian})
of $\nabla^{2}\Lcal(\btheta)$; due to the boundedness of $\Ccal$,
there exist some $0<\alpha_{1}\le\alpha_{2}$ such that $\alpha_{1}\Lz\preccurlyeq\nabla^{2}\Lcal(\btheta)\preccurlyeq\alpha_{2}\Lz$
for all $\btheta\in\Ccal$. Now, it is easy to checked that the required
conditions in Proposition~\ref{prop:PrecondGD_convergence} are satisfied
for all $t\ge0$ in the setting of Lemma~\ref{lem:MLE_existence_PGD_convergence},
which concludes our proof of the lemma.
\begin{proof}
[Proof of Proposition~\ref{prop:PrecondGD_convergence}]We first prove
for the simpler case where $\bU=\bI_{d}$. Denote $\bdeltat\coloneqq\bthetat-\bthetahat$.
According to the PrecondGD update and the optimality condition for
$\bthetahat$, namely $\nabla\Lcal(\bthetahat)=\mathbf{0}$, we have
\begin{equation}
\bdeltatp=\bdeltat-\eta\bM\nabla\Lcal(\bthetat)=\bdeltat-\eta\bM\big(\nabla\Lcal(\bthetat)-\nabla\Lcal(\bthetahat)\big)=\bdeltat-\eta\bM\int_{0}^{1}\nabla^{2}\Lcal\big(\bthetat(\tau)\big){\rm d}\tau\cdot\bdeltat,\label{eq:PGD_gradient_to_Hessian}
\end{equation}
where $\bthetat(\tau)\coloneqq\bthetahat+\tau\cdot(\bthetat-\bthetahat),0\le\tau\le1$.
Consequently,
\begin{align*}
\bM^{-\frac{1}{2}}\bdeltatp & =\bM^{-\frac{1}{2}}\bdeltat-\eta\bM^{\frac{1}{2}}\int_{0}^{1}\nabla^{2}\Lcal\big(\bthetat(\tau)\big){\rm d}\tau\cdot\bM^{\frac{1}{2}}\cdot\bM^{-\frac{1}{2}}\bdeltat\\
 & =\Big(\bI_{n}-\eta\bM^{\frac{1}{2}}\int_{0}^{1}\nabla^{2}\Lcal\big(\bthetat(\tau)\big){\rm d}\tau\cdot\bM^{\frac{1}{2}}\Big)\bM^{-\frac{1}{2}}\bdeltat.
\end{align*}
By convexity of $\Ccal$, we have $\bthetat(\tau)\in\Ccal$ for all
$0\le\tau\le1$, and hence $\alpha_{1}\bI_{n}\preccurlyeq\bM^{\frac{1}{2}}\int_{0}^{1}\nabla^{2}\Lcal(\bthetat(\tau)){\rm d}\tau\cdot\bM^{\frac{1}{2}}\preccurlyeq\alpha_{2}\bI_{n}$
by assumption. Therefore, if $0<\eta\le1/\alpha_{2}$, then $\mathbf{0}\preccurlyeq\bI_{n}-\eta\bM^{\frac{1}{2}}\int_{0}^{1}\nabla^{2}\Lcal(\bthetat(\tau)){\rm d}\tau\cdot\bM^{\frac{1}{2}}\preccurlyeq(1-\eta\alpha_{1})\bI_{n}$,
which implies $\|\bM^{-\frac{1}{2}}\bdeltatp\|_{2}=\|\bdeltatp\|_{\bM^{-1}}\le(1-\eta\alpha_{1})\|\bdeltat\|_{\bM^{-1}}$. 

The proof for the case of general $\bU$ is very similar, with some
minor modifications. Notice that, by assumption, we have $\bdeltat=\bU\bU^{\top}\bdeltat$,
$\bM=\bU\bU^{\top}\bM$, and $\bU\bU^{\top}\nabla\Lcal(\bthetahat)=\bm{0}$.
Therefore $\bM\nabla\Lcal(\bthetahat)=\bM\bU\bU^{\top}\nabla\Lcal(\bthetahat)=\bm{0}$,
and it is thus easy to check that (\ref{eq:PGD_gradient_to_Hessian})
still holds. Consequently,
\begin{align*}
(\bM^{\dagger})^{\frac{1}{2}}\bdeltatp & =(\bM^{\dagger})^{\frac{1}{2}}\bdeltat-\eta\bM^{\frac{1}{2}}\int_{0}^{1}\nabla^{2}\Lcal\big(\bthetat(\tau)\big){\rm d}\tau\cdot\bM^{\frac{1}{2}}\cdot(\bM^{\dagger})^{\frac{1}{2}}\bdeltat\\
 & =\Big(\bU\bU^{\top}-\eta\bM^{\frac{1}{2}}\int_{0}^{1}\nabla^{2}\Lcal\big(\bthetat(\tau)\big){\rm d}\tau\cdot\bM^{\frac{1}{2}}\Big)(\bM^{\dagger})^{\frac{1}{2}}\bdeltat.
\end{align*}
If $0<\eta\le1/\alpha_{2}$, then $\mathbf{0}\preccurlyeq\bU\bU^{\top}-\eta\bM^{\frac{1}{2}}\int_{0}^{1}\nabla^{2}\Lcal(\bthetat(\tau)){\rm d}\tau\cdot\bM^{\frac{1}{2}}\preccurlyeq(1-\eta\alpha_{1})\bU\bU^{\top}$,
which implies $\|(\bM^{\dagger})^{\frac{1}{2}}\bdeltatp\|_{2}=\|\bdeltatp\|_{\bM^{\dagger}}\le(1-\eta\alpha_{1})\|\bdeltat\|_{\bM^{\dagger}}$.
\end{proof}

\section{Miscellaneous results }
\begin{fact}
\label{fact:sigmoid}The sigmoid function $\sigmoid(x)=1/(1+e^{-x})$
satisfies that, for all $x\in\R$, $\sigmoid'(x)\ge1/(4e^{|x|})$.
As a result, for any $\alpha\ge1$, $|x|\le\log\alpha$ implies $\sigmoid'(x)\ge1/(4\alpha)$.
\end{fact}
\begin{proof}
$\sigmoid'(x)=\sigmoid(x)\cdot(1-\sigmoid(x))=\frac{1}{(1+e^{-x})(1+e^{x})}=\frac{1}{(1+e^{-|x|})(1+e^{|x|})}\ge\frac{1}{2(1+e^{|x|})}\ge\frac{1}{4e^{|x|}}$.
\end{proof}
\begin{example}
\label{exa:example_mle}We consider some special graphs where the
parameters $\{\Bkl\}$ (recalling that $\Bkl\lesssim\sqrt{\Omegakl(\Lz)\ke}$)
and/or $\{\Vaggkl\}$ defined in Section~\ref{sec:classical_algorithms}
can be explicitly calculated or bounded, with the help of Fact~\ref{fact:resistances}.
Some of our bounds on $\Omegakl$ have been present in the literature,
e.g.~\cite[Table~1]{hendrickx2019graph}. In the following, we always
let $n$ denote the number of nodes in the graph.
\begin{itemize}
\item \textbf{Trees.} Consider any $k\neq\ell$, and let $P_{k,\ell}$ be
the path connecting nodes $k$ and $\ell$; then $\Omegakl(\Lz)=\sum_{(i,j)\in P_{k,\ell}}1/\Lij\zij$
by the Series Law. In addition, for any $(k,\ell)\in E$, $(\ei-\ej)^{\top}\Lzinv(\ek-\el)$
is nonzero only for $(i,j)=(k,\ell$), which implies that $\Vaggkl=\Lkl\Omegakl(\Lz)=\Lkl\cdot(1/\Lkl\zkl)=1/\zkl\lesssim\ke$.
\item \textbf{Rings.} Assume $\Lij=L$ for simplicity. Consider any $k\neq\ell$,
and let $P_{k,\ell}$ be any of the two paths connecting nodes $k$
and $\ell$; then $\Omegakl(\Lz)\le\sum_{(i,j)\in P_{k,\ell}}1/\Lij\zij\lesssim\ke n/L$.
As for $\Vaggkl$ with $(k,\ell)\in E$, recall that $\Vaggkl=L\sumE|v_{i}-v_{j}|$,
where $\bv=\Lzinv(\ek-\el)$; as many of the $|v_{i}-v_{j}|$ terms
cancel out along the path connecting nodes $k$ and $\ell$, we get
$\Vaggkl=2L(v_{k}-v_{\ell})=2L\cdot\Omegakl(\Lz)\le2/\zkl\lesssim\ke$.
\item \textbf{Complete graphs. }Assume $\Lij=L$ for simplicity. For any
$k\neq\ell$, we can find $\Theta(n)$ parallel paths connecting nodes
$k$ and $\ell$, each with length at most $2$ and resistance $O(\ke/L)$.
Therefore, $\Omegakl(\Lz)\lesssim\ke/nL$ by the Parallel Law.
\item \textbf{Barbell graphs.} Suppose that two complete subgraphs, each
with $\Theta(n)$ nodes, are connected by a single edge $(s,t)$,
where $s$ belongs to one subgraph, and $t$ belongs to the other.
Assume $\Lij=L$ for all $(i,j)\neq(s,t)$. If nodes $k$ and $\ell$
belong to the same complete subgraphs, then $\Omegakl(\Lz)\lesssim\ke/nL$
according to the previous result for complete graphs. Otherwise, if
nodes $k,s$ belong to one subgraph and $\ell,t$ belong to the other,
then by the triangle inequality of effective resistances, we have
$\Omegakl(\Lz)\le\Omega_{k,s}(\Lz)+\Omega_{s,t}(\Lz)+\Omega_{t,\ell}(\Lz)\lesssim\ke(\frac{1}{nL}+\frac{1}{L_{s,t}})$;
this shows how increasing the sample size $L_{s,t}$ on the bottleneck,
i.e.~edge $(s,t)$, of the graph is crucial for achieving a small
estimation error. 
\end{itemize}
\end{example}

\bibliographystyle{alphaabbr}
\bibliography{refs}

\end{document}